\documentclass{article} % For LaTeX2e
\usepackage{iclr2023_conference,times}

\usepackage{amsmath,amsfonts,bm}

\def\eqref#1{equation~\ref{#1}}
\def\1{\bm{1}}

\DeclareMathAlphabet{\mathsfit}{\encodingdefault}{\sfdefault}{m}{sl}
\SetMathAlphabet{\mathsfit}{bold}{\encodingdefault}{\sfdefault}{bx}{n}

\newcommand{\KL}{D_{\mathrm{KL}}}
\newcommand{\Var}{\mathrm{Var}}

\usepackage{hyperref}       % hyperlinks
\usepackage{url}            % simple URL typesetting
\usepackage{nicefrac}       % compact symbols for 1/2, etc.
\usepackage{microtype}      % microtypography
\usepackage{xcolor}         % colors

\usepackage{amsthm,bbm,dsfont}

\usepackage{enumitem}

\usepackage{natbib}
\usepackage{algorithm}
\usepackage{algorithmic}

\usepackage{colortbl}%%table
\definecolor{DarkBlue}{RGB}{22,54,93}%%table

\usepackage[capitalize,noabbrev]{cleveref}
\usepackage{graphicx}
\usepackage{subfigure}
\usepackage{adjustbox}

\usepackage{wrapfig}

\usepackage{mathtools}
\usepackage{threeparttable}
\usepackage{setspace,wasysym}
\usepackage{commath}

\theoremstyle{plain}
\newtheorem{theorem}{Theorem}%[section]

\newtheorem{lemma}{Lemma}
\newtheorem{corollary}{Corollary}
\theoremstyle{definition}
\newtheorem{definition}{Definition}
\newtheorem{assumption}{Assumption}
\theoremstyle{remark}
\newtheorem{remark}{Remark}

\newcommand{\Pb}{\mathbb{P}}
\newcommand{\Rb}{\mathbb{R}}
\newcommand{\hP}{\hat{P}}
\newcommand{\hV}{\bar{V}}

\newcommand{\hphi}{\hat{\phi}}
\newcommand{\hmu}{\hat{\mu}}

\newcommand{\eo}{\epsilon_0}

\newcommand{\Eb}{\mathbb{E}}

\newcommand{\sphi}{\phi^*}

\newcommand{\Mc}{\mathcal{M}}
\newcommand{\Uc}{\mathcal{U}}
\newcommand{\Sc}{\mathcal{S}}
\newcommand{\Ac}{\mathcal{A}}
\newcommand{\Dc}{\mathcal{D}}
\newcommand{\Cc}{\mathcal{C}}
\newcommand{\Nc}{\mathcal{N}}
\newcommand{\Xc}{\mathcal{X}}
\newcommand{\Hc}{\mathcal{H}}
\newcommand{\hb}{\hat{b}}

\newcommand{\At}{\tilde{A}}
\newcommand{\hQ}{\bar{Q}}
\newcommand{\Ec}{\mathcal{E}}
\newcommand{\MLE}{\text{MLE}}

\newcommand{\jing}[1]{{\color{magenta}(JY: #1)}}

\allowdisplaybreaks[4]

\title{Safe Exploration Incurs Nearly No Additional Sample Complexity for Reward-free RL}

\author{Ruiquan Huang \\
The Pennsylvania State University\\
\texttt{rzh5514@psu.edu} \\
\And
Jing Yang\\
The Pennsylvania State University\\
\texttt{yangjing@psu.edu} \\
\And
Yingbin Liang \\
The Ohio State University \\
\texttt{liang.899@osu.edu} \\
}

\iclrfinalcopy % Uncomment for camera-ready version, but NOT for submission.
\begin{document}

\maketitle

%\fancyhead{}
\fancyhead[L]{Published as a conference paper at ICLR 2023}

\begin{abstract}
Reward-free reinforcement learning (RF-RL), a recently introduced RL paradigm, relies on random action-taking to explore the unknown environment without any reward feedback information. 
While the primary goal of the exploration phase in RF-RL is to reduce the uncertainty in the estimated model with minimum number of trajectories, in practice, the agent often needs to abide by certain safety constraint at the same time. It remains unclear how such safe exploration requirement would affect the corresponding sample complexity in order to achieve the desired optimality of the obtained policy in planning. In this work, we make a first attempt to answer this question. In particular, we consider the scenario where a safe baseline policy is known beforehand, and propose a unified Safe reWard-frEe ExploraTion (SWEET) framework. We then particularize the SWEET framework to the tabular and the low-rank MDP settings, and develop algorithms coined Tabular-SWEET and Low-rank-SWEET, respectively. Both algorithms leverage the concavity and continuity of the newly introduced truncated value functions, and are guaranteed to achieve zero constraint violation during exploration with high probability. Furthermore, both algorithms can provably find a near-optimal policy subject to any constraint in the planning phase. Remarkably, the sample complexities under both algorithms match or even outperform the state of the art in their constraint-free counterparts up to some constant factors, proving that safety constraint hardly increases the sample complexity for RF-RL.
\end{abstract}

\vspace{-0.1in}
\section{Introduction}
\vspace{-0.1in}
Reward-free reinforcement learning (RF-RL) is an RL paradigm under which a learning agent first explores an unknown environment without any reward signal in the exploration phase, and then utilizes the gathered information to obtain a near-optimal policy for {\it any} reward function during the planning phase. Since formally introduced in \citet{jin2020provably}, RF-RL has attracted increased attention in the research community~\citep{kaufmann2021adaptive,zhang2020nearly,zhang2021reward,wang2020reward,modi2021model}. It is particularly attractive for applications where many reward functions may be of interest, such as multi-objective RL~\citep{miryoosefi2021simple}, or the reward function is not specified by the environment but handcrafted in order to incentivize some desired behavior of the RL agent \citep{jin2020provably}. 

%Reward-free reinforcement learning (RL)studies how does an autonomous agent find a near optimal policy for any reward by efficiently explore an unknown environment. 
The ability of RF-RL to identify a near-optimal policy in response to an arbitrary reward function relies on the fact that the agent is allowed to explore any action during exploration. However, in practice, unrestricted exploration is often unrealistic or even harmful. In order to build safe, responsible and reliable artificial intelligence (AI), the RL agent often has to abide by certain application-dependent constraints, even during the exploration phase. Two motivating applications are provided as follows.

\vspace{-0.03in}
\begin{itemize}[leftmargin=*]\itemsep=0pt
\setlength{\leftmargin}{-0.5in}
    \item \textbf{Autonomous driving.} In order to learn a near-optimal driving strategy, an RL agent needs to try various actions at different states through exploration. While RF-RL is an appealing approach as the reward function is difficult to specify, it is of critical importance for the RL agent to take safe actions (even during exploration) in order to avoid catastrophic consequences. 
    \item \textbf{Cellular network optimization.} The operation of cellular network needs to take a diverse corpus of key performance indicators into consideration, which makes RF-RL a plausible solution. Meanwhile, the exploration also needs to meet certain system requirements, such as power consumption. 
\end{itemize}
%Therefore, safety is an important issue during explorations. 
\vspace{-0.05in}
While meeting these constraints throughout the learning process is a pressing need for the broad adoption of RL in real-world applications, it is a mission impossible to accomplish if no other information is provided, as the learner has little knowledge of the underlying MDP at the beginning of the learning process and will inevitably take undesirable actions (in hindsight) and violate the constraints. On the other hand, in various engineering applications, there often exist either rule-based (e.g., autonomous driving) or human expert-guided (e.g., cellular network optimization) solutions to ensure safe operation of the system. One natural question is, {\it is it possible to leverage such existing safe solutions to ensure safety throughout the learning process? If so, how would the safe exploration requirement affect the corresponding RF-RL performances in terms of the sample complexity of exploration and the optimality and safety guarantees of the obtained policy in planning?}

To answer these questions, in this work, we introduce a new {\bf safe RF-RL} framework. %Specifically, we consider episodic MDPs defined in the form of $\Mc=(\Sc,\Ac, P,r,c,H, s_1)$, where $\Sc$, $\Ac$, $P$, $r$, $H$, $s_1$ represent the state space, action space, transition kernel, reward function, horizon of each episode and initial state, respectively. The safe constraint is defined through the cost function $c$. 
In the proposed safe RF-RL framework, the agent does not receive any reward information in the exploration phase, but is aware of a cost function associated with actions at a given state. We require that the cumulative cost in {\it each episode} is below a given threshold during exploration, with the aid of a pre-existing safe baseline policy $\pi^0$. The ultimate learning goal of safe RF-RL is to find a safe and near-optimal policy for any given reward and cost functions after exploration.

\textbf{Main contributions.} We summarize our main contributions as follows.
\vspace{-0.03in}
\begin{itemize}
[leftmargin=*]\itemsep=0pt
\setlength{\leftmargin}{-0.5in}
    \item First, we introduce a novel safe RF-RL framework that imposes safety constraints during both exploration and planning of RF-RL, which may have implications in various applications.  
    \item Second, we propose a unified safe exploration strategy coined SWEET to leverage the prior knowledge of a safe baseline policy $\pi^0$. SWEET admits general model estimation and safe exploration policy construction modules, thus can accommodate various MDP structures and different algorithmic designs. Under the assumption that the approximation error function is concave and continuous in the policy space, SWEET is guaranteed to achieve zero constraint violation during exploration, and output a near-optimal safe policy for {\it any} given reward function and safety constraint under some assumptions in planning, both with high probability. 
    
    \item Third, in order to facilitate the specific design of the approximation error function to ensure its concavity, we introduce a novel definition of truncated value functions. It relies on a new clipping method to avoid underestimation of the approximation error captured by the corresponding value function, and ensures the concavity of the resulted value function.
    \item Finally, we particularize the SWEET algorithm for both tabular and low-rank MDPs, and propose Tabular-SWEET and Low-rank-SWEET, respectively. 
    Both algorithms inherit the optimality guarantee during planning, and the safety guarantees in both exploration and planning.  Remarkably, the sample complexities under both algorithms match or even outperform the state of the art of
    their constraint-free counterparts up to some constant factors, proving that safety constraint incurs nearly no additional sample complexity for RF-RL.
 %   \item We show that the safety constraint does not downgrade the performance of unconstrained reward-free exploration.
\end{itemize}
\vspace{-0.03in}

% which is of critical importance to ensure the safety constraint is satisfied during exploration.

%\textbf{Low-rank MDP.}

\vspace{-0.1in}
\section{Preliminaries and problem formulation}
\vspace{-0.1in}
\subsection{Episodic Markov Decision Processes}\label{subsec:eps mdp}
\vspace{-0.05in}
We consider episodic Markov decision processes (MDPs) in the form of $\Mc = (\Sc,\Ac, P,H, s_1)$, where $\Sc$ is the state space and $\Ac$ is the finite action space, $H$ is the number of time steps in each episode, $P=\{P_h\}_{h=1}^H$ is a collection of transition kernels, and $P_h(s_{h+1}|s_h,a_h)$ denotes the transition probability from the state-action pair $(s_h,a_h)$ at step $h$ to state $s_{h+1}$ in the next step. Without loss of generality, we assume that in each episode of the MDP, the initial state is fixed at $s_1$. In addition, an MDP may be equipped with certain specified utility functions $u=\{u_h\}_{h=1}^H$, where we assume $u_h:  \Sc \times \Ac \rightarrow [0,1]$ is a deterministic function for ease of exposition. 
%and $c=\{c_h\}_{h=1}^H$ are the collections of reward and cost functions, respectively, where $r_n, c_h:  \Sc \times \Ac \rightarrow [0,1]$ are deterministic functions. 
%We also term $P$ a model since all MDPs considered in the paper share the same state space, action space and time horizon. 

A Markov policy $\pi$ is a set of mappings $\{\pi_h : \Sc \rightarrow \Delta(\Ac)\}_{h=1}^H$, where $\Delta(\Ac)$ is the set of all possible distributions over the action space $\Ac$. In particular, $\pi_h(a|s)$ denotes the probability of selecting action $a$ in state $s$ at time step $h$. We denote the set of all Markov policies by $\mathcal{X}$. 
% If an agent implements a policy $\pi$ under an MDP $\Mc$, the agent starts at the fixed initial state $s_1$. 
For an agent adopting policy $\pi$ in an MDP $\Mc$, at each step $h \in [H]$ where $[H]:=\{1,\dots,H\}$, she observes state $s_h 
\in \Sc$, and takes an action $a_h \in \Ac$ according to $\pi$, after which the environment transits to the next state $s_{h+1}$ with probability $P_h(s_{h+1}|s_h,a_h)$. The episode ends after $H$ steps, and we use a virtual state $s_{H+1}$ to denote the terminal state at step $H+1$. 
%Executing a policy $\pi$ under an MDP $\Mc$ induces a Markov chain over trajectory $(s_1,a_1,\cdots,s_H,a_H)$. 
We use $\mathop{\Eb}_{P,\pi}$ to denote the expectation of the distribution induced by the transition kernel $P$ and policy $\pi$.

\if{0}
\jing{cut?}
\textbf{Notations:} We use $s_h \sim (P, \pi)$ to denote a state sampled by executing the policy $\pi$ under the transition kernel $P$ for $h-1$ steps. If the previous state-action pair $(s_{h-1},a_{h-1})$ is given, we denote $s_h\sim P$ that $s_h$ follows the distribution $P_h(\cdot|s_{h-1},a_{h-1})$. We use the notation $\mathop{\Eb}_{(s_h, a_h) \sim (P, \pi)}\left[\cdot\right]$ to denote the expectation over states $s_h \sim (P,\pi)$ and actions $a_h \sim \pi$. For simplicity, we use $\Eb^*_{\pi} [\cdot]$ in shorthand for $\mathop{\Eb}_{(s_h, a_h) \sim (P^*, \pi)}\left[\cdot\right]$.
\fi

 % can be used to measure the performance of a policy $\pi$ under an MDP $\mathtt{M}$.  (We also define $u(s_{H+1}) = 0$, so that the utility function is well-defined.) 
Let $Q^\pi_{h,P,u}(s_h,a_h)$ and $V^\pi_{h,P,u}(s_h)$ be the corresponding action-value function and value function at step $h$, respectively, for a given collection of utility functions $u$. Then, %which gives the expected sum of utiliy functions received under the policy $\pi$ and MDP $\mathtt{M}$, starting from a given $s_h$ until the end of the episode. Namely, 
%\[
%\textstyle    V^\pi_{h,P,u}(s_h): = \sum_{h^\prime=h}^{H} \mathop{\Eb}_{(s_{h^\prime},a_{h^\prime}) \sim (P,\pi)}\left[u_{h^\prime}(s_{h^\prime},a_{h^\prime})\bigg|s_h\right].
%\]
% \begin{align*}\textstyle
%     & V^\pi_{h,P,u}(s_h): =  \mathop{\Eb}_{ P,\pi}\bigg[\sum_{h^\prime=h}^{H} u_{h'}(s_{h'},a_{h'})\bigg|s_h\bigg],~~~ Q^\pi_{h,P,u}(s_h,a_h) := \mathop{\Eb}_{ P,\pi}\bigg[\sum_{h^\prime=h}^{H} u_{h'}(s_{h'},a_{h'})\bigg|s_h,a_h\bigg].
% \end{align*}
$V^\pi_{h,P,u}(s_h): =  \mathop{\Eb}_{ P,\pi}\big[\sum_{h^\prime=h}^{H} u_{h'}(s_{h'},a_{h'})\big|s_h\big]$, and $Q^\pi_{h,P,u}(s_h,a_h) := \mathop{\Eb}_{ P,\pi}\big[\sum_{h^\prime=h}^{H} u_{h'}(s_{h'},a_{h'})\big|s_h,a_h\big]$.
%\jing{ $u(s_{H+1}) = 0$. necessary?}
%Similarly, we use $Q^\pi_{h,P,u}: \Sc \times \Ac\rightarrow \Rb$ to denote the action-value function at step $h$, which is given by
We also use the shorthand $V^\pi_{P,u}$ to denote $V^\pi_{1,P,u}(s_1)$ due to the fixed initial state, and $P_h f(s_h,a_h) = \Eb_{s_{h+1}\sim P_h(\cdot|s_h,a_h)}\left[f(s_{h+1})\right]$ for any function $f:\Sc\rightarrow\Rb$. %When $u\in \{c,r\}$, the corre and $u_h\in\{r_h,c_h\}$.
%In the standard reinforcement learning, the object is to find a policy $\pi$ that maximizes the value function $V^\pi_{P,r}$, where $P$ is an unknown transition kernel, and the utility function $r$, also known as reward (function), is revealed to the agent during the interaction with the MDP $\mathtt{M}$. 
We further assume that the utility functions are normalized such that for any trajectory generated under a policy, the cumulative value over one episode is bounded by 1, i.e., $\sum_{h=1}^{H}u_h(s_h,a_h)\leq 1$. %, u\in\{r,c\}$.

\if{0}
Reward $r,r^*$ and cost $c,c^*$ satisfy the normalization assumption defined as follows.  
\begin{assumption}[Normalization]\label{assum:norm}
A utility function $u$ is normalized if 
\[\sum_{h=1}^{H}u(s_h,a_h)\leq 1, \forall \text{ trajectory } \{s_1,a_1,\cdots,s_H,a_H\},\] 
\end{assumption}
\fi

\vspace{-0.05in}
\subsection{Safe Reward-Free Reinforcement Learning} 
\vspace{-0.05in}
% We first introduce the safety constraint.
The safe policy considered in this work is formally defined as follows.
\begin{definition}\label{def:safety}
Given an MDP $\Mc^*=(\Sc,\Ac,P^*,H,s_1)$, a set of cost functions $c=\{c_h\}_{h=1}^H$ and $\tau\in(0,1]$, a policy $\pi$ is $(c,\tau)$-safe if $V_{P^*,c}^{\pi}\leq \tau$. 
\end{definition}

Based on this definition of $(c,\tau)$-safe policies, we 
now elaborate the proposed safe RF-RL framework, which contains two phases.
In the first phase of ``\textbf{exploration}'', the agent is required to efficiently explore the unknown environment without reward signals, and simultaneously not to violate a predefined safety constraint $(c,\tau)$ {\it in each episode} during this exploration phase. 
Let $\pi^{(n)}$ be the policy implemented in the $n$-th episode of the exploration. Then the agent's exploration should satisfy the safety constraint in every episode with high probability, namely,
\begin{align}\textstyle
    \Pb\left[V_{P^*,c}^{\pi^{(n)}}\leq \tau,\forall n\in[N]\right] \geq 1-\delta,\label{eqn: Safety}
\end{align}
where $\delta\in(0,1)$ and $N$ is the total number of episodes in the exploration phase.
Note that the agent is only given a set of cost functions $c$ but not the reward $r$ in this phase. This is reasonable for many RL applications, where the purpose of exploration is not to maximize certain reward but to learn the environment, while the safety constraint need to be satisfied throughout the learning process. %can be uniform over various planning tasks.

In the second phase of ``\textbf{planning}'', the agent is given an {\it arbitrary} set of reward functions $r^*$ and a {\it new} set of safety constraint $(c^*,\tau^*)$. Without further exploration, she is required to learn an $\epsilon$-optimal policy $\bar{\pi}$ with respect to the given reward $r^*$, and subject to the safety constraint $(c^*,\tau^*)$. 
\begin{definition}
Given an MDP $\Mc^*=(\Sc,\Ac,P^*,H,s_1)$, reward functions $r^*$, cost functions $c^*$ and $\tau^*\in (0,1]$, $\bar{\pi}$ is an $\epsilon$-optimal ($c^*,\tau^*$)-safe policy if
\begin{align}\textstyle
    V^{\pi^*}_{P^*,r^*} - V^{\bar{\pi}}_{P^*, r^*}\leq \epsilon, \text { and }   V^{\bar{\pi}}_{P^*, c^*}\leq \tau^*, \label{eqn:objective}
\end{align}
where $\epsilon\in(0,1)$, and $\pi^*$ is the policy satisfying
$\pi^* = \arg\max_{\pi} V^{\pi}_{P^*, r^*}\text{ s.t. } V^{\pi}_{P^*, c^*}\leq \tau^*$.
\end{definition}

%{We note that the safety constraint $(c^*,\tau^*)$ provided in the planning phase may be different from the corresponding constraint $(c,\tau)$ imposed in exploration. In other words, strictly speaking, the MDP environment in exploration may be different from the corresponding MDP during planning. However, they do share the same state and action spaces, transitional kernel, horizon, and initial state.}
 
The design goal of safe RF-RL algorithms is {\it three-fold}: 1) to collect as few sample trajectories as possible, 2) to satisfy the safety constraint $(c,\tau)$ in the exploration phase, and 3) to obtain an $\epsilon$-optimal ($c^*,\tau^*$)-safe policy for any given reward $r^*$ and constraint $(c^*,\tau^*)$ in the planning phase.   

We note that it is impossible to ensure zero constraint violation with high probability during exploration if an agent starts with no information about the system. Therefore, we make the assumption that a safe baseline policy is available to the learning agent during exploration. Besides, we also assume the constrained MDP always has enough feasible solutions, either during exploration or planning.

\begin{assumption}[Feasibility]\label{assm: baseline}
The agent has knowledge of a baseline policy $\pi^0$ and $\kappa\in(0,\tau)$ such that $V^{\pi^0}_{P^*,c}\leq \tau-\kappa$. Besides, for any given constraint $({c},{\tau})$ in exploration or planning phases, the safety margin, defined as $\Delta(c, \tau):=\tau - \min_{\pi}V_{P^*,{c}}^{\pi}$, is bounded away from zero, i.e. $\Delta(c,\tau)\geq \Delta_{\min}>0$. 
\end{assumption}

%We note that this assumption ensures that there exist feasible solutions for both safety constraints $(c,\tau)$ and $(c^*,\tau^*)$ during exploration and planning, respectively. \jing{planning? feasibility?}

We remark that assuming the existence of a safe baseline policy is reasonable in practice. Many engineering applications already have existing solutions deployed and verified to be safe, although their reward performances may not necessarily be near-optimal. Such solutions can naturally serve as the baseline for safe RF-RL. Additionally, there are practical ways to construct safe baseline policies, e.g., via imitation learning using expert demonstrations, or via policy gradient algorithms to reduce the cost value function to be below the required safety threshold. This assumption is also widely adopted in the safe RL literature (see \Cref{sec:related} for more discussions).% \citep{CPOhttps://doi.org/10.48550/arxiv.1705.10528,zheng2020constrained,liu2021learning,amani2021safe}.  %, and $\Delta(\cdot,\cdot)$ measures the size of a safe set. 

% Hence, we adopt a standard assumption that the agent has an access of a baseline policy that satisfies the constraint to start with. \citep{CPOhttps://doi.org/10.48550/arxiv.1705.10528,CUCBhttps://doi.org/10.48550/arxiv.2001.09377,ruidahttps://doi.org/10.48550/arxiv.2106.02684,linearsafeRLhttps://doi.org/10.48550/arxiv.2106.06239} \textcolor{blue}{cite some papers that assume baseline policies}
%We note that \Cref{assm: baseline} can be satisfied by intimate learning or an expert advice. 
% In practice, there are various ways to find such a baseline policy, for example, via imitation learning using expert demonstrations, or via policy gradient algorithms to minimize the cost value function to be below the required safety threshold.

%\textcolor{blue}{describe the difference from Jin Chi's paper that uses reward-free RL as a method for safe RL; its exploration is still unsafe?}
\vspace{-0.1in}
\section{The SWEET framework}
\vspace{-0.1in}
Compared with constraint-free RF-RL, the additional safety requirements during both exploration and planning bring two main \textbf{challenges} in the design of safe RF-RL algorithms. First, in order to obtain an $\epsilon$-optimal policy for \textit{any} given reward during planning, it requires all actions to be sufficiently covered in the exploration phase. In particular, uniform action selection is one of the enablers for reward-free exploration when the state space is undesirably large \citep{NEURIPS2020_e894d787,uehara2021representation,modi2021model}.  On the other hand, the predefined safety constraint ($c,\tau$) may preclude the agent from taking certain actions in exploration, which may affect the estimation accuracy of the environment and degrade the optimality of the output policy in planning. This dilemma requires a novel design to balance safety and state-action space coverage during exploration. Second, there may exist {\it safety constraint mismatch} between exploration and planning. Intuitively, the information obtained under a given set of constraint $(c,
\tau)$ during exploration may not provide enough coverage for the optimal policy under another set of constraint $(c^*,\tau^*)$ during planning. How to design the safe exploration algorithm to handle such constraint mismatch is non-trivial.

In this section, we introduce a unified framework for safe reward-free exploration, termed as SWEET. We will show that the general framework achieves the second and third design objectives, i.e., safe exploration, and $\epsilon$-optimality and ($c^*,\tau^*$)-safety for the output policy in planning. The first design objective, i.e., low sample complexity for exploration, is dependent on the underlying MDP structure and will be investigated in Section~\ref{sec:tabular} and Section~\ref{sec:low-rank} for tabular MDPs and low-rank MDPs, respectively.
%Intuitively, reward-free safe exploration is much harder than reward-known safe exploration. 
% \jing{safety constraint mismatch between exploration and planning}

%In most reinforcement learning with large state space, uniformly action selection is one of the key step during the exploration \citep{jiang2017contextual,sun2019model,NEURIPS2020_e894d787,uehara2021representation,modi2021model}
%However, due to the safety constraint, the uniformity does not applicable. A bypass is replace the uniformly action selection by $\epsilon_0$-greedy policy, where $\epsilon_0$ is carefully designed so that the exploration does not violate the safety constraint.

% In other words, at time step $h\in \mathcal{H}$, policy $\pi'$ play the action suggested by $\pi$ with probability $1-\eo+\eo/|\Ac|$, and uniformly play any other action with probability $\eo/|\Ac|$.

\vspace{-0.1in}
\subsection{Algorithm design}
\vspace{-0.05in}
The SWEET framework relies on several key design components, namely, the {\it $(\eo,t)$-greedy policy}, the {\it approximation error function}, and the {\it empirical safe policy set}, as elaborated below.

\begin{definition}[$(\eo,t)$-greedy policy]\label{def:greedy version}
Given $\eo\in(0,1)$ and $t\in\{0,1,\cdots,H\}$, $\pi'$ is an $(\epsilon_0,t)$-greedy version of $\pi$ if there exists $ \mathcal{H}\subset[H]$ with $|\mathcal{H}|=t$ such that $\pi'_{h} = \pi_{h}$ for all $h\notin \mathcal{H} $, and 
\begin{align*}
\textstyle
\pi'_h(a|s) = (1-\eo)\pi_h(a|s) + \eo/|\Ac|, \forall h\in \mathcal{H}, s\in\Sc, a
\in \Ac.
\end{align*}
\end{definition}

Essentially, under an $(\eo,t)$-greedy version of a given policy $\pi$, the agent follows policy $\pi$ except for $t$ out of $H$ steps, at which with probability $\eo$, she takes actions uniformly at random from the state space $\Ac$. One critical property of the $(\eo,t)$-greedy policy is that, the difference between the value functions under the $(\eo,t)$-greedy policy and its original policy is bounded by $\epsilon_0 t$ for any normalized utility function (See \Cref{lemma:greedy_performance} in \Cref{appx:meta}).

The {\it approximation error function} $\mathtt{U}(\hP,\pi)$ measures the uncertainty in the model estimate $\hP$ under a policy $\pi$. Specifically, for a given MDP $\Mc^*$, $\mathtt{U}(\hP,\pi)$ upper bounds the value function difference under $\hP$ and $P^*$, i.e. $\mathtt{U}(\hP,\pi)\geq \max_{u}|V_{\hP,u}^{\pi} - V_{P^*,u}^{\pi}|$, where $u$ is a {\it normalized} utility function.

\iffalse
Then, the agent first checks the performance of $\hP$ on $\pi^0$, since \Cref{assm: baseline} asserts that $V_{P^*,c}^{\pi^0}\leq \tau - \kappa$. 
Mathematically, if $V_{\hP,c}^{\pi^0} + \mathtt{U}(\hP,\pi^0) > \tilde{\tau} - \tilde{\kappa} > \tau - \kappa$, where $\tilde{\tau}$ and $\tilde{\kappa}$ are inputs of SWEET, then the agent sets the reference policy to be the baseline policy, and continues to the next episode. 

Otherwise, the agent computes a new reference policy $\pi_r$ by
\[\textstyle
\pi_r = \arg\max_{\pi} \mathtt{U}(\hP,\pi), ~~ \text{ s.t. } ~~ V_{\hP,c}^{\pi} + \mathtt{U}(\hP,\pi)\leq \tilde{\tau}\leq\tau.
\]
\fi

%Inspired by \cite{liu2021learning}, 
The {\it empirical safe policy set}, which is critical for constructing safe exploration policies, is defined as
\begin{align}\textstyle
    \mathcal{C}_{\hP,\mathtt{U}}(\tilde{\kappa},\eo,t) = \left\{
    \begin{aligned}
    &\{\pi^0\},\quad \text{ if } V^{\pi^0}_{\hP,c} + \mathtt{U}(\hP,\pi^0) \geq \tau - \eo t - \tilde{\kappa},\\
    &\left\{\pi: V^{\pi}_{\hP,c} + \mathtt{U}(\hP,\pi)\leq \tau - \eo t \right\},\quad \text{ otherwise },
    \end{aligned}
    \right.\label{eqn:SafeSet_meta}
\end{align}
where $\tilde{\kappa},\eo$ and $t$ are constants satisfying the condition that $\tau - \eo t - \tilde{\kappa}>\tau -\kappa$. 

The intuition behind the construction of the empirical safe policy can be explained as follows~\citep{liu2021learning}: 
%Since $V_{\hP,c}^{\pi} + \mathtt{U}(\hP,\pi)$ is an upper bound of of the true value $V_{P^*,c}^{\pi}$ for any $\pi$, if $V^{\pi^0}_{\hP,c} + \mathtt{U}(\hP,\pi^0)<\tau - \eo t - \tilde{\kappa}$, it implies that 
 %$V^{\pi^0}_{P^*,c}< $when $\hP$ is
if $V^{\pi^0}_{\hP,c} + \mathtt{U}(\hP,\pi^0)\geq\tau - \eo t - \tilde{\kappa}$, it indicates that $\hP$ is not sufficiently accurate. Thus, the empirical safe policy set only contains the safe baseline policy $\pi^0$. On the other hand, if $V^{\pi^0}_{\hP,c} + \mathtt{U}(\hP,\pi^0)<\tau - \eo t - \tilde{\kappa}$, which happens when $\mathtt{U}(\hP,\pi^0)$ is sufficiently small, %as $V^{\pi^0}_{P^*,c}<\tau-\kappa<\tau - \eo t - \tilde{\kappa}$, 
it indicates that $\hP$ is sufficiently accurate on $\pi^0$. Then, we relax the constraint on $V^{\pi}_{\hP,c} + \mathtt{U}(\hP,\pi)$ from $\tau - \eo t - \tilde{\kappa}$ to $\tau - \eo t$ to include $\pi^0$ and other policies in the empirical safe policy set. Since $V_{\hP,c}^{\pi} + \mathtt{U}(\hP,\pi)$ is an upper bound of the true value $V_{P^*,c}^{\pi}$ for any $\pi$, it ensures that 
$V_{P^*,c}^{\pi}<\tau - \eo t$ for all $\pi$ included in $\Cc_{\hat{P},\mathtt{U}}$. Moreover, all $(\epsilon_0,t)$-greedy versions of such policies satisfy the safety constraint $(c,\tau)$. 

%By selecting parameters $\eo,t,\tilde{\kappa}$ satisfying $\eo t+\tilde{\kappa}<\kappa$, SWEET ensures the following properties of $\Cc_{\hP}$: (i) all policies in the empirical safety set and all behavior policies satisfy the safety constraint $(c,\tau)$; (ii) $|\Cc_{\hP}|>1$ if $\mathtt{U}(\hP,\pi^0)$ is sufficiently small; (iii) there exists an ``interior point'' $\pi^0$ in $\Cc_{\hP}$ if $|\Cc_{\hP}|>1$. 

With those salient components, SWEET proceeds as follows. At the beginning of each episode, the agent executes a set of {\it behavior policies}, which are $(\eo,t)$-greedy versions of a {\it reference policy} $\pi_r$ obtained in the previous episode. For the first episode, the reference policy would be $\pi^0$. The general construction of $\pi_r$ will be elaborated below. By collecting trajectories generated under the behavior policies, the agent updates the estimated model $\hP$ and the corresponding approximation error function $\mathtt{U}(\hP,\cdot)$.

The agent then seeks a reference policy $\pi_r$ that maximizes the approximation error $\mathtt{U}(\hP,\pi)$ within the constructed empirical safe policy set.  Intuitively, $\mathtt{U}(\hP,\pi)$ is an upper bound of certain distance of distributions over trajectories induced by $\pi$ under $\hP$ and $P^*$. Therefore, $\pi_r$ induces a distribution that captures the most uncertainty in $\hP$. {Choosing $\pi_r$ thus reduces the uncertainty in $\hP$ in a greedy fashion.} %essentially invokes an UCB-like policy   
%\citep{menard2021fast}. 
If $\mathtt{U}(\hP,\pi_r)$ is less than a termination threshold $\mathtt{T}$ defined in SWEET, it indicates that the estimated model $\hP$ is sufficiently accurate for the planning task. The exploration phase then terminates. Otherwise, the agent continues to the next episode with the new $\pi_r$.  

After termination, SWEET enters the planning phase and receives arbitrary reward functions $r^*$ and a safety constraint $(c^*,\tau^*)$. The agent utilizes $\hP$ to compute a policy $\bar{\pi}$, which maximizes $V_{\hP,r^*}^{\pi}$ subject to an empirical safety constraint $V_{\hP,c}^{\pi} + \mathtt{U}(\hP,\pi)\leq \tau^*$. \Cref{alg: meta} has the detail of SWEET.

\vspace{-0.1in}
\begin{algorithm}[H]
\caption{SWEET (\textbf{S}afe Re\textbf{W}ard Fr\textbf{E}e \textbf{E}xplora\textbf{T}ion)}
\label{alg: meta}
\begin{algorithmic}[1]
\STATE {\bfseries Input:} Reference policy $\pi_r =\pi^0$, uncertainty function $\mathtt{U}$, $\eo,t$, $\tilde{\kappa}$ and $\mathtt{T}$.%such that $\tilde{\tau}\leq \tau,\tilde{\kappa}<\kappa$, and $\tau-\kappa \leq \tilde{\tau} - \tilde{\kappa}$, uncertainty function $\mathtt{U}$.
\STATE //\texttt{ Exploration:}  
\WHILE{TRUE} 
\STATE Construct a set of $(\eo,t)$-greedy policies of $\pi_r$ (see \Cref{def:greedy version}) and use them to collect data;
\STATE {Model estimation:} Update $\hP$ using collected data;
\STATE Obtain $\pi_r = \arg\max_{\pi\in\Cc_{\hP,\mathtt{U}}(\tilde{\kappa},\eo,t)}\mathtt{U}(\hP,\pi)$ where $\Cc_{\hP,\mathtt{U}}(\tilde{\kappa},\eo,t)$ is defined in \Cref{eqn:SafeSet_meta};
\IF {$V_{\hP,c}^{\pi^0} + \mathtt{U}(\hP,\pi^0)\leq \tau-\eo t - \tilde{\kappa} $ and $\mathtt{U}(\hP,\pi_r) \leq \mathtt{T}$ } 
\STATE Output $\hP$; {\bfseries break;}
\ENDIF

\ENDWHILE

\STATE // \texttt{Planning:} 
\STATE Receive reward function $r^*$ and a safety constraint $(c^*,\tau^*)$.
\STATE {\bfseries Output:}  $\bar{\pi}=\arg\max_{\pi} V_{\hP,r^*}^{\pi}~~ \text{ s.t. }~~ V_{\hP,c^*}^{\pi} + \mathtt{U}(\hP,\pi)\leq \tau^*$.

\end{algorithmic}
\end{algorithm}
\vspace{-0.2in}

\subsection{Theoretical analysis}\label{subsec:meta_theory}
\vspace{-0.05in}
Before we present the theoretical guarantee for SWEET, we first introduce the notion of mixture policies and equivalent policies and characterize the concavity over Markov policy space. %a property that the approximation error function should satisfy.
\begin{definition}\label{def:mixture_policy}
Given two Markov policies $\pi,\pi'\in\mathcal{X}$, we use $\gamma\pi\oplus(1-\gamma)\pi'$ to denote the mixture policy that uses $\pi$ with probability $\gamma$ and uses $\pi'$ with probability $1-\gamma$ during an episode. 
\end{definition}

\begin{definition}\label{def:equivalent_policy}
Given an MDP $\Mc$, two policies, including mixture policies, are equivalent if they induce the same marginal distribution over any state-action pair $(a,s)$ in any step $h\in[H]$.
\end{definition}

By Theorem 6.1 in \cite{Altman:CMDP:1999}, for any mixture policy $\gamma\pi\oplus(1-\gamma)\pi'$, there exists an equivalent Markov policy $\pi^{\gamma}(\pi,\pi')\in\mathcal{X}$. For ease of presentation, in the following, we simply use $\pi^{\gamma}$ to denote it when the definition is clear from the context. Therefore, the Markov policy space $\mathcal{X}$ is equipped with an abstract convexity by mapping all mixture policies to their equivalent Markov policies in $\mathcal{X}$. With this convexity, we can define concave functions on $\mathcal{X}$ as follows.

\begin{definition}\label{def:concave}
A function $\mathtt{f}: \mathcal{X}\rightarrow [0,1]$ is concave and continuous on the Markov policy space $\Xc$ if for any $\pi,\pi’\in\mathcal{X}$ and $\gamma\in[0,1]$, $\mathtt{f}(\pi^\gamma) \geq \gamma \mathtt{f}(\pi) + (1-\gamma)\mathtt{f}(\pi')$, and is continuous in $\gamma\in[0,1]$.
\end{definition}

With \Cref{def:concave}, we have the following result of SWEET.

\begin{theorem}[$\epsilon$-optimality and safety guarantee of SWEET]\label{main:thm:meta_safe}
Given an MDP $\Mc^*$ and model estimate $\hP$, assume $\mathtt{U}(\hP,\pi)$ is concave and continuous over the Markov policy space $\Xc$ and %and measures the approximation error of value functions $V_{P^*,u}^{\pi}$ and $V_{\hP,u}^{\pi}$ for any $\pi$ and normalized utility function $u$, i.e.
$\big|V_{P^*,u}^{\pi} - V_{\hP,u}^{\pi}\big|\leq \mathtt{U}(\hP,\pi)$ for any {normalized} utility $u$ and policy $\pi$, and Assumption 1 holds.  
Let $\eo, t$ and $\tilde{\kappa}$ be constants that satisfy
$\eo t + \tilde{\kappa} < \kappa.$ 
Let $\mathfrak{U}  = \min\left\{\frac{\epsilon}{2}, \frac{\Delta_{\min}}{2}, \frac{\epsilon\Delta_{\min}}{5}, \frac{\tau-\eo t}{4}, \frac{\tilde{\kappa}(\Delta(c,\tau) -\eo t - \tilde{\kappa})}{4(\Delta(c,\tau) -\eo t)}\right\}$, and $\mathtt{T} \leq (\Delta(c,\tau) -\eo t)\mathfrak{U}/2$ be the termination condition of SWEET. If SWEET terminates in finite episodes, then, the following statements hold:
\vspace{-0.03in}
\begin{itemize}
[leftmargin=0.25in]\itemsep=0pt
\item [(i)] The exploration phase is safe.
\item [(ii)] The output $\bar{\pi} $ of SWEET in the planning phase is an $\epsilon$-optimal $(c^*,\tau^*)$-safe policy. % w.r.t. the reward function $r^*$ and the constraint $(c^*,\tau^*)$.
\end{itemize}

\end{theorem}

\begin{wrapfigure}{r}{0.3\textwidth}
\begin{minipage}[b]{0.3\textwidth}
    \includegraphics[width=1.8in]{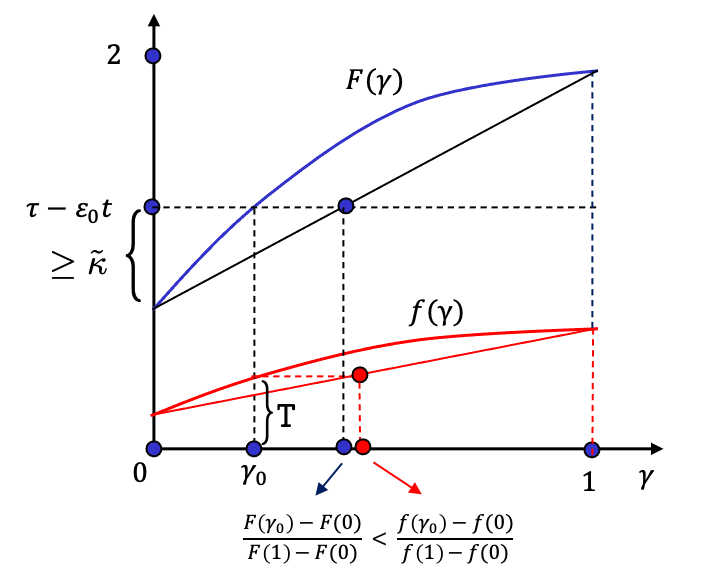}
    \vspace{-0.9cm}
\caption{{\small Illustration of the proof.}}\label{fig:proof}
\vspace{-0.9cm}
\end{minipage}
\end{wrapfigure}

The detailed proof of \Cref{main:thm:meta_safe} is deferred to Appendix~\ref{appx:meta}. We {highlight the main idea behind the analysis as follows.} While the construction of $\Cc_{\hP,\mathtt{U}}(\tilde{\kappa},\eo,t)$ ensures safe exploration, the ability for SWEET to find an $\epsilon$-optimal $(c^*,\tau^*)$-safe policy in planning relies on the concavity and continuity of $\mathtt{U}({\hP},\cdot)$. Note that when SWEET terminates, it is only guaranteed that the approximation error $\mathtt{U}(\hP,\pi)$ is upper bounded by $\mathtt{T}$ for the policies within $\Cc_{\hP,\mathtt{U}}(\tilde{\kappa},\eo,t)$. Due to a possibly different constraint in planning, it is desirable to have $\mathtt{U}(\hP,\pi)$ sufficiently small under any $\pi$, so that the agent is able to achieve the learning goal in planning with the estimated model $\hP$. 
Let $\tilde{\pi} = \arg\max_{\pi}\mathtt{U}(\hP,\pi)$, 
$\pi^\gamma$ be the equivalent Markov policy of $\gamma\tilde{\pi}\oplus(1-\gamma)\pi^0$, 
$f(\gamma) = \mathtt{U}(\hP,\pi^\gamma)$ and $g(\gamma) = V_{\hP,c}^{\pi^\gamma}$. Then, $f$ is concave and $g$ is linear in $\gamma$ by Theorem 6.1 in \cite{Altman:CMDP:1999}. Let $F(\gamma)=f(\gamma)+g(\gamma)$. The definition of $\Cc_{\hP,\mathtt{U}}(\tilde{\kappa},\eo,t)$ ensures that $F(0)\leq \tau-\eo t- \tilde{\kappa}$, and $F(\gamma)\leq \tau-\eo t$ if $\pi^\gamma$ lies in $\Cc_{\hP,\mathtt{U}}$. It suffices to consider the case when $F(1)>\tau-\epsilon_0 t$, under which we can show that both $g$ and $f$ increase with $\gamma$. Then, the concavity of $f$ and linearity of $g$ ensure that $\frac{f(\gamma)-f(0)}{f(1)-f(0)}{\geq} \frac{F(\gamma)-F(0)}{F(1)-F(0)}$, as illustrated in \Cref{fig:proof}. Let $\pi^{\gamma_0}$ be the policy under which $F(\gamma_0)=\tau-\eo t$. Then, $F(\gamma_0)-F(0)\geq \tilde{\kappa}$. Combining with the fact that $F(1)-F(0)\leq 2$, we have $f(1)\leq f(0)+(f(\gamma_0)-f(0))\frac{2}{\tilde{\kappa}}\leq 2\mathtt{T}/\tilde{\kappa}$, which provides an upper bound on $\mathtt{U}(\hP,\pi)$ for any $\pi$.

%By the linearity of $g_0$ and concavity of $g$, there must exist a lower bound \jing{?}$\gamma_{\min}$  such that it satisfies the empirical constraint\jing{which constraint}, which indicates $g(\gamma_{\min})\leq \mathtt{T}$. Therefore, by repeatedly using concavity of $\mathtt{U}$, the maxima of $\mathtt{U}(\hP,\cdot)$ is upper bounded by $g(1)\leq g(\gamma_{\min})/\gamma_{\min}\leq \mathtt{T}/\gamma_{\min}$ due to the termination condition, which controls the approximation error bound on any $\pi$.\jing{this part is still a bit confusing}

\vspace{-0.05in}
\subsection{Truncated value function}\label{subsec:truncated_value}
\vspace{-0.05in}
\Cref{main:thm:meta_safe} highlights the importance of the concave and continuous approximation error function on the Markov policy space. In the following, we introduce a prototype function, coined {\it truncated value function}, which is concave and continuous on the Markov policy space and can be used for the construction of the approximation error functions for both tabular and low-rank MDPs. 
\begin{definition}[Truncated value function]
Given an MDP $\Mc$, $\alpha>0$, and a set of (un-normalized) utility functions $u$, the truncated value function $\hV_{P,u}^{\alpha,\cdot} = \hV_{1,P,u}^{\alpha, \cdot}(s_1) : \mathcal{X}\rightarrow \Rb$ is defined as follows:
\begin{equation}\label{eqn:truncated_V}
\textstyle
\left\{\begin{aligned}
    &\bar{Q}_{h,P,u}^{\alpha,\pi}(s_h,a_h) = u(s_h,a_h) + \alpha P_h\hV_{h+1,P,u}^{\alpha, \pi}(s_h,a_h),\\
    &\hV_{h,P,u}^{\alpha,\pi}(s_h) = \min\bigg\{1, \mathop{\Eb}_{\pi}\left[\bar{Q}^{\alpha, \pi}_{h,P,u}(s_h,a_h)\right]\bigg\},
\end{aligned}\right.
\end{equation}
where $V_{H+1,P,u }^{\alpha,\pi}(s_{H+1}) = 0$ and we omit the upper index $\alpha$ for simplicity when $\alpha=1$.
\end{definition}

It is worth noting that the clipping technique is applied to the value function as opposed to the action-value function, where the latter is more conventional in the existing literature. This new method is critical for achieving the superior sample complexity in safe RF-RL, as will be elaborated later. Meanwhile, it preserves the desired concavity and continuity (see \Cref{appx:lemma:Concave} in \Cref{appx:meta}), which ensures the safety guarantee for both exploration and planning.%

\if{0}
We have the following importance observation on the truncated value function.

\begin{lemma}[Concavity of truncated value function ]\label{lemma:Concave}
 Let $\pi^{\gamma}$ be the equivalent markov policy of $\gamma{\pi} \oplus  (1-\gamma){\pi'}$ under a transition model ${P}$.
Then, 
\[\hV_{{P},u}^{\pi^{\gamma}}\geq \gamma \hV_{{P},u}^{{\pi}} + (1-\gamma) \hV_{{P},u}^{{\pi'}}.\]
Moreover, if utility function $u$ satisfies the normalization condition, then the equality holds, i.e.
\[\hV_{{P},u}^{\pi^{\gamma}} =  \gamma \hV_{{P},u}^{{\pi}} + (1-\gamma) \hV_{{P},u}^{{\pi'}}.\]
\end{lemma}
\jing{Need to introduce $\gamma{\pi} \oplus  (1-\gamma){\pi'}$ first}
The proof of \Cref{lemma:Concave} can be found in Appendix~\ref{appx:meta}.

\begin{remark}
If $\alpha= 1$ and utility function $u$ is normalized, the truncated value function $\hV_{h,P,u}^{\pi}$ degenerates to the value function $V_{h,P,u}^{\pi}$ under the same transition kernel $P$ due to the fact that all $Q$-value function is at most 1.
\end{remark}
\fi

%\yl{comment on design novelties in safe reward-free exploration}

\vspace{-0.1in}
\section{The Tabular-SWEET algorithm}\label{sec:tabular}
\vspace{-0.1in}
\subsection{Algorithm design}
\vspace{-0.05in}
%In this section, we particularize the SWEET algorithm to the tabular MDP setting where $\Sc$ and $\Ac$ are both finite (with size $S$ and $A$, respectively), and materialize the modules of model estimation and exploration policy construction, and specify the approximation error function and the parameter selection.

Under tabular MDPs, state space $\Sc$ and action space $\Ac$ are both finite (with sizes $S$ and $A$, respectively). We instantiate the modules of model estimation and exploration policy construction of the SWEET framework, and specify the approximation error function and parameter selection as follows. The details of the Tabular-SWEET algorithm is shown in \Cref{alg: Tab} in \Cref{appx:tabular}.

\textbf{Model estimation.} At each episode $n$, the agent uses $\pi^{(n-1)}$, which is the reference policy derived from the last episode $n-1$, to collect a trajectory $ {\{\small s_1^{(n)}, a_1^{(n)},\ldots, s_{H}^{(n)}, a_H^{(n)} \} }$. For that, we set $\epsilon_0$ and $t$ as $0$, i.e., it is essentially a $(0,0)$-greedy version of policy $\pi^{(n-1)}$. We note that although $\pi^{(n-1)}$ is a greedy policy, the uncertainty captured by the approximation error function $\mathtt{U}(\hP,\pi)$ will guide the agent to explore the uncertain state-action pairs and obtain sufficient coverage for the entire space. 

The agent then adds new data triples $\{s_h^{(n)},a_h^{(n)}, s_{h+1}^{(n)}\}_{h=1}^H$ to a maintained dataset $\Dc$. Let $N_h^{(n)}(s_h,a_h) = \sum_{m=1}^n\mathds{1}{\{s_h^{(n)} = s_h,~ a_h^{(n)} = a_h\}}$ and $N_h^{(n)}(s_h,a_h,s_{h+1}) = \sum_{m=1}^n \mathds{1}{\{s_h^{(n)} = s_h,~ a_h^{(n)} = a_h,~ s_{h+1}^{(n)} = s_{h+1} \}}$ be the visitation counters. The agent estimates $\hP_h^{(n)}(s_{h+1}|s_h,a_h)$ as $\frac{N_h^{(n)}(s_h,a_h,s_{h+1})}{N_h^{(n)}(s_h,a_h)}$
if ${ N_h^{(n)}(s_h,a_h)>1}$ and as $\frac{1}{S}$ otherwise.

\iffalse
{\small
\begin{align}
    &N_h^{(n)}(s_h,a_h) = \sum_{m=1}^n\mathds{1}{\{s_h^{(n)} = s_h,~ a_h^{(n)} = a_h\}}\nonumber\\
    &N_h^{(n)}(s_h,a_h,s_{h+1}) = \sum_{m=1}^n \mathds{1}{\{s_h^{(n)} = s_h,~ a_h^{(n)} = a_h,~ s_{h+1}^{(n)} = s_{h+1} \}} \\
    &\hP_h^{(n)}(s_{h+1}|s_h,a_h) = \frac{N_h^{(n)}(s_h,a_h,s_{h+1})}{N_h^{(n)}(s_h,a_h)}, \text{ if } N_h^{(n)}(s_h,a_h)>1, \text { } \hP_h^{(n)}(s_{h+1}|s_h,a_h) = \frac{1}{S}, \text{ otherwise. } \label{eqn:Est_Tab_model}
\end{align}
}
\fi

\textbf{Approximation error function.} Inspired by \cite{menard2021fast}, we adopt an uncertainty-driven virtual reward function 
   $ \hb_h^{(n)}(s_h,a_h) = \frac{\beta_0H}{N_h^{(n)}(s_h,a_h)}$ to guide the exploration, 
where $\beta_0$ is a fixed parameter.
Let $\alpha_H = 1+1/H$. Then, the approximation error function is specified as
%\begin{align}
   $ \mathtt{U}^{(n)}(\pi):=4\sqrt{\hV_{\hP^{(n)},\hb^{(n)}}^{\alpha_H,\pi}}$.
%\end{align}
According to \Cref{appx:lemma:Concave}, $ \mathtt{U}^{(n)}(\pi)$ is concave and continuous in $\pi$. Besides, as shown in {Lemma~\ref{lemma:TabValueDifference} in Appendix~\ref{appx:tabular}}, we have
$|V_{P^*,u}^{\pi} - V_{\hP,u}^{\pi}|\leq  \mathtt{U}^{(n)}(\pi)$ for any normalized utility $u$, 
i.e., $ \mathtt{U}^{(n)}(\pi)$ is a valid upper bound on the estimate error for the corresponding value function. The required properties of $\mathtt{U}$ in Theorem~\ref{main:thm:meta_safe} are thus satisfied.

%Intuitively, $\hb^{(n)}$ measures the mispecification of the estimated model $\hP^{(n)}$ from the true model $P^*$. Moreover, we recursively define the estimated(\textcolor{red}{truncated}) action value function $\hQ_{h,\hP^{(n)}, u}(s_h,a_h)$ and value function $\hV_{h,\hP^{(n)} , u}(s_h)$ starting from $\hV_{H+1,\hP^{(n)} , u}^{\pi}=0$.

\iffalse
\begin{align}
&\hat{V}_{h,\hP^{(n)}, u}^{\pi}(s_h)=\min\left\{1, \mathop{\Eb}_{\pi}\left[u(s_h,a_h)  +\sum\limits_{s_{h+1}}\hP_h^{(n)}(s_{h+1}|s_h,a_h)\hat{V}^{\pi}_{h+1,\hP^{(n)},u}(s_{h+1})\right] \right\} \label{eqn:estimated_V}. 
\end{align}
\fi

\textbf{Exploration policy.} To guarantee that the exploration is safe, we set $\tilde{\kappa} = \kappa/2$, and construct an empirical safety set $\Cc^{(n)} := \Cc_{\hP^{(n)},\mathtt{U}^{(n)}}(\kappa/2,0,0)$ (\Cref{eqn:SafeSet_meta}). %, where $\mathtt{U}^{(n)}(\pi) = 4\sqrt{\hV^{\alpha_H, \pi}_{\hP^{(n)},\hb^{(n)}}}.$
\iffalse
we define a policy set in which all policies are safe with high probability. Formally, a safe set $\mathcal{C}^{(n)}$ used in episode $n$ is defined by
\begin{align}\textstyle
    \mathcal{C}^{(n)} = \left\{
    \begin{aligned}
    &\{\pi^0\},\quad \text{ if } V^{\pi^0}_{\hP^{(n)},c} + 4\sqrt{\hat{V}^{\alpha_H, \pi^0}_{\hP^{(n)},\hb^{(n)}}} \geq \tau - \kappa/2\\
    &\left\{\pi: V^{\pi}_{\hP^{(n)},c} + 4\sqrt{\hV_{ \hP^{(n)}, \hb^{n}}^{\alpha_H,\pi}}\leq \tau \right\},\quad \text{ otherwise }
    \end{aligned}
    \right.\label{eqn:SafeSet_tab}
\end{align}

\fi
Hence, the algorithm finds a policy $\pi^{(n)}$ used for the next episode, which is in the safe set $\mathcal{C}^{(n)}$ and maximizes the truncated value function $\hV_{\hP^{(n)},\hb^{(n)}}^{\alpha_H,\pi}$. % with respect to the virtual reward function $\hb^{(n)}$. 
%Compared with SWEET, we now have $\tilde{\tau}=\tau$, $\tilde{\kappa}=\kappa/2$. Thus, the required conditions for $\tilde{\tau}$ and $\tilde{\tau}$ are satisfied.
The exploration phase stops at episode $n_{\epsilon}$ when %the truncated value function or $\mathtt{U}^{(n_{\epsilon})}$ is sufficiently small, i.e.
$\mathtt{U}^{(n_{\epsilon})}(\pi^{(n_{\epsilon})}) \leq \mathtt{T}$. The algorithm will utilize the model learned in episode $n_{\epsilon}$ to design an $\epsilon$-optimal policy with respect to arbitrary given reward $r^*$ and safety constraint $(c^*,\tau^*)$.

\vspace{-0.1in}
\subsection{Theoretical analysis}
\vspace{-0.05in}
The theoretical guarantee of Tabular-SWEET is characterized in the theorem below, whose proof can be found in Appendix~\ref{appx:tabular}.
\begin{theorem}[Sample complexity of Tabular-SWEET]\label{main:thm:tab}
Given $\epsilon,\delta\in(0,1)$, and safety constraint $(c,\tau)$, under Assumption 1, let $\mathfrak{U}  = \min\left\{\frac{\epsilon}{2}, \frac{\Delta_{\min}}{2}, \frac{\epsilon\Delta_{\min}}{5}, \frac{\tau}{4}, \frac{\kappa}{16}\right\}$, and $\mathtt{T}= \Delta(c,\tau)\mathfrak{U}/2$ be the termination condition of Tabular-SWEET. 
%If $\{\pi^{(n)}\}_{n=1}^{n_{\epsilon}}$ are exploration policies and  $\bar{\pi} $ is the return of the Tabular-SWEET algorithm provided any reward $r^*$ and constraint $(c^*,\tau^*)$ in the planning phase, 
Then, with probability at least $1-\delta$, Tabular-SWEET achieves the learning objective of safe reward-free exploration (\Cref{eqn: Safety,eqn:objective}), and
the number of trajectories collected in the exploration phase is at most  
$\tilde{O}\left(\frac{HSA(S+\log(1/\delta))}{\Delta(c,\tau)^2\mathfrak{U}^2} + \frac{HSA(S+\log(1/\delta))}{\kappa^2}\right).$

\end{theorem}

\vspace{-0.05in}
We discuss several possible scenarios and the corresponding selections of $\mathfrak{U}$ as follows. 
\vspace{-0.1in}
\begin{itemize}
[leftmargin=*]\itemsep=0pt
\setlength{\leftmargin}{-0.5in}
\item \textit{Constraint-free RF-RL.} For this case $\Delta(c,\tau) = \Delta_{\min} = \kappa = 1, $ and $c=0$. Thus, $\mathfrak{U} = \Theta(\epsilon)$ and the sample complexity is $\tilde{O}\left(HS^2A/\epsilon^2\right)$, which matches the state of the art \citep{menard2021fast}.%\jing{reference}

\item \textit{Constraint-free planning.} If only safe exploration is required, we set $\mathfrak{U} = \Theta(\min\{\epsilon, \kappa \})$, and the sample complexity scales in
$\tilde{O}\left(\frac{HS^2A}{\Delta(c,\tau)^2}(\frac{1}{\epsilon^2}  + \frac{1}{\kappa^2})\right).$
The blow-up factor $\frac{1}{\Delta(c,\tau)^2}$ depends on the {safety margin}, and the impact of baseline policy only appears in the $\epsilon$-independent term.

\item \textit{Constraint mismatch between exploration and planning.} For this case, we set $\mathfrak{U} = \Theta(\epsilon\Delta_{\min})$, and the sample complexity is at most 
$\tilde{O}\left(\frac{HS^2A}{\Delta(c,\tau)^2}(\frac{1}{\epsilon^2\Delta_{\min^2}} + \frac{1}{\kappa^2})\right)$. %where $\Delta_{\min}$ also affects the $\epsilon$-dependent term.
%The blow-up factor is also measured by the upper bound of constraints, and the cost value gap of baseline policy only appears in the $\epsilon$-independent term.

\end{itemize}

\vspace{-0.1in}
\section{The Low-rank-SWEET algorithm}\label{sec:low-rank}
\vspace{-0.1in}
\subsection{Low-rank MDP}
\vspace{-0.05in}

In this section, we present another SWEET variant for low-rank MDPs. 
\begin{definition}[Low-rank MDP \citep{jiang2017contextual,NEURIPS2020_e894d787, uehara2021representation}]\label{definition: Low_rank}
 An MDP $\Mc$ is a low-rank MDP with dimension $d\in \mathbb{N}$ if for each $h \in [H]$,  the transition kernel $P_h $ admits a $d$-dimensional decomposition, i.e., there exist two features $\phi_h: \Sc \times \Ac \rightarrow \Rb^d$ and $\mu_h: \Sc \rightarrow \Rb^d$ such that
% \begin{equation*}
$
 P_h(s_{h+1}|s_h,a_h)=\langle\phi_h(s_h,a_h),\mu_h(s_{h+1})\rangle, \forall s_h, s_{h+1} \in \Sc, a_{h} \in \Ac.
$
% \end{equation*}
Let $\phi=\{\phi_h\}_{h \in [H]}$ and $\mu=\{\mu^*_h\}_{h \in [H]}$ be the features for $P$. Then, $\|\phi^*_h(s,a)\|_2\leq 1$, $ \|\int \mu^*_h(s)g(s)ds \|_2\leq\sqrt{d}$, $\forall (s,a)\in\Sc\times\Ac$, $\forall g: \Sc \rightarrow [0,1]$. 
\end{definition}

Differently from linear MDPs \citep{wang2020reward, jin2020provably}, low-rank MDP does not assume that the features $\phi$ are known a priori. The lack of knowledge on features in fact invokes a nonlinear structure, which makes it impossible to learn a model in polynomial time if there is no assumption on features $\phi$ and $\mu$. We hence adopt the following conventional assumption~\citep{jiang2017contextual,NEURIPS2020_e894d787, uehara2021representation} from the recent studies on low-rank MDPs.

\begin{assumption}[Realizability]\label{assumption: realizability}
A learning agent can access a finite model class $\{(\Phi,\Psi)\}$ that contains the true model, i.e., $(\phi^*,\mu^*) \in \Phi\times\Psi$, where $\langle\phi^*_h(s_h,a_h),\mu^*(s_{h+1})\rangle = P^*_{h}(s_{h+1}|s_h,a_h).$
\end{assumption}
We note that finite model class assumption can be relaxed to the infinite case with bounded statistical complexity \citep{ sun2019model,NEURIPS2020_e894d787}. Then, we present the following standard oracle as a computational abstraction, which is commonly adopted in the literature \citep{NEURIPS2020_e894d787, uehara2021representation}. 
\begin{definition}[MLE oracle]\label{definition: MLE_oracle}
Given the model class $(\Phi, \Psi)$ and a dataset $\Dc$ of $(s_h, a_h, s_{h+1})$, the MLE oracle \MLE($\Dc$) takes $\Dc$ as the input and returns the following estimators as the output:
\begin{equation*}
    (\hphi_h,\hmu_h) = \MLE(\Dc) = \arg\max_{\phi_h \in \Phi, \mu_h \in \Psi}\sum_{(s_h, a_h, s_{h+1})\in\Dc} {\rm log}\langle\phi_h(s_h,a_h),\mu_h(s_{h+1})\rangle.
\end{equation*}
\end{definition}

\vspace{-0.25in}
\subsection{Algorithm design} 
\vspace{-0.05in}
%In most reinforcement learning with large state space, uniformly action selection is one of the key step during the exploration \citep{jiang2017contextual,sun2019model,NEURIPS2020_e894d787,uehara2021representation,modi2021model}
%However, due to the safety constraint, the uniformity does not applicable. A bypass is replace the uniformly action selection by $\epsilon_0$-greedy policy, where $\epsilon_0$ is carefully designed so that the exploration does not violate the safety constraint.

The instantiated SWEET algorithm, termed as Low-rank-SWEET, can be found in \Cref{alg:low-rank-sweet} in \Cref{appx:lowrank}. 
 It proceeds as follows. In each iteration $n$ during the exploration phase, the agent samples $H$ trajectories, indexed by $\{(n,h)\}_{h=1}^H$. During the $(n,h)$-th episode, the agent executes an $(\eo,2)$-greedy version of the reference policy $\pi^{(n-1)}$, where $\eo=\kappa/6$ and the $\eo$-greedy action selection only takes place at time steps $h$ and $h-1$. Denote the trajectory collected in episode $(n,h)$ as $\{s_1^{(n,h)},a_1^{(n,h)},\ldots,s_H^{(n,h)},a_H^{(n,h)}\}$. The agent maintains a dataset $\Dc_h$ for each time step $h$, which is updated through $\Dc_h^{(n)}\leftarrow \Dc_h^{(n-1)}\cup\{s_h^{(n,h)}, a_h^{(n,h)}, s_{h+1}^{(n,h)}\}$. Note that both $s_{h}^{(n,h)}$ and $a_h^{(n,h)}$ are affected by the $\eo$-greedy action selection.

\textbf{Model estimation.} Then, the agent obtains the model estimate $\hP^{(n)}$ through the MLE oracle: %based on the dataset. Specifically, %we have
\begin{align}\label{eq:p}
(\hphi_h^{(n)},\hmu_h^{(n)}) = \MLE(\Dc_h), \text{ and }\hP_h^{(n)}(s_{h+1}|s_h,a_h) = \langle \hphi_h^{(n)}(s_h,a_h), \hmu_h^{(n)}(s_{h+1})\rangle. 
\end{align}

\vspace{-0.03in}
\textbf{Approximation error function.} The algorithm will also use the estimated representation $\hphi_h^{(n)}$ to update the empirical covariance matrix $\hat{U}^{(n)}$ as
\begin{align}
    \hat{U}_h^{(n)} =& \textstyle \sum_{m=1}^{n} {\hphi}_h^{(n)}(s_h^{(m,h+1)},a_h^{(m,h+1)}) 
   \textstyle (\hphi_h^{(n)}(s_h^{(m,h+1)},a_h^{(m,h+1)}))^{\top} + \lambda_n I. \label{eq:u}
\end{align}
%where $\{s_h^{(m,h+1)},a_h^{(m,h+1)}, s_{h+1}^{(m,h+1)}\}$ is the tuple collected in the $(m,h+1)$-th episode, and step $h$. 
It is worth noting that only $a_h^{(m,h+1)}$ is affected by the $\eo$-greedy action selection, which is different from the dataset augmentation step.
Next, the agent uses both $\hphi_h^{(n)}$ and $\hat{U}^{(n)}$ to derive an exploration-driven virtual reward function as $\hat{b}_h^{(n)}(s,a) = \hat{\alpha}\|\hphi_h^{(n)}(s,a)\|_{(\hat{U}_h^{(n)})^{-1}}$ where $\|x\|_A : = \sqrt{x^\top Ax}$ and  $\hat{\alpha}$ is a pre-determined parameter. As shown in Lemma~\ref{lemma:lowrank_errorbound} in Appendix~\ref{appx:lowrank}, the approximation error can be bounded by the truncated value function with factor $\alpha=1$ up to a constant additive term, i.e. 
$|V_{P^*,u}^{\pi} - V_{\hP,u}^{\pi}|\leq  \hV_{\hP^{(n)},\hb^{(n)}}^{\pi} + \sqrt{\tilde{A}\zeta/n}:=\mathtt{U}_L^{(n)}(\pi),$  where ``$L$'' stands for ``Low-rank''.

%Compared with SWEET, we now set $\tilde{\tau}=\tau-\kappa/3$, $\tilde{\kappa}=\kappa/3$. Thus, the required conditions for $\tilde{\tau}$ and $\tilde{\tau}$ are also satisfied.

\iffalse
\begin{align}
&\hat{V}_{h,\hP^{(n)}, u}^{\pi}(s_h)=\min\left\{1, \mathop{\Eb}_{\pi}\left[u(s_h,a_h)  +\int_{s_{h+1}}\hP_h^{(n)}(s_{h+1}|s_h,a_h)\hat{V}^{\pi}_{h+1,\hP^{(n)},u}(s_{h+1})\right] \right\} \label{eqn:estimated_V_L}. 
\end{align}
The base case is that $\hat{V}_{H+1,\hP^{(n)},u}^{\pi} = 0$. 
\fi

\textbf{Exploration policy.} Based on SWEET, we choose $\Tilde{\kappa} = \kappa/3$ such that $\tilde{\kappa} + \eo t < \kappa$. Then, the algorithm defines the empirical safe policy set as $\mathcal{C}_{L}^{(n)} : = \mathcal{C}_{\hP^{(n)},\mathtt{U}_{L}^{(n)}}(\kappa/3,\kappa/6,2)$.
It then finds a reference policy $\pi^{(n)}$ in $\mathcal{C}_L^{(n)}$ that maximizes $\mathtt{U}_L^{\pi}(\pi)$, which is used for exploration at the next iteration. % by choosing a set of $(\eo,2)$-greedy policies of $\pi^{(n)}$.

%The exploration phase terminates at iteration $n_{\epsilon}$ where $2V_{\hP^{(n_{\epsilon})},\hb^{(n_{\epsilon})}}^{\pi^{(n_{\epsilon})}} + \sqrt{\tilde{A}\zeta_{n_{\epsilon}}}$ is sufficiently small. In the Planning phase, however, the algorithm maximize the value function $V_{\hP^{\epsilon},r}^{\pi}$ subject to a different constraint, i.e. $V_{\hP^{\epsilon},c}^\pi + \hV_{\hP^{\epsilon},\hb^{\epsilon}}^\pi + \sqrt{\tilde{A}\zeta_{n_{\epsilon}}}\leq \tau$.

\iffalse
We note that the safety set used in the exploration phase is different from that in the planning phase even if the constraint are the same, i.e. $(c,\tau) = (c^*,\tau^*)$. This is due to the $\eo$-greedy policy implemented in each episode and the safety constraint. Thanks to the concavity of the truncated value function, the effect of slightly shrinking the safe set on the optimization of objective function $\hV_{\hP^{(n)},\hb^{(n)}}^{(n)}$ can be controlled.
\fi

\vspace{-0.05in}
\subsection{Theoretical analysis}
\vspace{-0.05in}
We summarize the results of Low-rank-SWEET in Theorem~\ref{thm:lowrank}, and defer the proof to Appendix~\ref{appx:lowrank}. 
\begin{theorem}[Sample complexity of Low-rank-SWEET]\label{thm:lowrank}
Given $\epsilon,\delta\in(0,1)$, and safety constraint $(c,\tau)$, let $\mathfrak{U}  = \min\left\{\frac{\epsilon}{2}, \frac{\Delta_{\min}}{2}, \frac{\epsilon\Delta_{\min}}{5}, \frac{\tau}{6}, \frac{\kappa}{24}\right\}$, and $\mathtt{T} = \Delta(c,\tau)\mathfrak{U}/3$ be the termination condition of Low-rank-SWEET.
Then, under Assumption 1,2, with probability at least $1-\delta$, Low-rank-SWEET achieves the learning objective of safe reward-free exploration (\Cref{eqn: Safety,eqn:objective}) and 
the number of trajectories collected in the exploration phase is at most  
$\tilde{O}\left(\frac{H^3d^4A^2\log(1/\delta)}{\kappa^2\Delta(c,\tau)^2\mathfrak{U}^2}+  \frac{H^3d^4A^2\log(1/\delta)}{\kappa^4}\right)$.

\end{theorem}
\begin{remark}
For the constraint-free scenario, we set $\Delta(c,\tau) = \Delta_{\min} = \kappa = 1, $ $\mathfrak{U} = \Theta(\epsilon)$, and $c$ to be zero. Then, the sample complexity scales as $\tilde{O}\left(H^3d^4A^2/\epsilon^2\right)$, which outperforms the best known sample complexity of RF-RL~\citep{NEURIPS2020_e894d787,modi2021model} and even reward-known RL with computational feasibility~\citep{uehara2021representation}, all for low-rank MDPs. 
\end{remark}

\vspace{-0.1in}
\section{Related works}\label{sec:related}
\vspace{-0.1in}
\textbf{Reward-free reinforcement learning.}
Reward-free exploration is formally introduced by \cite{pmlr-v119-jin20d} for tabular MDP, where an algorithm called RF-RL-Explore is proposed, which achieves $\tilde{O}\left(H^3S^2A/\epsilon^2\right)$ sample complexity\footnote{The bound is adapted from the original result by normalizing the reward function.}. The result is then improved to $\tilde{O}\left(H^2S^2A/\epsilon^2\right)$ by \citet{kaufmann2021adaptive}. By leveraging an empirical Bernstein inequality, RF-Express \citep{menard2021fast} achieves $\tilde{O}\left(HS^2A/\epsilon^2\right)$ sample complexity, which matches the minimax lower bound in $H$ \citep{domingues2020episodic}. \citet{zhang2020nearly} considers the stationary case, and achieves $\tilde{O}\left(S^2A/\epsilon^2\right)$ sample complexity, which is nearly minimax optimal. When structured MDPs are considered, \citet{wang2020reward} studies linear MDPs and obtains $\tilde{O}\left(d^3H^4/\epsilon^2\right)$ sample complexity, where $d$ is the dimension of feature vectors. \citet{zhang2021reward} investigates linear mixture MDPs and achieves $\tilde{O}\left(H^3d^2/\epsilon^2\right)$ sample complexity.  \citet{zanette2020provably} considers a class of MDPs with {low inherent Bellman error} introduced by \citet{zanette2020learning}.  \citet{NEURIPS2020_e894d787} studies low-rank MDPs and proposes FLAMBE, whose learning objective can be translated to a reward-free learning goal with sample complexity $\tilde{O}\left(H^{22}d^7A^9/\epsilon^{10}\right)$. Subsequently, \citet{modi2021model} proposes a model-free algorithm MOFFLE for low-nonnegative-rank MDPs, for which the sample complexity scales as $\tilde{O}(\frac{H^5A^5d^3_{LV}}{\epsilon^2\eta})$, where $d_{LV}$ denotes the non-negative rank of the transition kernel. Recently, \cite{chen2022statistical} studies RF-RL with more general function approximation, but their result scales in $\tilde{O}(H^6d^3A/\epsilon^2)$  when specializes to low-rank MDPs, and cannot recover our upper bound. 

\textbf{Safe reinforcement learning.}
Safe RL is often cast in the Constrained MDP (CMDP) framework~\citep{Altman:CMDP:1999} under which the learning agent must satisfy a set of constraints \citep{efroni2020exploration,turchetta2020safe,zheng2020constrained,qiu2020upper,ding2020natural,kalagarla2020sample,liu2021learning,wei2022triple,ghosh2022provably}. However, most of the constraints considered in the existing works require the cumulative expected cost over a horizon falling below a certain threshold, which is less stringent than the episodic-wise constraint imposed in this work. Other forms of constraints such as minimizing the variance \citep{10.5555/3042573.3042784} or more generally maximizing some utility function %of state-action pairs 
\citep{ding2021provably}, have also been investigated. \citet{amani2021safe} studies safe RL with linear function approximation, where the constraint is defined using an (unknown) linear cost function of each state and action pair. In particular,  \citet{miryoosefi2021simple} utilizes a reward-free oracle to solve CMDP which, however, does not have any safety guarantee for the exploration phase. %but requires further exploration with knowledge of rewards and constraints. %It worth noting that although \citet{ding2021provably

%A standard formalism of Safe RL is constrained MDPs, where the agent optimizes an object without violating a predefined constraint.  In general, Lagrangian and primal dual methods can not guarantee that the learning process satisfy the predefined constraints \citep{bhatnagar2012online,chow2017risk,ruidahttps://doi.org/10.48550/arxiv.2106.02684,miryoosefi2021simple}. 

By assuming availability of a safe baseline policy, 
\citet{zheng2020constrained} considers a known MDP with unknown rewards and cost functions and presents C-UCRL that achieves regret $\tilde{O}(N^{3/4})$ and zero constraint violation, where $N$ is the number of episodes. 
\citet{liu2021learning} improves the result by proposing OptPess-LP, and achieves a regret of $\tilde{O}\left(H^2\sqrt{S^3AN}/\kappa\right)$
%without constraint violation
, where $\kappa$ is the cost value gap between the baseline policy and the constraint boundary. Safe baseline policy has been widely utilized in conservative RL as well, which is a special case of safe RL, as the constraint is defined in terms of the total expected reward being above a threshold~\citep{garcelon:AISTATS:2020,yang2021unified}.

\vspace{-0.1in}
\section{Conclusion}
\vspace{-0.1in}

We proposed a novel safe RF-RL framework where safety constraints are imposed during the exploration and planning phases of RF-RL. A unified algorithmic framework called SWEET was developed, which leverages an existing baseline policy to guide safe exploration. Leveraging a concave approximation error function, SWEET can achieve zero constraint violation in exploration and provably produce a near-optimal safe policy for any given reward function and safety constraint he feasible assumption in planning. We also instantiated SWEET to both tabular and low-rank MDPs, resulting in Tabular-SWEET and Low-rank-SWEET. The sample complexities of both algorithms match or outperform the state of the art in their constraint-free counterparts, proving that the safety constraint does not fundamentally impact the sample complexity of RF-RL.

\subsubsection*{Acknowledgments}
The work of R.Huang and J. Yang was supported by the U.S. National Science Foundation under the grant CNS-2003131. The work of Y. Liang was supported in part by the U.S. National Science Foundation under the grant RINGS-2148253.

\bibliographystyle{iclr2023_conference}
\bibliography{SafeRF}

\begin{thebibliography}{33}
\providecommand{\natexlab}[1]{#1}
\providecommand{\url}[1]{\texttt{#1}}
\expandafter\ifx\csname urlstyle\endcsname\relax
  \providecommand{\doi}[1]{doi: #1}\else
  \providecommand{\doi}{doi: \begingroup \urlstyle{rm}\Url}\fi

\bibitem[Agarwal et~al.(2020)Agarwal, Kakade, Krishnamurthy, and
  Sun]{NEURIPS2020_e894d787}
Alekh Agarwal, Sham Kakade, Akshay Krishnamurthy, and Wen Sun.
\newblock {FLAMBE}: Structural complexity and representation learning of low
  rank {MDPs}.
\newblock In H.~Larochelle, M.~Ranzato, R.~Hadsell, M.~F. Balcan, and H.~Lin
  (eds.), \emph{Advances in Neural Information Processing Systems}, volume~33,
  pp.\  20095--20107. Curran Associates, Inc., 2020.

\bibitem[Altman(1999)]{Altman:CMDP:1999}
E.~Altman.
\newblock \emph{Constrained Markov Decision Processes}.
\newblock Chapman and Hall, 1999.

\bibitem[Amani et~al.(2021)Amani, Thrampoulidis, and Yang]{amani2021safe}
Sanae Amani, Christos Thrampoulidis, and Lin Yang.
\newblock Safe reinforcement learning with linear function approximation.
\newblock In \emph{International Conference on Machine Learning}, pp.\
  243--253. PMLR, 2021.

\bibitem[Chen et~al.(2022)Chen, Modi, Krishnamurthy, Jiang, and
  Agarwal]{chen2022statistical}
Jinglin Chen, Aditya Modi, Akshay Krishnamurthy, Nan Jiang, and Alekh Agarwal.
\newblock On the statistical efficiency of reward-free exploration in
  non-linear rl.
\newblock \emph{arXiv preprint arXiv:2206.10770}, 2022.

\bibitem[Ding et~al.(2020)Ding, Zhang, Basar, and Jovanovic]{ding2020natural}
Dongsheng Ding, Kaiqing Zhang, Tamer Basar, and Mihailo Jovanovic.
\newblock Natural policy gradient primal-dual method for constrained {Markov}
  decision processes.
\newblock \emph{Advances in Neural Information Processing Systems}, 33, 2020.

\bibitem[Ding et~al.(2021)Ding, Wei, Yang, Wang, and
  Jovanovic]{ding2021provably}
Dongsheng Ding, Xiaohan Wei, Zhuoran Yang, Zhaoran Wang, and Mihailo Jovanovic.
\newblock Provably efficient safe exploration via primal-dual policy
  optimization.
\newblock In \emph{International Conference on Artificial Intelligence and
  Statistics}, pp.\  3304--3312. PMLR, 2021.

\bibitem[Domingues et~al.(2020)Domingues, Ménard, Kaufmann, and
  Valko]{domingues2020episodic}
Omar~Darwiche Domingues, Pierre Ménard, Emilie Kaufmann, and Michal Valko.
\newblock Episodic reinforcement learning in finite mdps: Minimax lower bounds
  revisited, 2020.

\bibitem[Efroni et~al.(2020)Efroni, Mannor, and Pirotta]{efroni2020exploration}
Yonathan Efroni, Shie Mannor, and Matteo Pirotta.
\newblock Exploration-exploitation in constrained {MDPs}.
\newblock \emph{arXiv preprint arXiv:2003.02189}, 2020.

\bibitem[Garcelon et~al.(2020)Garcelon, Ghavamzadeh, Lazaric, and
  Pirotta]{garcelon:AISTATS:2020}
Evrard Garcelon, Mohammad Ghavamzadeh, Alessandro Lazaric, and Matteo Pirotta.
\newblock Conservative exploration in reinforcement learning.
\newblock In \emph{Proceedings of the Twenty Third International Conference on
  Artificial Intelligence and Statistics}, pp.\  1431--1441, 26--28 Aug 2020.

\bibitem[Ghosh et~al.(2022)Ghosh, Zhou, and Shroff]{ghosh2022provably}
Arnob Ghosh, Xingyu Zhou, and Ness Shroff.
\newblock Provably efficient model-free constrained rl with linear function
  approximation.
\newblock \emph{arXiv preprint arXiv:2206.11889}, 2022.

\bibitem[He et~al.(2021)He, Zhou, and Gu]{he2021logarithmic}
Jiafan He, Dongruo Zhou, and Quanquan Gu.
\newblock Logarithmic regret for reinforcement learning with linear function
  approximation.
\newblock In \emph{International Conference on Machine Learning}, pp.\
  4171--4180. PMLR, 2021.

\bibitem[Jiang et~al.(2017)Jiang, Krishnamurthy, Agarwal, Langford, and
  Schapire]{jiang2017contextual}
Nan Jiang, Akshay Krishnamurthy, Alekh Agarwal, John Langford, and Robert~E
  Schapire.
\newblock Contextual decision processes with low bellman rank are
  pac-learnable.
\newblock In \emph{International Conference on Machine Learning}, pp.\
  1704--1713. PMLR, 2017.

\bibitem[Jin et~al.(2020{\natexlab{a}})Jin, Krishnamurthy, Simchowitz, and
  Yu]{pmlr-v119-jin20d}
Chi Jin, Akshay Krishnamurthy, Max Simchowitz, and Tiancheng Yu.
\newblock Reward-free exploration for reinforcement learning.
\newblock In Hal~Daumé III and Aarti Singh (eds.), \emph{Proceedings of the
  37th International Conference on Machine Learning}, volume 119 of
  \emph{Proceedings of Machine Learning Research}, pp.\  4870--4879. PMLR,
  13--18 Jul 2020{\natexlab{a}}.
\newblock URL \url{https://proceedings.mlr.press/v119/jin20d.html}.

\bibitem[Jin et~al.(2020{\natexlab{b}})Jin, Yang, Wang, and
  Jordan]{jin2020provably}
Chi Jin, Zhuoran Yang, Zhaoran Wang, and Michael~I Jordan.
\newblock Provably efficient reinforcement learning with linear function
  approximation.
\newblock In \emph{Conference on Learning Theory}, pp.\  2137--2143. PMLR,
  2020{\natexlab{b}}.

\bibitem[Kalagarla et~al.(2020)Kalagarla, Jain, and Nuzzo]{kalagarla2020sample}
Krishna~C Kalagarla, Rahul Jain, and Pierluigi Nuzzo.
\newblock A sample-efficient algorithm for episodic finite-horizon {MDP} with
  constraints.
\newblock \emph{arXiv preprint arXiv:2009.11348}, 2020.

\bibitem[Kaufmann et~al.(2021)Kaufmann, M{\'e}nard, Domingues, Jonsson,
  Leurent, and Valko]{kaufmann2021adaptive}
Emilie Kaufmann, Pierre M{\'e}nard, Omar~Darwiche Domingues, Anders Jonsson,
  Edouard Leurent, and Michal Valko.
\newblock Adaptive reward-free exploration.
\newblock In \emph{Algorithmic Learning Theory}, pp.\  865--891. PMLR, 2021.

\bibitem[Liu et~al.(2021)Liu, Zhou, Kalathil, Kumar, and Tian]{liu2021learning}
Tao Liu, Ruida Zhou, Dileep Kalathil, Panganamala Kumar, and Chao Tian.
\newblock Learning policies with zero or bounded constraint violation for
  constrained mdps.
\newblock \emph{Advances in Neural Information Processing Systems}, 34, 2021.

\bibitem[M{\'e}nard et~al.(2021)M{\'e}nard, Domingues, Jonsson, Kaufmann,
  Leurent, and Valko]{menard2021fast}
Pierre M{\'e}nard, Omar~Darwiche Domingues, Anders Jonsson, Emilie Kaufmann,
  Edouard Leurent, and Michal Valko.
\newblock Fast active learning for pure exploration in reinforcement learning.
\newblock In \emph{International Conference on Machine Learning}, pp.\
  7599--7608. PMLR, 2021.

\bibitem[Miryoosefi \& Jin(2021)Miryoosefi and Jin]{miryoosefi2021simple}
Sobhan Miryoosefi and Chi Jin.
\newblock A simple reward-free approach to constrained reinforcement learning,
  2021.

\bibitem[Modi et~al.(2021)Modi, Chen, Krishnamurthy, Jiang, and
  Agarwal]{modi2021model}
Aditya Modi, Jinglin Chen, Akshay Krishnamurthy, Nan Jiang, and Alekh Agarwal.
\newblock Model-free representation learning and exploration in low-rank mdps.
\newblock \emph{arXiv preprint arXiv:2102.07035}, 2021.

\bibitem[Qiu et~al.(2020)Qiu, Wei, Yang, Ye, and Wang]{qiu2020upper}
Shuang Qiu, Xiaohan Wei, Zhuoran Yang, Jieping Ye, and Zhaoran Wang.
\newblock Upper confidence primal-dual optimization: Stochastically constrained
  {Markov} decision processes with adversarial losses and unknown transitions.
\newblock \emph{arXiv preprint arXiv:2003.00660}, 2020.

\bibitem[Sun et~al.(2019)Sun, Jiang, Krishnamurthy, Agarwal, and
  Langford]{sun2019model}
Wen Sun, Nan Jiang, Akshay Krishnamurthy, Alekh Agarwal, and John Langford.
\newblock Model-based rl in contextual decision processes: Pac bounds and
  exponential improvements over model-free approaches.
\newblock In \emph{Conference on learning theory}, pp.\  2898--2933. PMLR,
  2019.

\bibitem[Tamar et~al.(2012)Tamar, Di~Castro, and
  Mannor]{10.5555/3042573.3042784}
Aviv Tamar, Dotan Di~Castro, and Shie Mannor.
\newblock Policy gradients with variance related risk criteria.
\newblock In \emph{Proceedings of the 29th International Conference on Machine
  Learning}, ICML'12, pp.\  1651--1658, Madison, WI, USA, 2012. Omnipress.
\newblock ISBN 9781450312851.

\bibitem[Turchetta et~al.(2020)Turchetta, Kolobov, Shah, Krause, and
  Agarwal]{turchetta2020safe}
Matteo Turchetta, Andrey Kolobov, Shital Shah, Andreas Krause, and Alekh
  Agarwal.
\newblock Safe reinforcement learning via curriculum induction.
\newblock \emph{arXiv preprint arXiv:2006.12136}, 2020.

\bibitem[Uehara et~al.(2021)Uehara, Zhang, and Sun]{uehara2021representation}
Masatoshi Uehara, Xuezhou Zhang, and Wen Sun.
\newblock Representation learning for online and offline {RL} in low-rank
  {MDPs}.
\newblock \emph{arXiv preprint arXiv:2110.04652}, 2021.

\bibitem[Wang et~al.(2020)Wang, Du, Yang, and Salakhutdinov]{wang2020reward}
Ruosong Wang, Simon~S Du, Lin~F Yang, and Ruslan Salakhutdinov.
\newblock On reward-free reinforcement learning with linear function
  approximation.
\newblock \emph{arXiv preprint arXiv:2006.11274}, 2020.

\bibitem[Wei et~al.(2022)Wei, Liu, and Ying]{wei2022triple}
Honghao Wei, Xin Liu, and Lei Ying.
\newblock Triple-q: A model-free algorithm for constrained reinforcement
  learning with sublinear regret and zero constraint violation.
\newblock In \emph{International Conference on Artificial Intelligence and
  Statistics}, pp.\  3274--3307. PMLR, 2022.

\bibitem[Yang et~al.(2021)Yang, Wu, Zhong, Garcelon, Pirotta, Lazaric, Wang,
  and Du]{yang2021unified}
Yunchang Yang, Tianhao Wu, Han Zhong, Evrard Garcelon, Matteo Pirotta,
  Alessandro Lazaric, Liwei Wang, and Simon~S Du.
\newblock A unified framework for conservative exploration.
\newblock \emph{arXiv preprint arXiv:2106.11692}, 2021.

\bibitem[Zanette et~al.(2020{\natexlab{a}})Zanette, Lazaric, Kochenderfer, and
  Brunskill]{zanette2020learning}
Andrea Zanette, Alessandro Lazaric, Mykel Kochenderfer, and Emma Brunskill.
\newblock Learning near optimal policies with low inherent bellman error.
\newblock In \emph{International Conference on Machine Learning}, pp.\
  10978--10989. PMLR, 2020{\natexlab{a}}.

\bibitem[Zanette et~al.(2020{\natexlab{b}})Zanette, Lazaric, Kochenderfer, and
  Brunskill]{zanette2020provably}
Andrea Zanette, Alessandro Lazaric, Mykel~J Kochenderfer, and Emma Brunskill.
\newblock Provably efficient reward-agnostic navigation with linear value
  iteration.
\newblock \emph{arXiv preprint arXiv:2008.07737}, 2020{\natexlab{b}}.

\bibitem[Zhang et~al.(2021)Zhang, Zhou, and Gu]{zhang2021reward}
Weitong Zhang, Dongruo Zhou, and Quanquan Gu.
\newblock Reward-free model-based reinforcement learning with linear function
  approximation.
\newblock \emph{Advances in Neural Information Processing Systems}, 34, 2021.

\bibitem[Zhang et~al.(2020)Zhang, Du, and Ji]{zhang2020nearly}
Zihan Zhang, Simon~S Du, and Xiangyang Ji.
\newblock Nearly minimax optimal reward-free reinforcement learning.
\newblock \emph{arXiv preprint arXiv:2010.05901}, 2020.

\bibitem[Zheng \& Ratliff(2020)Zheng and Ratliff]{zheng2020constrained}
Liyuan Zheng and Lillian~J Ratliff.
\newblock Constrained upper confidence reinforcement learning.
\newblock \emph{arXiv preprint arXiv:2001.09377}, 2020.

\end{thebibliography}

\newpage
\appendix

\textbf{\Large Supplementary Material}
\section{Analysis of the SWEET framework and the truncated value function}\label{appx:meta}

In this section, we first prove two supporting lemmas in \Cref{app:supportlemma}, which are useful for the proof of \Cref{main:thm:meta_safe}. Then, we provide the proof for \Cref{main:thm:meta_safe} in \Cref{app:prooftheorem1} and the proof for an important concavity property of the truncated value function in \Cref{app:proofconcavity}, which are essential for instantiating \Cref{main:thm:meta_safe} for Tabular-SWEET and Low-rank-SWEET.

%introduced in \Cref{subsec:truncated_value}.

\subsection{Supporting lemmas}\label{app:supportlemma}

\if{0}
Recall the definition of $G_{\mathcal{H}}^{\eo}\pi$.
\begin{align*}
    G_{\mathcal{H}}^{\eo}\pi(a_h|s_h) = \left\{
    \begin{aligned}
    &\frac{\eo}{|\Ac|} + (1-\eo)\pi(a_h|s_h), \text{ if } h\in\mathcal{H},\\
    &\pi(a_h,a_h), \text{ if } h\notin\mathcal{H}.
    \end{aligned}
    \right.
\end{align*}
\fi

We first show that the value function under an $(\eo,t)$-greedy policy deviates from that under the original policy by at most $\eo t$.

\begin{lemma}\label{lemma:greedy_performance}
Let $\pi'$ be an $(\eo,t)$-greedy version of policy $\pi$. Then, for an MDP with transition kernel $P$ and normalized utility function $u$, we must have
\[\left|V_{P,u}^{\pi'} - V_{P,u}^{\pi}\right|\leq \eo t.\]
\end{lemma}
\begin{proof}

First, we prove the statement for the case when $t=1$.

Assume policy $\pi'$ deviates from policy $\pi$ at step $h$, and denote $\rho_h^{\pi}(s_h)$ as the marginal distribution induced by $\pi$ under the transition kernel $P$. Let $\pi^{\Uc} = (\pi_1,\cdots,\pi_{h-1},\Uc,\pi_{h+1},\cdots,\pi_H)$, where $\Uc$ is the uniform policy over the action space, i.e., $\Uc(a_h|s_h) = 1/|\Ac|$.

Consider the equivalent Markov policy of $\eo \pi^{\Uc}\oplus(1-\eo)\pi$, denoted by $\pi^{\eo}$. By \Cref{lemma:thm6.1}, we have
\begin{align*}
    \pi'_h(a_h|s_h) &= \frac{\eo}{|\Ac|} + (1-\eo)\pi(a_h|s_h)\\
    & = \frac{\eo \rho_h^{\pi^{\Uc}}(s_h) \pi_h^{\Uc}(a_h|s_h) + (1-\eo)\rho_h^{\pi}(s_h)\pi(a_h|s_h)}{\eo \rho_h^{\pi^{\Uc}}(s_h)  + (1-\eo)\rho_h^{\pi}(s_h)}\\
    & = \pi_h^{\eo}(a_h|s_h),
\end{align*}
where the second equality follows from the fact that $\rho_h^{\pi^{\Uc}}(s_h) = \rho_h^{\pi}(s_h)$, since the first $h-1$ policies of $\pi$ and $\pi^{\Uc}$ are the same.

For any $h'\geq h+1$, we have
\begin{align*}
    \pi^{\eo}_{h'}(a_{h'}|s_{h'})     & = \frac{\eo \rho_{h'}^{\pi^{\Uc}}(s_{h'}) \pi_{h'}(a_{h'}|s_{h'}) + (1-\eo) \rho_{h'}^{\pi}(s_{h'}) \pi_{h'}(a_{h'}|s_{h'})}{\eo \rho_{h'}^{\pi^{\Uc}}(s_{h'}) + (1-\eo) \rho_{h'}^{\pi}(s_{h'})}\\
    & = \pi_{h'}(a_{h'}|s_{h'})\\
    & = \pi'_{h'}(a_{h'}|s_{h'}),
\end{align*}
where the last equality is due to the definition of $\pi'$.
Therefore, $\pi' = \pi^{\eo} $, which further yields that
\[V_{P,u}^{\pi'} = V^{\eo\pi^{\Uc}\oplus(1-\eo)\pi}_{P,u} =  \eo V_{P,u}^{\pi^{\Uc}} + (1-\eo) V_{P,u}^{\pi}.\]
Since the value function is upper bounded by 1 with normalized utility function $u$, we immediately obtain that
\[\left|V_{P,u}^{\pi'} - V_{P,u}^{\pi}\right|\leq \eo.\]

For the general case where $\pi'$ differs from $\pi$ at steps in $\Hc \subset[H]$ with $|\mathcal{H}| = t\leq H$, consider a sequence of subsets $\{\mathcal{H}_i\}_{i=1}^t$ such that $\mathcal{H}_i\subset \mathcal{H}_{i+1}$, $|\mathcal{H}_{i+1}| - |\mathcal{H}_i| = 1$, and $\mathcal{H}_t = \mathcal{H}$.
Then, we can define a sequence of policies $\{\pi^i\}_{i=1}^t$, such that $\pi^t$ is the $(\epsilon_0,i)$-greedy policy of $\pi$ that deviates from $\pi$ at steps in $\mathcal{H}_i$. 
Then, by the definition of $(\eo,t)$-greedy policy in \Cref{def:greedy version}, $\pi^{i+1}$ is an $(\eo,1)$-greedy version of policy $\pi_i$. Thus, by induction, we conclude that
\[\left|V_{P,u}^{\pi'}-V_{P,u}^{\pi}\right|\leq \sum_{i=1}^t\left|V_{P,u}^{\pi^{i}}-V_{P,u}^{\pi^{i-1}}\right|\leq \eo t,\]
where we denote $\pi'$ as $\pi^t$ and $\pi$ as $\pi^0$.
\end{proof}

The following lemma is critical for handling constraint mismatch not only between the exploration phase and the planning phase, but also between the constraint adopted for the construction of the empirical safe policy set used in exploration and the true constraint $V_{P^*,c}^{\pi}\leq \tau$.

\begin{lemma}\label{lemma:bound_max}
Consider a set $X$ in which the convex combination is defined through $\gamma x \oplus (1-\gamma) y\in X$, where $x,y\in X$ and $\gamma\in[0,1]$.  Let $f,g: X\rightarrow [0,1]$ be two functions on $X$ such that $f$ is concave and $g$ is convex, i.e. $f(\gamma x \oplus (1-\gamma) y)\geq \gamma f(x) + (1-\gamma)f(y)$ and $g(\gamma x \oplus (1-\gamma) y)\leq \gamma g(x) + (1-\gamma)g(y)$. We further assume that both $f,g$ are continuous w.r.t. $\gamma\in[0,1]$. Define an optimization problem (P) as follows:
\begin{align*}
    \text{ (P) }~~\max_{x} f(x)~~ s.t.~~ g(x) + f(x)\leq \tilde{\tau}, ~~ \tilde{\tau}\in(0,1].
\end{align*}
Assume there exists a strictly feasible solution $x_0\in X$ such that $g(x_0) + f(x_0) \leq \tilde{\tau} - \tilde{\kappa}$ where $\tilde{\kappa}\in (0,\tilde{\tau})$, and denote $x^*$ as an optimal solution to (P). 

If the optimal value of (P) is strictly less than $\tilde{\kappa}$, i.e. $f(x^*) < \tilde{\kappa}$, then 
 \[\max_{x\in X} f(x) \leq 2f(x^*)/\tilde{\kappa}.\]
\end{lemma}

\begin{proof}

Let $x_1 = \arg\max_{x\in X}f(x)$, and $x_{\gamma}=\gamma x_1\oplus (1-\gamma)x_0\in X$, which is the convex combination of $x_0$ and $x_1$. 

If $x_1$ satisfies the constraint in (P), the result is trivial. Therefore, it suffices to consider the case when $g(x_1) + f(x_1) > \tilde{\tau}$.

We first show that $g(x_1) \geq g(x_0)$ through contradiction.

Assume $g(x_1)<g(x_0)$. Then, we have 
\begin{align*}
f(x_1) &> \tilde{\tau} - g(x_1)\\
&= \tilde{\tau} - \tilde{\kappa} - g(x_1) + \tilde{\kappa}\\
&\overset{(\romannumeral1)}\geq g(x_0) + f(x_0) - g(x_1) + \tilde{\kappa}\\
& \overset{(\romannumeral2)}> \tilde{\kappa},
\end{align*}
where $(\romannumeral1)$ follows because $x_0$ is a strictly feasible solution, and $(\romannumeral2)$ follows from the assumption that $f(x_0)\in[0,1]$ and $g(x_0) > g(x_1)$.
Note that $f(x_0) < \tilde{\kappa}$, and $f(x_{\gamma})$ is a continuous function with respect to $\gamma\in[0,1]$. Thus, we can choose $\gamma_1\in(0,1)$ such that 
\begin{equation}\label{eqn:f=kappa}
f(x_{\gamma_1}) = \tilde{\kappa} \in [f(x_0), f(x_1)].
\end{equation}

In addition, by the convexity of $g$ and the assumption that $g(x_1)<g(x_0)$, we have 
\begin{equation}\label{eqn:g1<g0}
g(x_{\gamma_1}) \leq \gamma_1 g(x_1) + (1-\gamma_1) g(x_0) < g(x_0).
\end{equation}

Combining \Cref{eqn:f=kappa,eqn:g1<g0}, we have
\begin{align*}
    f(x_{\gamma_1}) + g(x_{\gamma_1}) &\leq  g(x_0) + \tilde{\kappa} \\
    & \leq g(x_0) + f(x_0) + \tilde{\kappa}\\
    &\leq \tilde{\tau},
\end{align*}
which implies that $x_{\gamma_1}$ is a feasible solution of the optimization problem (P).

Thus, by the optimality of $x^*$, we have $\tilde{\kappa} > f(x^*)\geq f(x_{\gamma_1}) = \tilde{\kappa}$, which is a contradiction. Therefore, $g(x_0)\leq g(x_1)$.

Then, let $\gamma_0$ be the solution to the following equation.
\begin{equation}\label{eqn:gamma*_tab}
\gamma_0\left(g(x_1) + f(x_1)\right) + (1-\gamma_0)\left( g(x_0) + f(x_0) \right) = \tilde{\tau}.
\end{equation}

Since $g(x_0) + f(x_0) \leq \tilde{\tau} - \tilde{\kappa}$ and $f,g\in[0,1]$, we have:
\begin{align*}
    \tilde{\tau} \leq 2\gamma_0 + \tilde{\tau}-\tilde{\kappa},
\end{align*}
which implies $\gamma_0 \geq \tilde{\kappa}/2$.

Since $f$ is concave and continuous w.r.t $\gamma$, there exists $\gamma^*\leq \gamma_0$ such that 
\begin{align}
f(x_{\gamma^*}) = \gamma_0 f(x_1) + (1-\gamma_0) f(x_0).\label{eqn:f = f_1+f_2}
\end{align}

On the other hand, due to the convexity of $g$, we have
\begin{align}
    g(x_{\gamma^*}) & \leq \gamma^* g(x_1) + (1-\gamma^*) g(x_0) \nonumber\\
    &\overset{(\romannumeral1)} \leq \gamma_0 g(x_1) + (1-\gamma_0) g(x_0),\label{eqn:g<g_1+g_2}
\end{align}
where $(\romannumeral1)$ follows from the fact that $\gamma^*\leq \gamma_0$ and $g(x_1)>g(x_0)$. 

Combining \Cref{eqn:f = f_1+f_2,eqn:g<g_1+g_2}, we have
\begin{align*}
    g(x_{\gamma^*}) + f(x_{\gamma^*}) \leq \gamma_0 g(x_1) + (1-\gamma_0) g(x_0) + \gamma_0 f(x_1) + (1-\gamma_0) f(x_0) = \tilde{\tau},
\end{align*}

which indicates that $x_{\gamma^*}$ is a feasible solution of the optimization problem (P).
Thus, by the optimality of $x^*$, \Cref{eqn:f = f_1+f_2}, and $\gamma_0\geq \tilde{\kappa}/2$, we conclude that
\begin{align*}
    \max_{x\in X}f(x) &= \frac{\gamma_0 f(x_1)}{\gamma_0}\\
    &\leq \frac{\gamma_0 f(x_1) + (1-\gamma_0) f(x_0)}{\gamma_0}\\
    & = \frac{f(x_{\gamma^*})}{\gamma_0}\leq 2f(x^*)/\tilde{\kappa}.
\end{align*}
\end{proof}

\subsection{Proof of \Cref{main:thm:meta_safe}}\label{app:prooftheorem1}
We first formally restate \Cref{main:thm:meta_safe} below and then provide the proof for this theorem. 
\begin{theorem}[Restatement of \Cref{main:thm:meta_safe}]\label{thm:meta_safe}
Given an MDP $\Mc^*$ and model estimate $\hP$, assume $\mathtt{U}(\hP,\pi)$ is concave and continuous over the Markov policy space $\Xc$ and %and measures the approximation error of value functions $V_{P^*,u}^{\pi}$ and $V_{\hP,u}^{\pi}$ for any $\pi$ and normalized utility function $u$, i.e.
$\big|V_{P^*,u}^{\pi} - V_{\hP,u}^{\pi}\big|\leq \mathtt{U}(\hP,\pi)$ for any {normalized} utility $u$ and policy $\pi$.  
Let $\eo, t$ and $\tilde{\kappa}$ be constants that satisfy
$\eo t + \tilde{\kappa} < \kappa.$ 
Let $\mathfrak{U}  = \min\left\{\frac{\epsilon}{2}, \frac{\Delta_{\min}}{2}, \frac{\epsilon\Delta_{\min}}{5}, \frac{\tau-\eo t}{4}, \frac{\tilde{\kappa}(\Delta(c,\tau) -\eo t - \tilde{\kappa})}{4(\Delta(c,\tau) -\eo t)}\right\}$, and $\mathtt{T} \leq (\Delta(c,\tau) -\eo t)\mathfrak{U}/2$ be the termination condition of SWEET. If SWEET terminates in finite episodes, then, the following statements hold:
\vspace{-0.03in}
\begin{itemize}
[leftmargin=0.25in]\itemsep=0pt
\item [(i)] The exploration phase is safe (See \Cref{eqn: Safety}).
\item [(ii)] The output $\bar{\pi} $ of SWEET in the planning phase is an $\epsilon$-optimal $(c^*,\tau^*)$-safe policy (See \Cref{eqn:objective}). % w.r.t. the reward function $r^*$ and the constraint $(c^*,\tau^*)$.
\end{itemize}

\iffalse
Suppose $\mathtt{U}(\pi)$ is concave over the Markov policy space $\mathcal{X}$ and measures the approximation error of value functions $V_{P^*,u}^{\pi}$ and $V_{\hP,u}^{\pi}$ for any $\pi$ and utility function $u$ is normalized, i.e.
\[\left|V_{P^*,u}^{\pi} - V_{\hP,u}^{\pi}\right|\leq \mathtt{U}(\pi), \forall u,\pi.\]

Let $\mathfrak{U}  = \min\left\{\frac{\epsilon}{2}, \frac{\epsilon\Delta_{\min}}{5}, \frac{\Delta_{\min}}{2}, \frac{\tilde{\tau}}{4}, \frac{\tilde{\kappa}(\Delta + \tilde{\tau} - \tau - \tilde{\kappa})}{4(\Delta + \tilde{\tau} - \tau)}\right\}$, $\mathtt{T} = (\Delta + \tilde{\tau} - \tau)\mathfrak{U}/2$ be the termination condition in SWEET.

Let $\bar{\pi} $ be the return of the SWEET algorithm provided any reward $r^*$ and constraint $(c^*,\tau^*)$ in the planning phase, then
$\bar{\pi}$ is a safe $\epsilon$-optimal policy w.r.t. the reward $r^*$ and the constraint $(c^*,\tau^*)$.
\fi

\end{theorem}

\begin{proof}
The proof consists of three steps: Step 1 shows that the exploration phase of SWEET is safe; Step 2 shows that SWEET can find an $\epsilon$-optimal policy in the planning phase for any given reward and without the constraint requirement; and Step 3 shows that SWEET can find an $\epsilon$-optimal policy in the planning phase for any given reward and under any constraint $(c^*,\tau^*)$ requirement. We next provide details for each step.

\textbf{Step 1.} This step shows that the exploration phase of SWEET is safe.

Note that the exploration policy, denoted by $\pi_b$, is an $(\eo,t)$-greedy version of the reference policy $\pi_r$, where $\pi_r$ is a solution to the following optimization problem:

\begin{align*}
    \max_{\pi\in\Cc_{\hP,\mathtt{U}}(\tilde{\kappa},\eo,t)}\mathtt{U}(\hP,\pi),
\end{align*}
where 
\begin{align*}\textstyle
    \mathcal{C}_{\hP,\mathtt{U}}(\tilde{\kappa},\eo,t) = \left\{
    \begin{aligned}
    &\{\pi^0\},\quad \text{ if } V^{\pi^0}_{\hP,c} + \mathtt{U}(\hP,\pi^0) \geq \tau - \eo t - \tilde{\kappa},\\
    &\left\{\pi: V^{\pi}_{\hP,c} + \mathtt{U}(\hP,\pi)\leq \tau - \eo t \right\},\quad \text{ otherwise}.
    \end{aligned}
    \right.
\end{align*}

If $\pi_r = \pi^0$, then by \Cref{lemma:greedy_performance},  we have
\[V_{P^*,c}^{\pi_b}\leq V_{P^*,c}^{\pi^0} + \eo t\leq \tau - \kappa + \eo t<\tau,\]
where the last inequality is due the condition $\eo t + \tilde{\kappa}<\kappa$.

If $\pi_r\neq \pi^0$, then $V^{\pi_r}_{\hP,c} + \mathtt{U}(\hP,\pi_r)\leq \tau - \eo t.$ By \Cref{lemma:greedy_performance}, we have
\begin{align*}
    V_{P^*,c}^{\pi_b}&\leq V_{P^*,c}^{\pi_r} + \eo t\\
    & \overset{(\romannumeral1)}\leq V_{\hP,c}^{\pi_r} + \mathtt{U}(\hP,\pi_r) + \eo t\\
    &\leq \tau - \eo t + \eo t = \tau,
\end{align*}
where $(\romannumeral1)$ follows from the definition of $\mathtt{U}(\hP,\pi)$. Therefore, the exploration phase is safe.

\textbf{Step 2:} This step shows that SWEET can find an $\epsilon$-optimal policy in the planning phase for any given reward $r^*$ in the constraint-free setting ($\tau^*=\infty$), i.e., the planning phase does not have a constraint requirement. 
%\yl{Is the above a reasonable statement of constraint-free case?}

Consider the Markov policy space $\mathcal{X}$ with the convex combination defined by the mixture policy $\gamma\pi\oplus(1-\gamma)\pi'$, and $g(\pi) = V_{\hP,c}^{\pi}$.

Let $\pi_r$ be the reference policy when the termination condition is satisfied. Then, by the property of $\mathtt{U}(\hP,\pi)$ and the termination condition in SWEET, the following statements hold:
\begin{itemize}
    \item $g(\pi)$ is convex (linear) and $\mathtt{U}(\pi)$ is concave on $\mathcal{X}$. Moreover, they are both continuous.
    \item The baseline policy $\pi^0\in \mathcal{X}$ and $g(\pi^0) + \mathtt{U}(\hP,\pi^0)\leq \tau - \eo t - \tilde{\kappa}$.
    \item $\pi_r = \arg\max_{\pi} \mathtt{U}(\hP, \pi)~~ \text{ s.t. } g(\pi) + \mathtt{U}(\hP,\pi)\leq \tau - \eo t$. Moreover, $\mathtt{U}(\hP,\pi_r)\leq \mathtt{T} < \tilde{\kappa}.$ 
    
\end{itemize}

Applying \Cref{lemma:bound_max} with the baseline policy $\pi^0$, we have
\begin{align}
    \max_{\pi} \mathtt{U}(\pi) \leq \frac{2}{\tilde{\kappa}}\mathtt{U}(\pi_r)\leq \frac{2\mathtt{T}}{\tilde{\kappa}}:= x_1.
\end{align}

Let $\pi^{\min} = \arg\min_{\pi} V^{\pi}_{P^*,c}$. By the definition of $\mathtt{U}$ and \Cref{assm: baseline}, we have
\begin{align}
    g(\pi^{\min}) + \mathtt{U}(\pi^{\min}) & \leq V_{P^*,c}^{\pi^{\min}} + 2\mathtt{U}(\pi^{\min}) \nonumber\\
    &\leq V_{P^*,c}^{\pi^{\min}} + \frac{4\mathtt{T}}{\tilde{\kappa}} \nonumber \\
    & = \tau - \eo t - \left(\tau - \eo t - V_{P^*,c}^{\pi^{\min}} - \frac{4\mathtt{T}}{\tilde{\kappa}}\right).\label{eqn:property of pi_min}
\end{align}

We again apply \Cref{lemma:bound_max} with the feasible solution fixed as policy $\pi^{\min}$ to conclude that 
\begin{align*}
    \max_{\pi} \mathtt{U}(\pi) \overset{(\romannumeral1)}\leq \frac{2\mathtt{T}}{\tau -\eo t - V_{P^*,c}^{\pi^{\min}} - 4\mathtt{T}/\tilde{\kappa}} \overset{(\romannumeral2)}= \frac{2\mathtt{T}}{\Delta(c,\tau) -\eo t  - 4\mathtt{T}/\tilde{\kappa}} := x_2,
\end{align*}
where $(\romannumeral1)$ follows from \Cref{eqn:property of pi_min}, and $(\romannumeral2)$ follows from the definition of $\Delta(c,\tau)$.

Continuing this process, we get a sequence $\{x_n\}$ with recursive formula \[x_{n+1} = 2\mathtt{T}/(\Delta(c,\tau) - \eo t - 2x_{n}),\] and 
\[\max_{\pi} \mathtt{U}(\pi) \leq  \inf\{x_n\}_{n=1}^{\infty}.\]

Denote $\Delta(c,\tau) - \eo t$ by $\tilde{\Delta}_c $. Then, $\mathtt{T}\leq \tilde{\Delta}_c\mathfrak{U}/2 < \frac{\tilde{\kappa}(\tilde{\Delta}_c - \tilde{\kappa})}{4},$ which implies that
\[\left|\frac{\tilde{\Delta}_c}{2} - x_1 \right|\leq \sqrt{\frac{\tilde{\Delta}_c^2}{4} - 4\mathtt{T}^2}.\]
Then, based on \Cref{lemma:sequence converge}, $\{x_n\}$ converges to
\[\frac{\tilde{\Delta}_c - \sqrt{\tilde{\Delta}_c^2 - 16\mathtt{T}}}{4}. \]

Therefore, we conclude that 
\begin{equation}\label{eqn:max_U<U}
\max_{\pi} \mathtt{U}(\pi)\leq \frac{\tilde{\Delta}_{c} - \sqrt{\tilde{\Delta}_c^2 - 16\mathtt{T}}}{4} =\frac{\mathfrak{U}\tilde{\Delta}_c}{\tilde{\Delta}_c + \sqrt{\tilde{\Delta}_c^2-16\mathtt{T}}} \leq \mathfrak{U}.
\end{equation}

Let $\tilde{\pi} = \arg\max_{\pi} V_{P^*,r^*}^{\pi}$. By the definition of $\mathtt{U}(\pi)$, we can compute the suboptimality gap of $\bar{\pi}$ as follows. 
\begin{align*}
    V_{P^*,r^*}^{\tilde{\pi}} - V_{P^*,r^*}^{\bar{\pi}}& = V_{P^*,r^*}^{\tilde{\pi}} - V^{\tilde{\pi}}_{\hP^{\epsilon},r^*} + V^{\tilde{\pi}}_{\hP^{\epsilon},r^*} - V^{\bar{\pi}}_{\hP^{\epsilon},r^*} + V^{\bar{\pi}}_{\hP^{\epsilon},r^*} -  V_{P^*,r}^{\bar{\pi}}\\
    &\overset{(\romannumeral1)}\leq \mathtt{U}(\tilde{\pi}) + \mathtt{U}(\bar{\pi})\\
    &\leq 2\max_{\pi} \mathtt{U}(\pi)\\
    &\leq 2\mathfrak{U}\leq \epsilon,
\end{align*}
where $(\romannumeral1)$ follows from the optimality of $\bar{\pi}$, i.e.  $V^{\tilde{\pi}}_{\hP^{\epsilon},r} \leq  V^{\bar{\pi}}_{\hP^{\epsilon},r} .$

\textbf{Step 3:} This step shows that SWEET can find an $\epsilon$-optimal policy in the planning phase for any given reward $r^*$ and under any constraint $(c^*,\tau^*)$. 

Let ${g_0}(\pi) = V_{\hP,c^*}^{\pi}$. Recall that 
\begin{align*}
    &\pi^* = \arg\max_{\pi} V_{P^*,r^*}^{\pi}~~ \text{ s.t. }~~ V_{P^*,c^*}^{\pi}\leq \tau^*,\\
    &\bar{\pi} = \arg\max_{\pi} V_{\hP^{\epsilon},r^*}^{\pi}~~ \text{ s.t. }~~ {g_0}(\pi) + \mathtt{U}(\pi)\leq \tau^*.
\end{align*}

If ${g_0}(\pi^*) + \mathtt{U}(\pi^*)\leq \tau^*$, by the optimality of $\bar{\pi}$, we immediately have 
\[V^{\bar{\pi}}_{\hP^{\epsilon},r^*} \geq  V^{\pi^*}_{\hP^{\epsilon},r^*}. \]

If ${g_0}(\pi^*) + \mathtt{U}(\pi^*) > \tau^*$, then by the definition of $\mathtt{U}$ and \Cref{eqn:max_U<U}, we have 
\begin{equation}\label{eqn:characterize_pi*}
\tau^* < {g_0}(\pi^*) + \mathtt{U}(\pi^*) \leq V_{P^*,c^*}^{\pi} + 2\mathtt{U}(\pi^*) \leq \tau^* + 2\mathfrak{U}.
\end{equation}

Let $\underline{\pi} = \arg\min_{\pi} V_{P^*,c^*}^{\pi}$. By \Cref{eqn:max_U<U}, we have 
\begin{equation} \label{eqn:characterize_under_pi}
{g_0}(\underline{\pi}) + \mathtt{U}(\underline{\pi})\leq V_{P^*,c^*}^{\underline{\pi}} + 2\mathtt{U}(\underline{\pi})\leq V_{P^*,c^*}^{\underline{\pi}} + 2\mathfrak{U} < \tau^*. 
\end{equation}

Let $\pi^{\gamma}$ be the Markov policy equivalent to the mixture policy $\gamma\pi^*\oplus(1-\gamma)\underline{\pi}$ under the estimated model $\hP$. Let $\Delta_{c^*} = \Delta(c^*,\tau^*)= \tau^* - V_{P^*,c^*}^{\underline{\pi}} $ and $\gamma = (\Delta_{c^*} - 3\mathfrak{U})/\Delta_{c^*}$. By the linearity of ${g_0}$ and \Cref{eqn:max_U<U}, we have
\begin{align*}
{g_0}(\pi^{\gamma})& + \mathtt{U}(\pi^{\gamma})\\
& \leq \gamma {g_0}(\pi^*) + (1-\gamma){g_0}(\underline{\pi}) + \mathfrak{U} \\
& \overset{(\romannumeral1)} \leq \gamma (\tau^* + 2\mathfrak{U}) + (1-\gamma)(V_{P^*,c^*}^{\underline{\pi}} + 2\mathfrak{U})  + \mathfrak{U} \\
& = \gamma\Delta_{c^*} + 3\mathfrak{U} + V_{P^*,c^*}^{\underline{\pi}} \\
&= \Delta_{c^*}  + V_{P^*,c^*}^{\underline{\pi}} = \tau^*,
\end{align*}
where $(\romannumeral1)$ follows from \Cref{eqn:characterize_pi*,eqn:characterize_under_pi}.
This implies that $\pi^{\gamma}$ is a feasible solution of the optimization problem solved in the planning phase. By the optimality of $\bar{\pi}$ and the linearity of $V_{\hP,r^*}^{\pi}$, we have
\begin{align}
    V^{\bar{\pi}}_{\hP,r^*} &\geq  V^{\pi^{\gamma}}_{\hP,r^*} \nonumber\\
    & \geq \gamma V^{\pi^*}_{\hP,r^*} \nonumber \\
    & \overset{(\romannumeral1)}= V^{\pi^*}_{\hP,r^*}  - \frac{3\mathfrak{U}}{\Delta_{c^*}}V^{\pi^*}_{\hP,r^*} \nonumber \\
    &\geq  V^{\pi^*}_{\hP,r^*}  - \frac{3\mathfrak{U} }{\Delta_{\min}},\label{eqn:difference of V_bar and V*}
\end{align}
where $(\romannumeral1)$ follows because $\gamma = (\Delta_{c^*}-3\mathfrak{U})/\Delta_{c^*}$, and the last inequality follows from the normalization condition and \Cref{assm: baseline}.

Recall $\mathfrak{U} \leq \frac{\Delta_{\min}\epsilon}{5}$. Therefore, the suboptimality gap under $\bar{\pi}$ can be computed as follows.
\begin{align*}
    V_{P^*,r^*}^{\pi^*}  - V_{P^*,r^*}^{\bar{\pi}} & \leq V_{\hP^{\epsilon},r^*}^{\pi^*} + \mathtt{U}(\pi^*) - V_{\hP^{\epsilon},r^*}^{\bar{\pi}} + \mathtt{U}(\bar{\pi})\\
    &\overset{(\romannumeral1)}\leq \frac{3\mathfrak{U}}{\Delta_{\min}} + 2\mathfrak{U} \\
    &\leq  \epsilon,
\end{align*}
where $(\romannumeral1)$ follows from \Cref{eqn:difference of V_bar and V*}.

\end{proof}

\subsection{Proof of concavity of the truncated value function}\label{app:proofconcavity}

In this subsection, we show that the truncated value function defined in \Cref{eqn:truncated_V} is concave and continuous on the Markov policy space $\mathcal{X}$. These are crucial properties to be used for instantiating our theorem to Tabular-SWEET and Low-rank-SWEET.

\begin{lemma}[Concavity of the truncated value function]\label{appx:lemma:Concave}
 Let $\pi^{\gamma}$ be the equivalent markov policy of $\gamma{\pi} \oplus  (1-\gamma){\pi'}$ under a transition model ${P}$.
Then, 
\[\hV_{{P},u}^{\pi^{\gamma}}\geq \gamma \hV_{{P},u}^{{\pi}} + (1-\gamma) \hV_{{P},u}^{{\pi'}}.\]
In addition, $\hV_{P,u}^{\pi^{\gamma}}$ is continuous w.r.t. $\gamma\in[0,1]$. Moreover, if the utility function $u$ satisfies the normalization condition, then the equality holds, i.e.,
\[\hV_{{P},u}^{\pi^{\gamma}} =  \gamma \hV_{{P},u}^{{\pi}} + (1-\gamma) \hV_{{P},u}^{{\pi'}}.\]
\end{lemma}

\begin{proof}
Recall that the truncated value function is defined recursively as follows:
\begin{equation*}
\left\{\begin{aligned}
    &\hQ_{h,\hP^{(n)},u}^{\pi}(s_h,a_h) = u(s,a) + \alpha\hP_h^{(n)}\hV_{h+1,\hP^{(n)},u}^{\pi}(s_h,a_h)\\
    &\hV_{h,\hP^{(n)},u}^{\pi} = \min\bigg\{1, \mathop{\Eb}_{\pi}\left[\hQ^{\pi}_{h,\hP^{(n)},u}(s,a)\right]\bigg\},
\end{aligned}\right.
\end{equation*}
with $V_{H+1,P,u }^{\alpha,\pi}(s_{H+1}) = 0$.

The continuity of the truncated value function is straightforward, since it is a composition of $H$ continuous functions. Therefore, we focus on the concavity part in the following analysis.

By \Cref{lemma:thm6.1}, the following equality holds for any utility function $u$ and time step $h$.
\begin{equation}\label{eqn: Eb_pi_gamma}
    \Eb_{P,\pi^{\gamma}}\left[u_h(s_h,a_h)\right] = \gamma \Eb_{P,\pi}\left[u_h(s_h,a_h)\right] + (1-\gamma) \Eb_{P,\pi'}\left[u_h(s_h,a_h)\right].
\end{equation}

We then prove the claim by induction.

First, we note that when $h = H+1$,

\begin{align*}
  \gamma \mathop{\Eb}_{{P},{\pi}}& \left[\hV^{{\pi}}_{H+1,{P},u}(s_{H+1})\right]  + (1-\gamma) \mathop{\Eb}_{{P},{\pi'}} \left[\hV^{{\pi'}}_{H+1,{P},u}(s_{H+1})\right] \\
  & = 0 \leq \min\left\{1, \mathop{\Eb}_{{P}, \pi^{\gamma}}\left[\hQ_{H+1,{P},u}^{\pi^{\gamma}}(s_{H+1},a_{H+1})\right]\right\}.
\end{align*}

Assume it holds for step $h+1$, i.e.,
\begin{align*}
  \gamma \mathop{\Eb}_{{P},{\pi}}& \left[\hV^{{\pi}}_{h+1,{P},u}(s_{h+1})\right]  + (1-\gamma) \mathop{\Eb}_{{P},{\pi'}} \left[\hV^{{\pi'}}_{h+1,{P},u}(s_{h+1})\right] \\
  &\leq \min\left\{1, \mathop{\Eb}_{{P}, \pi^{\gamma}}\left[\hQ_{h+1,{P},u}^{\pi^{\gamma}}(s_{h+1},a_{h+1})\right]\right\}.
\end{align*}

Then, for step $h$, by Jensen's inequality, we have

\begin{align*}
  \gamma \mathop{\Eb}_{{P},{\pi}}& \left[\hV^{{\pi}}_{h,{P},u}(s_{h})\right]  + (1-\gamma) \mathop{\Eb}_{{P},{\pi'}} \left[\hV^{{\pi'}}_{h,{P},u}(s_{h})\right] \\
  &\leq \min\left\{1 , 
  \gamma \mathop{\Eb}_{{P},{\pi}} \left[\hQ^{{\pi}}_{h,{P},u}(s_{h},a_h)\right]  + (1-\gamma) \mathop{\Eb}_{{P},{\pi'}} \left[\hQ^{{\pi'}}_{h,{P},u}(s_{h},a_h)\right]
  \right\} \\
  & = \min\left\{1 , 
  \gamma \mathop{\Eb}_{{P},{\pi}} \left[u(s_{h},a_h)\right]  + (1-\gamma) \mathop{\Eb}_{{P},{\pi'}} \left[u(s_{h},a_h)\right]
   \right.\\
  &\quad \left. + \alpha \gamma \mathop{\Eb}_{{P},{\pi}} \left[\hV^{{\pi}}_{h+1,{P},u}(s_{h+1})\right]  + (1-\gamma) \alpha\mathop{\Eb}_{{P},{\pi'}} \left[\hV^{{\pi'}}_{h+1,{P},u}(s_{h+1})\right]
  \right\} \\
  &\overset{(\romannumeral1)}\leq \min\left\{1 , 
  \mathop{\Eb}_{{P},\pi^{\gamma}} \left[u(s_{h},a_h)\right]  +  \alpha \mathop{\Eb}_{{P},\pi^{\gamma}} \left[\hQ^{\pi^{\gamma}}_{h+1,{P},u}(s_{h+1},a_{h+1})\right] 
  \right\} \\
  & = \min\left\{1, \mathop{\Eb}_{{P}, \pi^{\gamma}}\left[\hQ_{h,{P},u}^{\pi^{\gamma}}(s_{h},a_{h})\right]\right\} ,
\end{align*}
where $(\romannumeral1)$ follows from \Cref{eqn: Eb_pi_gamma} and the induction hypothesis.

Therefore, at step $h=1$, we have 
\begin{align*}
  \gamma  \hV^{{\pi}}_{{P},u}  + (1-\gamma) \hV^{{\pi'}}_{{P},u}  \leq \min\left\{1, \mathop{\Eb}_{ \pi^{\gamma}}\left[\hQ_{{P},u}^{\pi^{\gamma}}(s_{1},a_{1})\right]\right\}  = \hV^{\pi^{\gamma}}_{{P},u}.
\end{align*}

If $u$ satisfies the normalization condition, then $Q_{h,P,u}^{\pi}(s_h,a_h)\leq 1$ holds for any $\pi$ and $h$. By the definition of the truncated value function, we have $\hQ_{h,P,u}^{\pi}(s_h,a_h) = Q_{h,P,u}^{\pi}(s_h,a_h)$ and $\hV_{h,P,u}^{\pi}(s_h) = V_{h,P,u}^{\pi}(s_h)$, which implies that the truncated value function and the true value function are identical.

By \Cref{lemma:thm6.1}, $\pi^{\gamma}$ introduces the same marginal probability over any state-action pair as the mixture policy $\gamma\pi\oplus(1-\gamma)\pi'$. Therefore, when $u$ is normalized,
\begin{align*}
  \gamma  \hV^{{\pi}}_{{P},u}&  + (1-\gamma) \hV^{{\pi'}}_{{P},u} = \gamma  V^{{\pi}}_{{P},u}  + (1-\gamma) V^{{\pi'}}_{{P},u}=V^{\pi^{\gamma}}_{{P},u}=\hV^{\pi^{\gamma}}_{{P},u},
\end{align*}
which completes the proof.
\end{proof}
\vspace{0.3cm}

\section{Analysis of Tabular-SWEET}\label{appx:tabular}

In this section, we first elaborate the  Tabular-SWEET algorithm in \Cref{app:tabularsweet}. To provide the analysis for this algorithm, we first provide several supporting lemmas in \Cref{app:supportlemmatabular}, and then prove \Cref{main:thm:tab} in \Cref{app:prooftheorem2}.

\subsection{The Tabular-SWEET algorithm}\label{app:tabularsweet}
We first specify the parameters in Tabular-SWEET. The detail of Tabular-SWEET is shown in \Cref{alg: Tab}.
%We first present the detail of Tabular-SWEET, as given in \Cref{alg: Tab}. In particular, we will specify the parameters in Tabular-SWEET in order for it to have provable guarantee. 

Let $\mathfrak{U}  = \min\left\{\frac{\epsilon}{2}, \frac{\Delta_{\min}}{2}, \frac{\epsilon\Delta_{\min}}{5}, \frac{\tau}{4}, \frac{\kappa}{16}\right\}$, and $\mathtt{T}= \Delta(c,\tau)\mathfrak{U}/2$ be the termination condition of Tabular-SWEET. Let the maximum iteration number $N$ be the solution of the following equation:
\begin{equation}\label{eqn:Tab_N}
N = \frac{2^{10}e^3 30^2\beta HSA\log(N+1)}{\Delta(c,\tau)^2\mathfrak{U}^2} + \frac{2^{15}e^3\beta HSA\log(N+1)}{\kappa^2}, 
\end{equation}
where $\beta = \log(2SAH/\delta) + S\log(e(1+N))$.

Recall that the estimated model is computed by
\begin{align}\textstyle
    \hP_h^{(n)}(s_{h+1}|s_h,a_h) =
    \begin{cases}\frac{N_h^{(n)}(s_h,a_h,s_{h+1})}{N_h^{(n)}(s_h,a_h)}, \quad & \text{ if } { N_h^{(n)}(s_h,a_h)>1}, \\  \hP_h^{(n)}(s_{h+1}|s_h,a_h) = \frac{1}{S}, & \text{ otherwise, }
    \end{cases}\label{eqn:Est_Tab_model}
\end{align}
where $N_h^{(n)}(s_h,a_h)$ and $N_h^{(n)}(s_h,a_h,s_{h+1})$ denote the numbers of visits of $(s_h,a_h)$ and $(s_h,a_h,s_{h+1})$ up to $n$-th episode, respectively. 

Then, the exploration-driven virtual reward is defined as
\begin{align}\textstyle
    \hb_h^{(n)}(s_h,a_h) = \frac{\beta_0H}{N_h^{(n)}(s_h,a_h)}, \label{eqn :Rwd_Tab}
\end{align}
where $\beta_0=8\beta$. 

The \textit{approximation error bound} is a concave function of the truncated value function, defined as $\mathtt{U}^{(n)}(\pi) = 4\sqrt{\hV_{\hP^{(n)},\hb^{(n)}}^{\pi}}$.

Since $\eo = t = 0$ and $\tilde{\kappa} = \kappa/2 $, the safety set $\Cc^{(n)}$ is given by
\begin{align}\textstyle
    \mathcal{C}_{\hP,\mathtt{U}}(\tilde{\kappa},\eo,t) = \left\{
    \begin{aligned}
    &\{\pi^0\},\quad \text{ if } V^{\pi^0}_{\hP,c} + \mathtt{U}^{(n)}(\pi^0) \geq \tau - \kappa/2,\\
    &\left\{\pi: V^{\pi}_{\hP,c} + \mathtt{U}^{(n)}(\pi)\leq \tau \right\},\quad \text{ otherwise. }
    \end{aligned}
    \right.\label{eqn:SafeSet_tab}
\end{align}

\begin{algorithm}[!ht]
\caption{Tabular-SWEET}
\label{alg: Tab}
\begin{algorithmic}[1]
\STATE {\bfseries Input:} Baseline policy $\pi^0$, dataset $\Dc = \emptyset$, constants $\tau,\kappa$, $\alpha_H = \frac{H+1}{H}, \mathtt{T} = \Delta(c,\tau)\mathfrak{U}/2$. 
\STATE // \texttt{Exploration:} 
\FOR{$n=1,\ldots, N$}

\STATE Use $\pi^{(n-1)}$ to collect  {\small $\{s_1^{(n)},\ldots, a_{H}^{(n)}\}$};  $\Dc\gets \Dc\cup\{s_h^{(n)},a_h^{(n)},s_{h+1}^{(n)}\}_{h=1}^H$;

\STATE Estimate $\hP^{(n)}$ ; update $\hat{b}_h^{(n)}$, $\mathtt{U}^{(n)}(\pi)$ and the empirical safe policy set $\mathcal{C}^{(n)}$ (\Cref{eqn:Est_Tab_model,eqn :Rwd_Tab,eqn:SafeSet_tab});

\STATE Solve $\pi^{(n)} = \arg\max_{\pi\in\mathcal{C}^{(n)}}\mathtt{U}^{(n)}(\pi)$;

\IF {$\left|\mathcal{C}^{(n)}\right|>1$ and $\mathtt{U}^{(n)}(\pi^{(n)}) \leq \mathtt{T}$ } 
\STATE $\left(n_{\epsilon},\hP^{\epsilon}, \hb^{\epsilon}\right)\leftarrow \left(n, \hP^{(n)},\hb^{(n)}\right) $, {\bfseries break;}
\ENDIF

\ENDFOR

\STATE // \texttt{Planning:} 
\STATE Receive reward function $r^*$ and safety constraint $(c^*,\tau^*)$;
\STATE {\bfseries Output:}  $\bar{\pi}=\arg\max_{\pi} V_{\hP^{\epsilon},r^*}^{\pi}~~ \text{ s.t. }~~ V_{\hP^{\epsilon},c^*}^{\pi} + 4\sqrt{\hV_{\hP^{\epsilon},\hb^{\epsilon}}^{\alpha_H, \pi} }\leq \tau^*$.
\end{algorithmic}
\end{algorithm}

\subsection{Supporting Lemmas}\label{app:supportlemmatabular}

%To prove \Cref{main:thm:tab}, we follow \cite{menard2021fast} by adapting their results to the normalized utility functions and randomized policies. 
First, denote $\Var_{P_h}V_{h+1,P',u}^{\pi}(s_h,a_h)$ as the variance of value function $V_{h+1,P',u}^{\pi}(s_{h+1})$, where $s_{h+1}$ follows the distribution $P_h(\cdot|s_h,a_h)$, i.e.,
\begin{align}\label{eqn:var}
    \Var_{P_h}V_{h+1,P',u}^{\pi}(s_h,a_h) = \Eb_{P_h}\left[\left(V_{h+1,P',u}^{\pi}(s_{h+1}) -  P_hV_{h+1,P',u}^{\pi}(s_h,a_h)\right)^2| s_h,a_h\right].
\end{align}

Then, we have the following lemma.
\begin{lemma}[Lemma 3 in \cite{menard2021fast}]
\label{lemma:GoodEvent}
Let $\rho_{*,h}^{\pi}(s_h,a_h)$ be the marginal probability over state-action pair $(s_h,a_h)$ induced by policy $\pi$ under the true environment $P^*$. Suppose the utility function $u$ satisfies the normalization condition. Denote
\begin{align*}
    &\mathcal{E}_0 = \left\{\forall n,h,s_h,a_h, \KL(\hP_h^{(n)}(\cdot|s_h,a_h)||P_h^*(\cdot|s_h,a_h))\leq \frac{\beta}{N_h^{(n)}(s_h,a_h)} \right\}, \\
    &\mathcal{E}_1 = \left\{\forall n,h,s_h,a_h, N_h^{(n)}(s_h,a_h)\geq \frac{1}{2}\sum_{m=0}^{n-1} \rho_{*,h}^{\pi^{(m)}}(s_h,a_h) - \beta_1\right\}, \\
    &\mathcal{E}_2 = \bigg\{\forall n,h,s_h,a_h,  \left|\left(\hP_h^{(n)} - P^*_h\right)V_{h+1,P^*,u}^{\pi}(s_h,a_h)\right|\\
    &\quad\quad\quad\quad \leq \sqrt{\frac{2\beta_2\Var_{P_h^*}\left(V_{h+1,P^*,u}\right)(s_h,a_h)}{N_h^{(n)}(s_h,a_h)}} + \frac{3\beta_2}{N_h^{(n)}(s_h,a_h)}\bigg\},
\end{align*}
where $\beta = \log(3SAH/\delta) + S\log(8e(1+N))$, $\beta_1 = \log(3SAH/\delta)$, and $\beta_2 = \log(3SAH/\delta) + \log(8e(1+N))$. Note that $\beta \geq \beta_2\geq \beta_1$. Let $\mathcal{E} = \mathcal{E}_0 \cap \mathcal{E}_1\cap\mathcal{E}_2$. Then, we have 
 \[\Pb\left[\mathcal{E}\right]\geq 1-\delta.\]
 \end{lemma}

The following lemma shows the relationship between the visitation counters $N_h^{(n)}(s_h,a_h)$ and the pseudo-counter $\sum_{m=0}^{n-1}\rho_{*,h}^{\pi^{(m)}}(s_h,a_h)$, where $\rho_{*,h}^{\pi}(s_h,a_h)$ is the marginal distribution on $(s_h,a_h)$ induced by policy $\pi$ under true model $P^*$. 
\begin{lemma}[Lemma 7 in \cite{kaufmann2021adaptive}, Lemma 8 in \cite{menard2021fast}]\label{lemma7OfAdaptive} 
 On the event $\mathcal{E}$, we have 
 \[\min\left(\frac{\beta}{N_h^{(n)}(s_h,a_h)},1\right)\leq \frac{4\beta}{\max\left\{\sum_{m=0}^{n-1}\rho_{*,h}^{\pi^{(m)}}(s_h,a_h) , 1\right\}}.\]
\end{lemma}

In addition, we generalize Lemma 7 in \cite{menard2021fast} from deterministic policies to randomized policies. This is important for safe RL, as the optimal policy in constrained RL is possibly randomized. 
\begin{lemma}[Law of total variance with randomized policy]\label{lemma:law_variance}
Given model $P$, policy $\pi$, and normalized utility function $u$, define another {utility} function $\sigma_h(s_h,a_h)$ as 
\[\sigma_h(s_h,a_h) = \Var_{P_h}V_{h+1,P,u}^{\pi}(s_h,a_h).\]

Then, for any Markov policy $\pi$ and $h\in[H]$, the following bound holds:
\begin{equation}\label{eqn:law_of_TV}
\Eb_{\pi}\left[\left( Q_{h,P,u}^{\pi}(s_h,a_h) - \sum_{h'\geq h}u(s_{h'},a_{h'})\right)^2\bigg|s_h,a_h\right] \geq Q_{h,P,\sigma}^{\pi}(s_h,a_h).
\end{equation}

In particular, when $h=1$, we have
\begin{align*}
    1 &\geq \Eb_{\pi}\left[\left( Q_{P,u}^{\pi}(s_1,a_1) - \sum_{h\geq 1}u(s_{h},a_{h})\right)^2\bigg|s_1\right] \geq \Eb_{\pi}\left[Q_{P,\sigma}^{\pi}(s_1,a_1)|s_1\right] \\
    &= \sum_{h\geq 1}\sum_{s_h,a_h}\rho_h^{\pi}(s_h,a_h)\sigma_{h}(s_h,a_h) = \sum_{h\geq 1}\sum_{s_h,a_h}\rho_h^{\pi}(s_h,a_h)\Var_{P_h}V_{h+1,P,u}^{\pi}(s_h,a_h),
\end{align*}
where $\rho_h^{\pi}(s_h,a_h)$ is the marginal distribution over state-action pair $(s_h,a_h)$ induced by policy $\pi$ under model $P$.
\end{lemma}

\begin{proof}

%We prove the result following similar steps as in \cite{menard2021fast}. 
First, we note that the statement is trivial for $h=H+1$ since all Q-value functions are 0.

Then, we prove the result through induction. Assume that at time step $h+1$,
\[\Eb_{\pi}\left[\left( Q_{h+1,P,u}^{\pi}(s_{h+1},a_{h+1}) - \sum_{h'\geq h+1}u(s_{h'},a_{h'})\right)^2\bigg|s_{h+1},a_{h+1}\right] \geq Q_{h+1,P,\sigma}^{\pi}(s_{h+1},a_{h+1}).\]

Then, at time step $h$, the LHS of \Cref{eqn:law_of_TV} can be computed as follows.

\begin{align*}
    \Eb_{\pi}&\left[\left( Q_{h,P,u}^{\pi}(s_h,a_h) - \sum_{h'\geq h}u(s_{h'},a_{h'})\right)^2\bigg|s_h,a_h\right]\\
    &= \Eb_{\pi}\left[\left( P_hV_{h+1}^{\pi}(s_h,a_h) - \sum_{h'\geq h+1}u(s_{h'},a_{h'}) + Q_{h+1,P,u}^{\pi}(s_{h+1},a_{h+1}) - Q_{h+1,P,u}^{\pi}(s_{h+1},a_{h+1})\right)^2\bigg|s_h,a_h\right]\\
    & = \Eb_{\pi}\left[\left( Q_{h+1,P,u}^{\pi}(s_{h+1},a_{h+1}) - \sum_{h'\geq h+1}u(s_{h'},a_{h'})\right)^2\bigg|s_h,a_h\right]\\
    &\quad + \Eb_{\pi}\left[\left( Q_{h+1,P,u}^{\pi}(s_{h+1},a_{h+1}) - P_hV_{h+1}^{\pi}(s_h,a_h)\right)^2\bigg|s_h,a_h\right]\\
    &\quad + 2\Eb_{\pi}\left[\left( Q_{h+1,P,u}^{\pi}(s_{h+1},a_{h+1}) - \sum_{h'\geq h+1}u(s_{h'},a_{h'})\right)\left( Q_{h+1,P,u}^{\pi}(s_{h+1},a_{h+1}) - P_hV_{h+1}^{\pi}(s_h,a_h)\right)\bigg|s_h,a_h\right].
\end{align*}

The term within the expectation in the third term equals 0 if we further condition it on {$s_{h+1}, a_{h+1}$}, indicating that the third term is 0. Therefore, from the assumption, we have
\begin{align*}
    \Eb_{\pi}&\left[\left( Q_{h,P,u}^{\pi}(s_h,a_h) - \sum_{h'\geq h}u(s_{h'},a_{h'})\right)^2\bigg|s_h,a_h\right]\\
    &\geq \Eb_{\pi}\left[Q_{h+1,P,\sigma}(s_{h+1},a_{h+1})\bigg|s_h,a_h\right] \\
    &\quad + \mathop{\Eb
    }_{\pi}\left[\mathop{\Eb}_{a_{h+1}\sim\pi}\left[\left( Q_{h+1,P,u}^{\pi}(s_{h+1},a_{h+1}) - P_hV_{h+1}^{\pi}(s_h,a_h)\right)^2\bigg|s_{h+1}\right]\bigg|s_h,a_h\right]\\
    & \overset{(\romannumeral1)}\geq P_hV_{h+1,P,\sigma}(s_h,a_h) + \Eb_{\pi}\left[\left( V_{h+1,P,u}^{\pi}(s_{h+1}) - P_hV_{h+1}^{\pi}(s_h,a_h)\right)^2\bigg|s_h,a_h\right]\\
    & = \sigma_h(s_h,a_h) + P_hV_{h+1,P,\sigma}(s_{h},a_{h}) \\
    &= Q_{h,P,\sigma}^{\pi}(s_h,a_h),
\end{align*}
where $(\romannumeral1)$ follows from Jensen's inequality.

Thus, \Cref{eqn:law_of_TV} holds for all step $h$, and the proof is completed.

\end{proof}

The following lemma is the key to ensure that  Tabular-SWEET satisfies the termination condition.
\begin{lemma}\label{lemma:sublinear_tab}
On the event $\mathcal{E}$, the summation of $\hV^{\alpha_H, \pi^{(n)}}_{\hP^{n} , \hb^{(n)}}$ over any subset $\Nc\subset[N]$ scales in the order of $\log|\Nc|$ , i.e.
\[\sum_{n\in\Nc} \hV_{\hP^{(n)},\hb^{(n)}}^{\alpha_H,\pi^{(n)}}\leq 64e^3\beta HSA\log(1+|\Nc|).\]

\end{lemma}

\begin{proof}

First, similar to the truncated value function, we extend the definitions of value function to incorporate the additional factor $\alpha_H$. Specifically, $\forall h\in[H]$,
\begin{align*}
    &Q^{\alpha_H,\pi}_{h,P,u} = u(s_h,a_h)+\alpha_HP_hV_{h+1,P,u}^{\alpha_H,\pi},\\
    &V_{h,P,u}^{\alpha_H,\pi} = \Eb_{\pi}\left[Q^{\alpha_H,\pi}_{h,P,u}\right], 
\end{align*}
and $ V_{H+1,P,u}^{\alpha_H,\pi}=0$.

We then examine the difference between the truncated Q-value function defined with respect to model $\hP^{(n)}$ and the  Q-value function defined with respect to model $P^*$. 
\begin{align*}
    \hQ^{\alpha_H,\pi}_{h,\hP^{(n)},\hb^{(n)}}&(s_h,a_h) -  Q^{\alpha_H,\pi}_{h,P^*,\hb^{(n)}}(s_h,a_h)\\
    &  = \alpha_H \hP^{(n)}\hV^{\alpha_H,\pi}_{h+1,\hP^{(n)},\hb^{(n)}}(s_h,a_h) - \alpha_H P_h^* V^{\alpha_H,\pi}_{h+1,P^*,\hb^{(n)}}(s_h,a_h)\\
    & = \alpha_H \left(\hP_h^{(n)} - P_h^*\right)\hV^{\alpha_H,\pi}_{h+1,\hP^{(n)},\hb^{(n)}}(s_h,a_h) + \alpha_H P^*_h\left( \hV^{\alpha_H,\pi}_{h+1,\hP^{(n)},\hb^{(n)}} - V^{\alpha_H,\pi}_{h+1,P^*,\hb^{(n)}}\right)(s_h,a_h).
\end{align*}

By Lemma 10 in \cite{menard2021fast}, we bound the first term as follows.
\begin{align*}
    \left(\hP_h^{(n)} - P_h^*\right)&\hV^{\alpha_H,\pi}_{h,\hP^{(n)},\hb^{(n)}}(s_h,a_h)\\
    &\leq \min\left\{1, \sqrt{2\Var_{P_h^*}\hV^{\alpha_H,\pi}_{h+1,\hP^{(n)},\hb^{(n)}}(s_h,a_h)\frac{\beta}{N_h^{(n)}(s_h,a_h)}} + \frac{2\beta}{3N_h^{(n)}(s_h,a_h)}\right\}\\
    &\overset{(\romannumeral1)}\leq \frac{\Var_{P_h^*}\hV^{\alpha_H,\pi}_{h+1,\hP^{(n)},\hb^{(n)}}(s_h,a_h)}{H} + \min\left\{1, \frac{(2+H/2)\beta}{3N_h^{(n)}(s_h,a_h)}\right\}\\
    &\overset{(\romannumeral2)}\leq \frac{P_h^* \hV^{\alpha_H,\pi}_{h+1,\hP^{(n)},\hb^{(n)}}(s_h,a_h)}{H} + \min\left\{\hb_h^{(n)}(s_h,a_h)/8,1\right\},
\end{align*}
where $(\romannumeral1)$ follows from $\sqrt{2AB}\leq A/H + BH/2$, and $(\romannumeral2)$ is due to the truncated value function is at most 1 and $\Var(X)\leq \Eb[X]$ if $X\in[0,1]$.

Therefore, by combining the above two inequalities and taking expectation, we have,
\begin{align*}
    \hV_{h, \hP^{(n)},\hb^{(n)}}^{\alpha_H,\pi}(s_h) & \leq 
    \mathop{\Eb}_{\pi}\left[\hQ^{\alpha_H,\pi}_{h,\hP^{(n)},\hb^{(n)}}(s_h,a_h) |s_h\right]\\
    & \leq \Eb_{\pi}\left[\min\left\{\hb_h^{(n)}(s_h,a_h),1\right\}\bigg|s_h\right] 
    + \alpha_H\Eb_{\pi}\left[\min\left\{\hb_h^{(n)}(s_h,a_h)/8,1\right\}\bigg|s_h\right] \\
    &\quad + \Eb_{\pi}\left[(\alpha_H+\frac{\alpha_H}{H}) P_h^* \hV^{\alpha_H,\pi}_{h+1,\hP^{(n)},\hb^{(n)}}(s_h,a_h)\bigg|s_h\right]\\
    &\leq \Eb_{\pi}\left[2\min\left\{\hb_h^{(n)}(s_h,a_h),1\right\} + \left(1+\frac{3}{H}\right) P_h^* \hV^{\alpha_H,\pi}_{h+1,\hP^{(n)},\hb^{(n)}}(s_h,a_h)\bigg|s_h\right].
\end{align*}

Telescoping the above inequality from $h=1$ to $H$ and defining $b_h^{(n)}(s_h,a_h) = \min\left\{\hb_h^{(n)}(s_h,a_h),1\right\}$, we get 
\[\hV_{\hP^{(n)},\hb^{(n)}}^{\alpha_H,\pi}\leq V_{P^*,2b^{(n)}}^{1+3/H,\pi}\leq 2e^3V_{P^*,\hb^{(n)}}^{\pi}.\]

Therefore, if $\rho_{*,h}^{\pi^{(n)}}(s_h,a_h)$ is the marginal distribution over state-action pairs induced by exploration policy $\pi^{(n)}$ under the true model $P^*$, we have 

\begin{align*}
    \sum_{n\in\Nc}\hV_{\hP^{(n)},\hb^{(n)}}^{\alpha_H,\pi^{(n)}}\leq
    &\sum_{n\in\Nc} 2e^3V^{\pi^{(n)}}_{P^*,b^{n}}\\
    & \leq 2e^3\sum_{n\in\Nc} \sum_{h=1}^H \mathop{\Eb}_{P^*, \pi^{(n)}}\left[\min\left\{\frac{8H\beta}{N^{(n)}(s_h,a_h)},1\right\}\right]\\
    & = 2e^3\sum_{n\in\Nc} \sum_{h=1}^H \sum_{s_h,a_h} \rho_{*,h}^{\pi^{(n)}}(s_h,a_h)\min\left\{\frac{8H\beta}{N^{(n)}(s_h,a_h)},1\right\}\\
    & \overset{(\romannumeral1)}\leq  2e^3\sum_{h=1}^H \sum_{s_h,a_h} \sum_{n\in\Nc} \rho_{*,h}^{\pi^{(n)}}(s_h,a_h)\frac{8H\beta}{\max\left\{1, \sum_{m=0}^{n-1}\rho_{*,h}^{\pi_{m}}(s_h,a_h)\right\}} \\
    &\leq  16e^3H\beta\sum_{h=1}^H \sum_{s_h,a_h} \sum_{n\in\Nc} \frac{\rho_{*,h}^{\pi^{(n)}}(s_h,a_h)}{\max\left\{1, \sum_{m\in\Nc,m<n}\rho_{*,h}^{\pi_{m}}(s_h,a_h)\right\}}\\
    & \overset{(\romannumeral2)} \leq 64e^3H\beta\sum_{h=1}^H \sum_{s_h,a_h} \log\left(1 + \sum_{n\in\Nc} \rho_{*,h}^{\pi^{(n)}}(s_h,a_h)\right)\\
    & \overset{(\romannumeral3)}\leq  64e^3\beta HSA\log(1+|\Nc|),
\end{align*}
where $(\romannumeral1)$ is due to \Cref{lemma7OfAdaptive},  $(\romannumeral2)$ follows from \Cref{lemma:1/N < logN}, and $(\romannumeral3)$ follows the fact that $\rho_{*,h}^{\pi(m)}(s_h,a_h) \leq 1.$
Therefore,
\[\sum_{n\in\Nc} \hV_{\hP^{(n)},\hb^{(n)}}^{\alpha_H,\pi^{(n)}}\leq 64e^3\beta HSA\log(1+|\Nc|).\]
\end{proof}

\subsection{Proof of \Cref{main:thm:tab}}\label{app:prooftheorem2}

\begin{theorem}[Complete version of \Cref{main:thm:tab}]\label{appx:thm:tab}
Given $\epsilon,\delta\in(0,1)$, and safety constraint $(c,\tau)$, let $\mathfrak{U}  = \min\left\{\frac{\epsilon}{2}, \frac{\Delta_{\min}}{2}, \frac{\epsilon\Delta_{\min}}{5}, \frac{\tau}{4}, \frac{\kappa}{16}\right\}$, and $\mathtt{T}= \Delta(c,\tau)\mathfrak{U}/2$ be the termination condition of Tabular-SWEET. 
%If $\{\pi^{(n)}\}_{n=1}^{n_{\epsilon}}$ are exploration policies and  $\bar{\pi} $ is the return of the Tabular-SWEET algorithm provided any reward $r^*$ and constraint $(c^*,\tau^*)$ in the planning phase, 
Then, with probability at least $1-\delta$, Tabular-SWEET achieves the learning objective of safe reward-free exploration (\Cref{eqn: Safety,eqn:objective}), and
the number of trajectories collected in the exploration phase is at most  
\[O\left(\frac{\beta HSA\iota}{\Delta(c,\tau)^2\mathfrak{U}^2} + \frac{\beta HSA\iota}{\kappa^2}\right),\]
where 
$\iota = \log\left(  \frac{\beta HSA}{\Delta(c,\tau)^2\mathfrak{U}^2} + \frac{\beta HSA}{\kappa^2}   \right)$,
and $\beta = \log(2SAH/\delta) + S\log(e(1+N))$.
\end{theorem}

\begin{proof}
The proof of \Cref{main:thm:tab} mainly instantiates \Cref{main:thm:meta_safe} by verifying that (a) $\mathtt{U}^{(n)}(\pi) = 4\sqrt{\hV^{\alpha_H,\pi}_{\hP^{(n)},\hb^{(n)}}}$ is a valid approximation error bound for $V_{\hP^{(n)},u}^{\pi}$, and (b) Tabular-SWEET satisfies the termination condition within $N$ episodes. The proof consists of three steps with the first two steps verifying the above two conditions and the last step characterizes the sample complexity.

{\bf Step 1:} This step establishes the following lemma, which shows that $\mathtt{U}^{(n)}(\pi) = 4\sqrt{\hV^{\alpha_H,\pi}_{\hP^{(n)},\hb^{(n)}}}$ is a valid approximation error bound.
\begin{lemma}\label{lemma:TabValueDifference}
With $\alpha_H=1+1/H$ defined in Tabular-SWEET (\Cref{alg: Tab}), on the event $\mathcal{E}$, for any policy $\pi$ and any utility normalized function $u$,
 \begin{align*}
     \left|V_{\hP^{(n)} , u}^{\pi} - V^{\pi}_{P^*,u}\right| \leq 4\sqrt{\hV^{\alpha_H,\pi}_{\hP^{(n)},\hb^{(n)}}}.
 \end{align*}
\end{lemma}

\begin{proof}
Recall that
\[\hb_h^{(n)} = \frac{\beta_0 H}{N_h^{(n)}(s_h,a_h)},\]
where $\beta_0 = 8\beta$.

Define utility function $u^v$ as
\begin{align*}
    &u_h^{v}(s_h,a_h) = \sqrt{\Var_{\hP_h^{(n)}}V_{h+1,\hP^{(n)},u}^{\pi}(s_h,a_h)\min\left\{\frac{8\beta}{N_h^{(n)}(s_h,a_h)},\frac{1}{H}\right\}}.
\end{align*}

Following Step 1 of Lemma 1 in \cite{menard2021fast}, we get
\begin{equation}\label{eqn:step1_lemma1_menard}
\left|V_{\hP^{(n)} , u}^{\pi} - V^{\pi}_{P^*,u}\right| \leq \hV^{\alpha_H, \pi}_{\hP^{(n)},u^{v}} +\hV^{\alpha_H, \pi}_{\hP^{(n)},\hb^{(n)}}. 
\end{equation}

Next, we aim to show that 
\begin{equation}\label{eqn:step2_lemma1_menard_randomized}
\hV^{\alpha_H, \pi}_{\hP^{(n)},u^{v}} \leq e\sqrt{\hV^{\alpha_H, \pi}_{\hP^{(n)},\hb^{(n)}}}.
\end{equation}

For that, let $\hat{\rho}^{\pi}(s_h,a_h)$ be the marginal distribution over state-action pair $(s_h,a_h)$ induced by model $\hP^{(n)}$ and policy $\pi$. Note that the truncated value function is a lower bound of the corresponding value function. Thus, we can expand $\hV^{\alpha_H, \pi}_{\hP^{(n)},u^{v}}$ as follows:
\begin{align*}
    \hV^{\alpha_H, \pi}_{\hP^{(n)},u^{v}} &= \sum_{h=1}^H\sum_{s_h,a_h}\alpha_H^{h-1}\hat{\rho}^{\pi}(s_h,a_h)u_h^v(s_h,a_h)\\
    &\overset{(\romannumeral1)}\leq e\sum_{h=1}^H\sum_{s_h,a_h}\hat{\rho}^{\pi}(s_h,a_h)\sqrt{\Var_{\hP_h^{(n)}}V_{h+1,\hP^{(n)},u}^{\pi}(s_h,a_h)\min\left\{\frac{8\beta}{N_h^{(n)}(s_h,a_h)},\frac{1}{H}\right\}}\\
    & \overset{(\romannumeral2)}\leq e\sqrt{\sum_{h=1}^H\sum_{s_h,a_h}\hat{\rho}^{\pi}(s_h,a_h)\Var_{\hP_h^{(n)}}V_{h+1,\hP^{(n)},u}^{\pi}(s_h,a_h)}\sqrt{\sum_{h=1}^H\sum_{s_h,a_h}\hat{\rho}^{\pi}(s_h,a_h)\min\left\{\frac{8\beta}{N_h^{(n)}(s_h,a_h)},\frac{1}{H}\right\}},
\end{align*}
where $(\romannumeral1)$ follows from the fact that $(1+1/H)^H\leq e$ and $(\romannumeral2)$ follows from Cauchy-Schwarz inequality.

Note that in contrast to the optimistic policy, $\pi$ could be a randomized policy in general. By \Cref{lemma:law_variance}, we have 
\[\sum_{h=1}^H\sum_{s_h,a_h}\hat{\rho}^{\pi}(s_h,a_h)\Var_{\hP_h^{(n)}}V_{h+1,\hP^{(n)},u}^{\pi}(s_h,a_h)\leq 1.\]
Meanwhile, if we define $u_h^b(s_h,a_h) = \min\left\{\frac{8\beta}{N_h^{(n)}(s_h,a_h)},\frac{1}{H}\right\}$, which is obviously a normalized utility function, then, we have
\[\sum_{h=1}^H\sum_{s_h,a_h}\hat{\rho}^{\pi}(s_h,a_h)\min\left\{\frac{8\beta}{N_h^{(n)}(s_h,a_h)},\frac{1}{H}\right\} = \hV_{\hP^{(n)},u^b}^{\pi}\leq \hV_{\hP^{(n)},\hb^{(n)}}^{\alpha_H, \pi},\]
where the last inequality follows from the facts that $u_h^b(s_h,a_h)\leq \hb_h^{(n)}(s_h,a_h)$ and $\alpha_H>1$. 
Thus, we have \Cref{eqn:step2_lemma1_menard_randomized} established.

Combining \Cref{eqn:step1_lemma1_menard,eqn:step2_lemma1_menard_randomized}, we conclude that 
\[\left|V_{\hP^{(n)} , u}^{\pi} - V^{\pi}_{P^*,u}\right| \leq e\sqrt{\hV^{\alpha_H, \pi}_{\hP^{(n)},\hb^{(n)}}} + \hV^{\alpha_H, \pi}_{\hP^{(n)},\hb^{(n)}} \overset{(\romannumeral1)}\leq (1+e)\sqrt{\hV^{\alpha_H, \pi}_{\hP^{(n)},\hb^{(n)}}}\leq 4\sqrt{\hV^{\alpha_H, \pi}_{\hP^{(n)},\hb^{(n)}}}, \]

where $(\romannumeral1)$ is due to the fact that the truncated value function is at most 1.
\end{proof}

{\bf Step 2:} This step establishes the following lemma, which shows that Tabular-SWEET will terminate within $N$ episodes.
\begin{lemma}\label{lemma:stop_complexity}
On the event $\mathcal{E}$, there exists $n_{\epsilon}\in[N]$ such that $\left|\mathcal{C}^{(n_{\epsilon})}\right|>1$ and  $\hV_{\hP^{(n_{\epsilon})},\hb^{(n_{\epsilon})}}^{\alpha_H, \pi^{(n_{\epsilon})}} \leq \mathtt{T}^2/16$, where $N$ is defined in \Cref{eqn:Tab_N}, and $\mathtt{T}$ is defined in Tabular-SWEET (\Cref{alg: Tab}).
\end{lemma}

\begin{proof}
Denote $\Nc_0 = \{n\in[N] : \pi^{(n)} = \pi^0 \}$.
We first prove $\Nc_0$ is finite. Note that for all $n\in\Nc_0$, $V_{\hP^{(n)},c}^{\pi^0} + 4\sqrt{\hV_{\hP^{(n)},\hb^{(n)}}^{\alpha_H, \pi^{0}} } \geq \tau - \kappa/2.$

By \Cref{lemma:TabValueDifference,lemma:sublinear_tab}, we have 
\begin{align*}
    |\Nc_0|\kappa/2 &\leq \sum_{n\in\Nc_0} \left(V_{\hP^{(n)},c}^{\pi^0} + 4\sqrt{\hV_{\hP^{(n)},\hb^{(n)}}^{\alpha_H, \pi^{0}}} - V^{\pi^0}_{P^*,c}\right)\\
    &\leq \sum_{n\in\Nc_0} 8 \sqrt{\hV_{\hP^{(n)},\hb^{(n)}}^{\alpha_H, \pi^{0}}}\\
    &\leq 64\sqrt{e^3|\Nc_0|\beta HSA\log(1+N)},
\end{align*}
where the last inequality is due to Cauchy-Schwarz inequality.
Therefore, $|\Nc_0|\leq \frac{2^{14}e^3\beta HSA\log(N+1)}{\kappa^2}.$

Then, we prove \Cref{lemma:stop_complexity} by contradiction. Assume $\hV_{\hP^{(n)},\hb^{(n)}}^{\pi^{(n)}}>\mathtt{T}^2/16, \forall n\in[N]\backslash\Nc_0$. According to \Cref{lemma:sublinear_tab}, we have
\begin{align*}
    (N-|\Nc_0|)\mathtt{T}^2/16 & < \sum_{n\in[N]/\Nc_0} \hV_{\hP^{(n)},\hb^{(n)}}^{\alpha_H, \pi^{(n)}}\\
    &\leq 64e^3\beta HSA\log(N-|\Nc_0|+1),
\end{align*}
which implies that 
\[N< \frac{2^{10}e^3\beta H SA\log(N+1)}{\mathtt{T}^2} + \frac{2^{14}e^3\beta HSA\log(N+1)}{\kappa^2}.\]
This contradicts with the condition that $N=\frac{2^{10}e^3\beta H SA\log(N+1)}{\mathtt{T}^2} + \frac{2^{14}e^3\beta HSA\log(N+1)}{\kappa^2}$. Therefore, by noting that $\mathtt{U}^{(n)}(\pi) = 4\sqrt{\hV_{\hP^{(n)},\hb^{(n)}}^{\alpha_H, \pi^{(n)}}},$ there exists $n_{\epsilon}\in[N]$ such that the exploration phase under Tabular-SWEET terminates.
\end{proof}

{\bf Step 3:} This step analyzes the sample complexity as follows.

On the event $\mathcal{E}$, since $\mathtt{T} = \Delta(c,\tau)\mathfrak{U}/2$, by \Cref{lemma:stop_complexity}, the sample complexity is at most 
\[N = \frac{2^{8}e^3 30^2\beta HSA\log(N+1)}{\Delta(c,\tau)^2\mathfrak{U}^2} + \frac{2^{15}e^3\beta HSA\log(N+1)}{\kappa^2}.\]

Note that $n = c_0\log(c_1 n)$ implies $n\leq 2c_0 \log(c_0c_1)$. Thus,
\[N = O\left(\frac{\beta HSA\iota}{\Delta(c,\tau)^2\mathfrak{U}^2} + \frac{\beta HSA\iota}{\kappa^2}\right),\]
where 
\[\iota = \log\left(  \frac{\beta HSA}{\Delta(c,\tau)^2\mathfrak{U}^2} + \frac{\beta HSA}{\kappa^2}   \right).\]
Therefore, Tabular-SWEET terminates in finite episodes. 

Besides, $\mathtt{U}(\pi) = 4\sqrt{\hV^{\alpha_H, \pi}_{\hP^{\epsilon}, \hb^{\epsilon}}}$. On the event $\mathcal{E}$, by \Cref{appx:lemma:Concave}, \Cref{lemma:TabValueDifference} and the concavity of $\sqrt{x}$ , $\mathtt{U}(\pi)$ is an approximation error function under $\hP^{\epsilon}$, and is concave and continuous on $\mathcal{X}$.

We further note that $\frac{\kappa/2(\Delta(c,\tau) - \kappa/2)}{4\Delta(c,\tau)} \geq \frac{\kappa}{16}$ due to the condition $\Delta(c,\tau)\geq \kappa$, which indicates that $\mathtt{T} = \Delta(c,\tau)\mathfrak{U}/2 $ satisfies the requirement in \Cref{thm:meta_safe}.

Therefore, by \Cref{thm:meta_safe}, we conclude that with probability at least $1-\delta$, the exploration phase of Tabular-SWEET is safe and $\bar{\pi}$ is an $\epsilon$-optimal policy subject to the safety constraint $(c^*,\tau^*).$   

\end{proof}

%\newpage
\section{Analysis of Low-rank-SWEET}\label{appx:lowrank}

%In this section, we present the detail and analysis of Low-rank-SWEET algorithm.

In this section, we first elaborate the  Low-rank-SWEET algorithm in \Cref{app:lowranksweet}. To provide the analysis for this algorithm, we first provide several supporting lemmas in \Cref{app:supportlemmalowrank}, and then prove \Cref{thm:lowrank} in \Cref{app:prooftheorem3}.

\subsection{The Low-rank-SWEET algorithm}\label{app:lowranksweet}

We first specify the parameters adopted in Low-rank-SWEET, which is presented in \Cref{alg:low-rank-sweet}.

Let $\mathfrak{U}  = \min\left\{\frac{\epsilon}{2}, \frac{\Delta_{\min}}{2}, \frac{\epsilon\Delta_{\min}}{5}, \frac{\tau}{6}, \frac{\kappa}{24}\right\}$, and $\mathtt{T} = \Delta(c,\tau)\mathfrak{U}/3$ be the termination condition of Low-rank-SWEET. Recall that we set $\eo = \kappa/6$, $t=2$, and $\tilde{\kappa} = \kappa/3$. 

%Moreover, we use $G_{\mathcal{H}}^{\eo}\pi$ to denote an $(\eo, |\mathcal{H}|)$-greedy version of $\pi$. Specifically, $G_{\mathcal{H}}^{\eo}\pi$ follows $\pi$ at time step $t\in\mathcal{H}$ with probability $1-\eo$ and takes uniformly action selection with the probability $\eo$. 

We define the maximum number of iterations $N$ as
\begin{equation}\label{eqn:N_lowrank}
    N = \frac{2^{10}\beta_3 H^2d^4A^2\zeta^2}{\kappa^2\mathtt{T}^2} + \frac{2^{12}\cdot3^2\beta_3 H^2d^4A^2\zeta^2}{\kappa^4},
\end{equation}
where $\zeta = \log\left(2|\Phi||\Psi|NH/\delta\right)$, and $\beta_3$ is defined in \Cref{lemma:high_prob_event}.

Besides, we set $\tilde{A}=A/\epsilon_0$ and $\hat{\alpha} = 5\sqrt{\beta_3\zeta(\tilde{A} + d^2)}$.

For ease of exposition, we introduce the following notation for an $(\epsilon_0,t)$-greedy version of policy $\pi$, denoted as $G_{\mathcal{H}}^{\eo}\pi$, as follows:
\begin{align}\label{eqn:greedy}
    G_{\mathcal{H}}^{\eo}\pi(a_h|s_h) = \left\{
    \begin{array}{ll}
    \frac{\eo}{|\Ac|} + (1-\eo)\pi(a_h|s_h), & \text{ if } h\in\mathcal{H},\\
    \pi(a_h,a_h), &\text{ if } h\notin\mathcal{H}.
    \end{array}
    \right.
\end{align}
where $|\Hc|=t$. Intuitively, $G_{\mathcal{H}}^{\eo}\pi$ follows $\pi$ at time step $h\in\mathcal{H}$ with probability $1-\eo$ and takes uniformly action selection with the probability $\eo$. 

We also define 
\begin{align}
	\Pi_n = \mbox{Unif}\{\pi^{(m)}\}_{m=0}^{n-1},\label{eqn:uniform_policy}
\end{align}
where $\mbox{Unif}\left(\mathcal{X}_0\right)$ is a mixture policy that uniformly chooses one policy from the policy set $\mathcal{X}_0\subset\mathcal{X}$. We use $G_{\mathcal{H}}^{\eo}\Pi_n$ to denote the $(\eo, |\Hc|)$-greedy version of $\Pi_n$.

\begin{algorithm}[!ht]
\caption{Low-rank-SWEET}
\label{alg:low-rank-sweet}
\begin{algorithmic}[1]
\STATE {\bfseries Input:} Constants $\eo = \kappa/6, \tilde{A} = A/\eo$, and termination condition $\mathtt{T}.$
\STATE // \texttt{Exploration:} 
\FOR{$n=1,\ldots, N$}

\FOR{$h=1,..., H$}
\STATE Execute policy $G_{h-1,h}^{\eo}\pi^{(n-1)}$ and collect data {\small $s_1^{(n,h)}, a_1^{(n,h)},\ldots, s_{H}^{(n,h)}, a_H^{(n,h)}$}
\STATE $\Dc_h^n\leftarrow D_h^n\cup\{s_h^{(n,h)}, a_h^{(n,h)}, s_{h+1}^{(n,h)}\}$
\ENDFOR

\STATE Learn $(\hphi_h^{(n)},\hmu_h^{(n)}) = \mbox{MLE}(\Dc_h^n)$, and update $\hP^{(n)}$ according to \Cref{eq:p}
\STATE Update empirical covariance matrix $\hat{U}_h^{(n)}$ according to \Cref{eq:u}

\STATE Define exploration-driven reward function $\hat{b}_h^{(n)}(\cdot,\cdot) = \min\left\{\hat{\alpha}\|\hphi_h^{(n)}(\cdot,\cdot)\|_{(\hat{U}_h^{(n)})^{-1}},1\right\}$.
\STATE Define $\mathtt{U}^{(n)}(\pi) = \hV_{\hP^{(n)},\hb^{(n)}}^{\pi^{(n)}} + \sqrt{\tilde{A}\zeta/n}$
\begin{align}\textstyle
    \mathcal{C}_{L}^{(n)} = \left\{
    \begin{aligned}
    &\{\pi^0\},\quad \text{ if } V^{\pi^0}_{\hP,c} + \mathtt{U}^{(n)}(\pi^0) \geq \tau - 2\kappa/3,\\
    &\left\{\pi: V^{\pi}_{\hP,c} + \mathtt{U}^{(n)}(\pi)\leq \tau - \kappa/3 \right\},\quad \text{ otherwise. }
    \end{aligned}
    \right.\label{eqn:SafeSet_L}
\end{align}
\STATE Solve $\pi^{(n)} = \arg\max_{\pi\in\mathcal{C}_L^{(n)}} \mathtt{U}^{(n)}(\pi)$, where $\mathcal{C}_L^{(n)}$ is defined in \Cref{eqn:SafeSet_L}.
\IF{$\left|\mathcal{C}_{L}^{(n)}\right|>1$ and $\mathtt{U}^{(n)}(\pi^{(n)})\leq \mathtt{T}$}
\STATE $\left(n_{\epsilon},\hP^{\epsilon}, \hb^{\epsilon}\right)\leftarrow \left(n, \hP^{(n)},\hb^{(n)}\right) $, {\bfseries break}
\ENDIF

\ENDFOR

\STATE // \texttt{Planning:} 
\STATE Receive reward function $r^*$ and constraint $(c^*,\tau^*)$, 
\STATE {\bfseries Output:}  $\bar{\pi}=\arg\max_{\pi} V_{\hP^{\epsilon},r^*}^\pi~~ \text{ s.t. }~~ V_{\hP^{\epsilon},c^*}^{\pi} + \hV_{\hP^{\epsilon},\hb^{\epsilon}}^{\pi} +\sqrt{\tilde{A}\zeta_{n_{\epsilon}}}\leq \tau^*.$

\end{algorithmic}
\end{algorithm}

%The next lemma provides concentration of the bonus term. See Lemma 39 in \citet{zanette2020learning} for the version of fixed $\phi$ and Lemma 11 in \citet{uehara2021representation}.

%\textbf{Proof overview of \Cref{thm:lowrank}:} Similar to the analysis in \Cref{appx:tabular}, we aim to verify that $\mathtt{U}^{(n)}(\pi)$ is a valid approximation error bound and Low-rank-SWEET terminates within $N$ iterations. First, we present several supporting lemmas.

\subsection{Supporting Lemmas}\label{app:supportlemmalowrank}
We first characterize the following high probability event.

\begin{lemma}\label{lemma:high_prob_event}
Denote 
\begin{align}
    &f_h^{(n)}(s_h,a_h) = \left\|P_h^*(\cdot|s_h,a_h)-\hP_h^{(n)}(\cdot|s_h,a_h)\right\|_1,\\
    &U_{h,\phi}^{(n)} = n\mathop{\Eb}_{s_h\sim(P^*,\Pi_n )\atop a_h\sim G_h^{\eo}\Pi_n}\left[\phi(s_h,a_h)(\phi(s_h,a_h))^\top\right] + \lambda I,\label{eqn:expected_U_matrix}
\end{align}
where $\lambda = \beta_3 d\log(2NH|\Phi|/\delta))$ and $\beta_3 = O(1)$.

Define events $\mathcal{E}_0$ and $\mathcal{E}_1$ as 
\begin{align*}
    &\mathcal{E}_0 = \bigg\{\forall n\in[N],h\in[H],s_h\in\Sc,a_h\in\Ac,  \mathop{\Eb}_{s_{h}\sim (P^*,G_{h-1}^{\eo}\Pi_n)\atop a_h\sim G_h^{\eo}\Pi_n}\left[f_h^{(n)}(s_h,a_h)^2\right]\leq \zeta/n\bigg\}, \\
    &\mathcal{E}_1 = \bigg\{\forall n\in[N],h\in[H], s_h\in\Sc, a_h\in\Ac, \\
    &\quad\quad \frac{1}{5} \left\|\hphi_{h-1}^{(n)}(s,a)\right\|_{(U_{h-1,\hphi}^{(n)})^{-1}} \leq \left\|\hphi_{h-1}^{(n)}(s,a)\right\|_{(\hat{U}_{h-1}^{(n)})^{-1}} \leq 3 \left\|\hphi_{h-1}^{(n)}(s,a)\right\|_{(U_{h-1,\hphi}^{(n)})^{-1}}\bigg\},
\end{align*}
where $\zeta = \log\left(2|\Phi||\Psi|NH/\delta\right)$. 

Denote $\mathcal{E}:= \mathcal{E}_0\cap\mathcal{E}_1$. Then, $\Pb[\mathcal{E}]\geq 1-\delta$.

\end{lemma}

\begin{proof}

By \Cref{coro:MLE} in \Cref{appx:auxiliary}, we have $\Pb[\mathcal{E}_0]\geq 1-\delta/2$.  Further, by Lemma 39 in \citet{zanette2020learning} for the version of fixed $\phi$ and Lemma 11 in \citet{uehara2021representation}, we have $\Pb[\mathcal{E}_1]\geq 1-\delta/2$. Therefore, $\Pb[\mathcal{E}]\geq 1-\delta$. 
\end{proof}

Based on \Cref{lemma:high_prob_event}, we can bound the exploration-driven virtual reward in Low-rank-SWEET as follows.
\begin{corollary}\label{coro:concentration on b}
Given that the event $\Ec$ occurs, the following inequality holds for any $n\in[N],h\in[H], s_h\in\Sc,a_h\in\Ac$:
\begin{align*}
    \min\left\{\frac{\hat{\alpha}} {5}\left\|\hphi_{h}^{(n)}(s_{h},a_{h})\right\|_{(U_{h,\hphi}^{(n)})^{-1}},1\right\} \leq \hb_h^{(n)}(s_h,a_h) \leq  3\hat{\alpha} \left\|\hphi_{h}^{(n)}(s_{h},a_{h})\right\|_{(U_{h,\hphi}^{(n)})^{-1}},
\end{align*}
where $\hat{\alpha} = 5\sqrt{\beta_3\zeta(\tilde{A} + d^2)}$.
\end{corollary}
\begin{proof}
Recall $\hb_h^{(n)}(s_h,a_h)=\min\left\{\hat{\alpha}\left\|\hphi_{h}^{(n)}(s,a)\right\|_{(\hat{U}_{h}^{(n)})^{-1}},1\right\}$. Applying \Cref{lemma:high_prob_event}, we can immediately obtain the result.
\end{proof}

The following lemma summarizes Lemmas 12 and 13 in \cite{uehara2021representation} and generalizes them to $\eo$-greedy policies.  We provide the proof for completeness.

\begin{lemma}\label{lemma:Step_Back} 
Let $P_{h-1} = \langle\phi_{h-1},\mu_{h-1}\rangle$ be a low-rank MDP model, and $\Pi$ be an arbitrary and possibly a mixture policy. Define an expected Gram matrix as follows:
\[M_{h-1,\phi} = 
\lambda I + n\mathop{\Eb}_{s_{h-1}\sim (P^*,\Pi) \atop a_{h-1}\sim \Pi}\left[\phi_{h-1}(s_{h-1},a_{h-1})\left(\phi_{h-1}(s_{h-1},a_{h-1})\right)^\top\right].
\] 

%be an expected gram matrix, where $\Pi = G_{h-1}^{\eo}\Pi_n$ when $M_{h,\phi} = U^{(n)}_{h,\phi}$, and $\Pi = \Pi_n$ when $M_{h,\phi} = W_{h,\phi}^{(n)}$. Note that $\phi_h\in\{\hphi_h^{(n)},\phi_h^*\}$. 
Further, let $f_{h-1}(s_{h-1},a_{h-1})$ be the total variation distance between $P_{h-1}^*$ and $P_{h-1}$ at time step $h-1$. Suppose $g \in \mathcal{S} \times \mathcal{A} \rightarrow \mathbb{R}$ is bounded by $B\in(0,\infty)$, i.e., $\|g\|_\infty \leq B$. Then, for $h \geq 2$, any policy $\pi_h$,
\begin{align*}
    \mathop{\Eb}_{s_h \sim P_{h-1} \atop a_h \sim \pi_h}&[g(s_h,a_h)|s_{h-1},a_{h-1}]  \\
    &\leq \left\|\phi_{h-1}(s_{h-1},a_{h-1})\right\|_{(M_{h-1,\phi})^{-1}} \times\nonumber\\
    &\quad
    \sqrt{n\tilde{A}\mathop{\Eb}_{s_{h}\sim(P^*, \Pi)\atop a_h \sim G_h^{\eo}\Pi }[g^2(s_h,a_h)]+\lambda dB^2 + nB^2\mathop{\Eb}_{s_{h-1}\sim(P^*,\Pi)\atop a_{h-1}\sim\Pi }\left[f_{h-1}(s_{h-1},a_{h-1})^2\right]}.
\end{align*}
\end{lemma}

\begin{proof}
We first derive the following bound:
    \begin{align*}
    & \mathop{\Eb}_{s_h \sim P_{h-1} \atop a_h \sim \pi_h}\left[ 
     g(s_h,a_h)|s_{h-1},a_{h-1}\right]\nonumber\\
    &=\int_{s_h}\sum_{a_h}g(s_h,a_h)\pi(a_h|s_h)\langle\phi_{h-1}(s_{h-1},a_{h-1}),\mu_{h-1}(s_h)\rangle d{s_h}\\
    & \leq \left\|\phi_{h-1}(s_{h-1},a_{h-1})\right\|_{(M_{h-1,\phi})^{-1}}\left\|\int\sum_{a_h}g(s_h,a_h)\pi(a_h|s_h)\mu_{h-1}(s_h)d{s_h}\right\|_{M_{h-1,\phi}},
    \end{align*}
where the inequality follows from Cauchy-Schwarz inequality. We further expand the second term in the RHS of the above inequality as follows.
    \begin{align*}
       &\hspace{-5mm} \left\|\int\sum_{a_h}g(s_h,a_h)\pi(a_h|s_h)\mu_{h-1}(s_h)d{s_h}\right\|_{M_{h-1,\phi}}^2\\
        & \overset{(\romannumeral1)}{\leq} n \mathop{\Eb}_{s_{h-1}\sim (P^*,\Pi) \atop a_{h-1}\sim \Pi}
        \left[\left(\int_{s_h}\sum_{a_h}g(s_h,a_h)\pi_h(a_h|s_h)\mu(s_h)^\top\phi(s_{h-1},a_{h-1})d{s_h}\right)^2\right] + \lambda d B^2\\
        &= n \mathop{\Eb}_{s_{h-1}\sim(P^*,\Pi) \atop a_{h-1} \sim \Pi } \left[\left(\mathop{\Eb}_{s_h \sim P_{h-1} \atop a_h \sim \pi_h} \left[g(s_h,a_h)\bigg|s_{h-1},a_{h-1}\right]\right)^2\right] + \lambda d B^2\\
        &\overset{(\romannumeral2)}{\leq} 2n \mathop{\Eb}_{s_{h-1}\sim(P^*,\Pi) \atop a_{h-1}\sim\Pi }\left[\mathop{\Eb}_{s_h\sim P_{h-1}^* \atop a_h\sim \pi_h}\left[g(s_h,a_h)^2\bigg|s_{h-1},a_{h-1}\right]\right] + \lambda d B^2 \\
        &\quad + 2nB^2\mathop{\Eb}_{s_{h-1}\sim(P^*,\Pi) \atop a_{h-1}\sim\Pi }\left[f_{h-1}(s_{h-1},a_{h-1})^2\right]\\
        &\overset{(\romannumeral3)}{\leq} n\At \mathop{\Eb}_{s_{h}\sim(P^*,\Pi)\atop a_h\sim G_h^{\eo}\Pi }\left[g(s_h,a_h)^2\right] + \lambda d B^2 + nB^2\mathop{\Eb}_{s_{h-1}\sim(P^*,\Pi) \atop a_{h-1}\sim \Pi }\left[f
        _{h-1}(s_{h-1},a_{h-1})^2\right],
    \end{align*}
where $(\romannumeral1)$ follows from the assumption that $\|g\|_{\infty}\leq B$, $(\romannumeral2)$ follows from Jensen's inequality, and that $f_{h-1}(s_{h-1},a_{h-1})$ is the total variation distance between $P_{h-1}^*$ and $P_{h-1}$ at time step $h-1$. For $(\romannumeral3)$, note that $G_h^{\eo}\Pi(\cdot|s_h)\geq \eo/A = 1/\tilde{A}$, which implies that $\pi_h(\cdot|s_h)\leq 1 \leq \tilde{A}G_h^{\eo}\Pi(\cdot|s_h)$.
This finishes the proof.
\end{proof}

Based on \Cref{lemma:Step_Back}, we summarize three useful inequalities in the following lemma, which bridges the total variation $f_h^{(n)}$ and the exploration-driven reward $\hb_h^{(n)}$. 
\begin{lemma}\label{lemma:Bound_TV}
Define 
\begin{align}
    W_{h,\phi}^{(n)} = n\mathop{\Eb}_{s_h\sim(P^*,\Pi_n) \atop a_h\sim \Pi_n}\left[\phi(s_h,a_h)(\phi(s_h,a_h))^\top\right] + \lambda I,
\end{align}
where $\lambda = \beta_3 d\log(2NH|\Phi|/\delta)$.
Given that the event $\mathcal{E}$ occurs, the following inequalities hold for any iteration $n$: When $h \geq 2$, 
\begin{align} 
    & \mathop{\Eb}_{s_{h}\sim\hP_{h-1}^{(n)}\atop a_{h}\sim \pi }\left[f_{h}^{(n)}(s_{h},a_{h})\bigg|s_{h-1},a_{h-1}\right]\leq \alpha \left\|\hphi_{h-1}^{(n)}(s_{h-1},a_{h-1})\right\|_{(U_{h-1,\hphi}^{(n)})^{-1}},\label{ineq:hP_Bound_TV}\\
    & \mathop{\Eb}_{s_{h}\sim P_{h-1}^{*}\atop a_{h\sim \pi}} \left[f_{h}^{(n)}(s_{h},a_{h})\bigg|s_{h-1},a_{h-1}\right]\leq \alpha\left\|\phi_{h-1}^{*}(s_{h-1},a_{h-1})\right\|_{(U_{h-1,\phi^*}^{(n)})^{-1}},\label{ineq:P*_Bound_TV}\\
    &\mathop{\Eb}_{s_{h}\sim P_{h-1}^{*} \atop a_{h}\sim \pi}\left[\hb_{h}^{(n)}(s_{h},a_{h})\bigg|s_{h-1},a_{h-1}\right]\leq \gamma \left\|\phi_{h-1}^{*}(s_{h-1},a_{h-1})\right\|_{(W_{h-1,\phi^*}^{(n)})^{-1}},\label{ineq:P*_Bound_bn}
\end{align}
where \begin{align*}
    \alpha  = \sqrt{\beta_3\zeta(\tilde{A} +  d^2)}, \quad \gamma = \sqrt{45\beta_3\zeta\At d(\At+d^2)}.
\end{align*}
When $h = 1$, 
\begin{align}
    &\mathop{\Eb}_{a_1\sim \pi}\left[f_1^{(n)}(s_1,a_1)\right]
     \leq \sqrt{\tilde{A}\zeta/n}, 
    &\mathop{\Eb}_{a_1\sim \pi}\left[\hb(s_1,a_1)\right]
    \leq 15\alpha\sqrt{\frac{d\At}{n}}.\label{ineq:Step_1_Bound}
\end{align}
\end{lemma}

\begin{proof}
We start by showing \Cref{ineq:hP_Bound_TV} as follows. Given that the event $\mathcal{E}$ occurs, we have
\begin{align*}
    \mathop{\Eb}_{s_{h}\sim\hP_{h-1}^{(n)}\atop a_{h}\sim \pi }
    &\left[f_{h}^{(n)}(s_{h},a_{h})\bigg|s_{h-1},a_{h-1}\right] \\
    &\overset{(\romannumeral1)}{\leq} \left\|\hphi_{h-1}^{(n)}(s_{h-1},a_{h-1})\right\|_{(U_{h-1,\hphi}^{(n)})^{-1}} \times \\
    &\sqrt{n\tilde{A}\mathop{\Eb}_{s_{h}\sim(P^*, G_{h-1}^{\eo}\Pi_n)\atop a_h \sim G_h^{\eo}\Pi_n }[f^{(n)}_h(s_h,a_h)^2]+\lambda d + n\mathop{\Eb}_{s_{h-1}\sim(P^*,G_{h-1}^{\eo}\Pi_n)\atop (a_{h-1})\sim G_{h-1}^{\eo}\Pi_n }\left[f^{(n)}_{h-1}(s_{h-1},a_{h-1})^2\right]}\\
    &\overset{(\romannumeral2)}{\leq} \left\|\hphi_{h-1}^{(n)}(s_{h-1},a_{h-1})\right\|_{(U_{h-1,\hphi}^{(n)})^{-1}} \times \\
    &\sqrt{n\tilde{A}\mathop{\Eb}_{s_{h}\sim(P^*, G_{h-1}^{\eo}\Pi_n)\atop a_h \sim G_h^{\eo}\Pi_n }[f^{(n)}_h(s_h,a_h)^2]+\lambda d + n\tilde{A}\mathop{\Eb}_{s_{h-1}\sim(P^*, G_{h-2}^{\eo}\Pi_n)\atop (a_{h-1})\sim G_{h-1}^{\eo}\Pi_n }\left[f^{(n)}_{h-1}(s_{h-1},a_{h-1})^2\right]}\\
    & \overset{(\romannumeral3)}{\leq}\left\|\hphi_{h-1}^{(n)}(s_{h-1},a_{h-1})\right\|_{(U_{h-1,\hphi}^{(n)})^{-1}}\sqrt{2\zeta\tilde{A} + \beta_3\zeta d^2 }\\
    &\leq \alpha\left\|\hphi_{h-1}^{(n)}(s_{h-1},a_{h-1})\right\|_{(U_{h-1,\hphi}^{(n)})^{-1}},
\end{align*}
where $(\romannumeral1)$ follows from \Cref{lemma:Step_Back} and the fact that $f_{h}^{(n)}(s_{h},a_{h}) \leq 1$, $(\romannumeral2)$ follows from importance sampling at time step $h-2$, and $(\romannumeral3)$ follows from  \Cref{lemma:high_prob_event}.

\Cref{ineq:P*_Bound_TV} follows from the arguments similar to the above. 

To obtain \Cref{ineq:P*_Bound_bn}, we first apply \Cref{lemma:Step_Back} and obtain
\begin{align*}  
    & \mathop{\Eb}_{s_{h}\sim P^*_{h-1} \atop a_{h} \sim \pi^{(n)} }\left[\hat{b}^{(n)}_{h}(s_{h},a_{h})\bigg|s_{h-1},a_{h-1}\right]\\
    & \leq \left\|\phi_{h-1}^*(s_{h-1},a_{h-1})\right\|_{(W_{h-1,\sphi}^{(n)})^{-1}}
    \sqrt{n\At\mathop{\Eb}_{s_{h}\sim (P^*, \Pi_n)\atop a_h\sim G_{h}^{\eo}\Pi_n}[\{\hat{b}_h^{(n)}(s_h,a_h)\}^2]+\lambda d },
\end{align*}
where we use the fact that $\hat{b}_h^{(n)}(s_h,a_h)\leq 1$. We further bound the term $n\mathop{\Eb}_{s_{h}\sim (P^*, {\Pi}_n)\atop a_h\sim G_h^{\eo}\Pi_n}[(\hat{b}_h^{(n)}(s_h,a_h))^2]$ as follows:
\begin{align*}
    & \quad n\mathop{\Eb}_{s_{h}\sim (P^*, \Pi_n)\atop a_h\sim G_h^{\eo}\Pi_n}\left[\left(\hat{b}_h^{(n)}(s_h,a_h)\right)^2\right]\\
    & \leq n\mathop{\Eb}_{s_{h}\sim (P^*, \Pi_n)\atop a_h\sim G_h^{\eo}\Pi_n}\left[\hat{\alpha}^2 \left\|\hphi_h^{(n)}(s_h,a_h)\right\|^2_{(\hat{U}_{h,\hphi}^{(n)})^{-1}}\right]\\
    & \overset{(\romannumeral1)}\leq n\mathop{\Eb}_{s_{h}\sim (P^*, \Pi_n)\atop a_h\sim G_h^{\eo}\Pi_n}\left[9 \hat{\alpha}^2 \left\|\hphi_h^{(n)}(s_h,a_h)\right\|^2_{({U}_{h,\hphi}^{(n)})^{-1}}\right]\\
    & = 9\hat{\alpha}^2 {\rm tr}\left\{n\mathop{\Eb}_{s_{h}\sim (P^*,\Pi_n) \atop a_h\sim G_h^{\eo}\Pi_n}\left[\hphi_h^{(n)}(s_h,a_h)\hphi_h^{(n)}(s_h,a_h)^\top\left(n\mathop{\Eb}_{s_h\sim(P^*,\Pi_n) \atop a_h\sim G_h^{\eo}\Pi_n}\left[\hphi_h(s_h,a_h)\hphi_h^{(n)}(s_h,a_h)^\top\right] + \lambda I\right)^{-1}\right]\right\}\\
    & \leq 9\hat{\alpha}^2 {\rm tr}(I) = 9 \hat{\alpha}^2  d,
    \end{align*}
where $(\romannumeral1)$ follows from \Cref{lemma:high_prob_event}, and we use ${\rm tr}(A)$ to denote the trace of any matrix $A$.

Thus,
\begin{align*}
    \mathop{\Eb}_{s_{h}\sim P_{h-1}^{*} \atop a_{h}\sim \pi}\left[\hb_{h}^{(n)}(s_{h},a_{h})\bigg|s_{h-1},a_{h-1}\right]&\leq  \left\|\phi_{h-1}^{*}(s_{h-1},a_{h-1})\right\|_{(W_{h-1,\phi^*}^{(n)})^{-1}}\sqrt{ 9 \tilde{A}\hat{\alpha}^2 d+ \lambda d}\\
    &\leq  \gamma \left\|\phi_{h-1}^{*}(s_{h-1},a_{h-1})\right\|_{W_{h-1,\phi^*}^{(n)})^{-1}}.
\end{align*}
    In addition, for $h=1$, we have 
    
\begin{equation*}
    \mathop{\Eb}_{a_1\sim \pi^{(n)}}\left[f_1^{(n)}(s_1,a_1)\right]
    \overset{(\romannumeral1)}{\leq} \sqrt{\tilde{A} \mathop{\Eb}_{a_1\sim G_1^{\eo}\Pi_n}\left[f_1^{(n)}(s_1,a_1)^2\right]} \leq \sqrt{\tilde{A} \zeta/n},
\end{equation*}
and
\begin{align*}
    \mathop{\Eb}_{a_1\sim \pi^{(n)}}\left[\hb(s_1,a_1)\right]
    &\overset{(\romannumeral2)}{\leq}
    \hat{\alpha}\sqrt{\At \mathop{\Eb}_{a_1\sim G_1^{\eo}\Pi_n}\left[\|\hphi_1(s_1,a_1)\|^2_{(\hat{U}_{1,\hphi}^{(n)})^{-1}}\right]}\\
    &\leq 3\hat{\alpha}\sqrt{\At \mathop{\Eb}_{a_1\sim G_1^{\eo}\Pi_n}\left[\|\hphi_1(s_1,a_1)\|^2_{(U_{1,\hphi}^{(n)})^{-1}}\right]}\\
    &\leq 3\sqrt{\frac{25\At \alpha^2  d}{n}} = 15\alpha\sqrt{\At\zeta/n},
\end{align*}
where both $(\romannumeral1)$ and $(\romannumeral2)$ follow from Jensen's inequality and importance sampling.
\end{proof}

The following lemma is key to ensure that Low-Rank-SWEET terminates in finite episodes.

\begin{lemma}\label{lemma:sublinear_lowrank}
Given that the event $\Ec$ occurs, the summation of the truncated value functions $\hV^{\pi^{(n)}}_{\hP^{(n)},\hb^{(n)}}$ under exploration policies $\{\pi^{(n)}\}_{n\in\Nc}$ is sublinear with respect to $|\Nc|$ for any $\Nc\subset[N]$, i.e., the following bound holds:
\begin{align*}
        \sum_{n\in\Nc} \hV_{\hP^{(n)},\hb^{(n)}}^{\pi^{(n)}} + \sqrt{\At\zeta/n}  \leq  32\zeta Hd^2\At\sqrt{\beta_3|\Nc|}.
\end{align*}
\end{lemma}

\begin{proof}
Note that $\hV_{h,\hP^{(n)} , \hb^{(n)}}^{\pi}\leq 1$ holds for any policy $\pi$ and $h\in[H]$. We first have 
\begin{align*}
\hV_{\hP^{(n)} , \hb^{(n)}}^{\pi^{(n)}} - V^{\pi^{(n)}}_{P^*,\hb^{(n)}}& \leq
\mathop{\Eb}_{\pi^{(n)}}\left[\hP^{(n)}_{1}\hV^{\pi^{(n)}}_{2,\hP^{(n)},\hb^{(n)}}(s_1,a_1) - P^*_1 V^{\pi^{(n)}}_{2,P^*,\hb^{(n)}}(s_1,a_1)\right]\\
& = \mathop{\Eb}_{\pi^{(n)}}\left[\left(\hP^{(n)}_{1} - P^*_1\right)\hV^{\pi^{(n)}}_{2,\hP^{(n)},\hb^{(n)}}(s_1,a_1) + P^*_1\left(\hV^{\pi^{(n)}}_{2,\hP^{(n)},\hb^{(n)}} - V^{\pi^{(n)}}_{2,P^*,\hb^{(n)}}\right)(s_1,a_1)\right]\\
&\leq \mathop{\Eb}_{\pi^{(n)}}\left[f_1^{(n)}(s_1,a_1) + P^*_1\left(\hV^{\pi^{(n)}}_{2,\hP^{(n)},\hb^{(n)}} - V^{\pi^{(n)}}_{2,P^*,\hb^{(n)}}\right)\right]\\
&\leq \ldots\\
&\leq \mathop{\Eb}_{P^*, \pi^{(n)}}\left[\sum_{h=1}^H f^{(n)}(s_h,a_h)\right] = V^{\pi^{(n)}}_{P^*,f^{(n)}},
\end{align*}

which implies
$
\hV_{\hP^{(n)} , \hb^{(n)}}^{\pi^{(n)}} \leq  V^{\pi^{(n)}}_{P^*,\hb^{(n)}} + V^{\pi^{(n)}}_{P^*,f^{(n)}}.
$

Applying the \Cref{ineq:P*_Bound_bn} and \Cref{ineq:Step_1_Bound}, we obtain the following bound on the value function $V_{P^*, {\hb}^{(n)}}^{{\pi}_n}$:
\begin{align*}
    V_{P^*, {\hb}^{(n)}}^{{\pi}_n}
    & = \sum_{h=1}^{H} \mathop{\Eb}_{s_h\sim (P^*, \pi^{(n)}) \atop a_h\sim\pi^{(n)}}\left[\hat{b}_n(s_h,a_h)\right]\\
    & \leq \sum_{h=2}^{H} \mathop{\Eb}_{s_{h-1}\sim (P^*, {\pi^{(n)}}) \atop a_{h-1}\sim\pi^{(n)}} 
    \left[\gamma\left\|\phi_{h-1}^*(s_{h-1},a_{h-1})\right\|_{(W_{h-1,\sphi}^{(n)})^{-1}} \right]
    +  15 \alpha\sqrt{\frac{d\At}{n}}\\
    &\leq\sum_{h=1}^{H} \mathop{\Eb}_{s_{h}\sim (P^*, \pi^{(n)}) \atop a_{h} \sim \pi^{(n)}} \left[\gamma\left\|\phi_{h}^*(s_{h},a_{h})\right\|_{(W_{h,\sphi}^{(n)})^{-1}}
    \right]
    + 15 \alpha\sqrt{\frac{d\At}{n}}.
\end{align*}

Similarly, we obtain
\begin{align*}
    V_{P^*, f^{(n)}}^{{\pi}_n}
    & = \sum_{h=1}^{H} \mathop{\Eb}_{s_h\sim (P^*, \pi^{(n)}) \atop a_h\sim\pi^{(n)}}\left[\hat{b}_n(s_h,a_h)\right]\\
    & \leq \sum_{h=2}^{H} \mathop{\Eb}_{s_{h-1}\sim (P^*, {\pi^{(n)}}) \atop a_{h-1}\sim\pi^{(n)}} 
    \left[\alpha\left\|\phi_{h-1}^*(s_{h-1},a_{h-1})\right\|_{(U_{h-1,\sphi}^{(n)})^{-1}} \right]
    +  \sqrt{\frac{\zeta\At}{n}}\\
    &\leq\sum_{h=1}^{H} \mathop{\Eb}_{s_{h}\sim (P^*, \pi^{(n)}) \atop a_{h} \sim \pi^{(n)}} \left[\alpha\left\|\phi_{h}^*(s_{h},a_{h})\right\|_{(U_{h,\sphi}^{(n)})^{-1}}
    \right]
    + \sqrt{\frac{\zeta\At}{n}}.
\end{align*}

Then, taking the summation of $V_{P^*, {\hb}^{(n)}+f^{(n)}}^{{\pi}_n}$ over $n\in\Nc$, we have
\begin{align*}
    \sum_{n\in\Nc}&V_{P^*, f^{(n)} + \hb^{(n)}}^{{\pi}_n} + \sqrt{\At\zeta/n}\\
    &\leq \sum_{n\in\Nc}15\alpha\sqrt{\frac{d\At}{n}} + 2\sum_{n\in\Nc}\sqrt{\frac{\At\zeta}{n}}  + \sum_{n\in\Nc}\sum_{h=1}^{H} \mathop{\Eb}_{s_{h}\sim (P^*, \pi^{(n)}) \atop a_{h} \sim \pi^{(n)}} \left[\gamma_n\left\|\phi_{h}^*(s_{h},a_{h})\right\|_{(W_{h,\sphi}^{(n)})^{-1}}
    \right] \\
    & \quad + \sum_{n\in\Nc} \sum_{h=1}^{H} \mathop{\Eb}_{s_{h}\sim (P^*, \pi^{(n)}) \atop a_{h} \sim \pi^{(n)}} \left[\alpha\left\|\phi_{h}^*(s_{h},a_{h})\right\|_{(U_{h,\sphi}^{(n)})^{-1}}
    \right]\\
    &\overset{(\romannumeral1)}{\leq} 17\alpha\sqrt{\zeta d\At|\Nc|} + \gamma\sum_{h=1}^{H}\sqrt{|\Nc|\sum_{n\in\Nc}\mathop{\Eb}_{s_{h}\sim(P^*,\pi^{(n)}) \atop a_h\sim\pi^{(n)}}\left[\left\|\phi_{h}^*(s_{h},a_{h})\right\|^2_{(W_{h,\sphi}^{(n)})^{-1}}\right]}\\
    &\quad + \alpha\sum_{h=1}^{H}\sqrt{\tilde{A}|\Nc|\sum_{n\in\Nc}\mathop{\Eb}_{s_{h}\sim(P^*,\pi^{(n)}) \atop a_h\sim G_h^{\eo}\pi^{(n)}}\left[\left\|\phi_{h}^*(s_{h},a_{h})\right\|^2_{(U_{h,\sphi}^{(n)})^{-1}}\right]}\\
    &\overset{(\romannumeral2)}{\leq} 17\zeta\sqrt{2\beta_3 d\At(\At+d^2)|\Nc|} + H\sqrt{45\beta_3\zeta d\At(\At+d^2)}\sqrt{d|\Nc|\zeta}\\
    & \quad + H\sqrt{\beta_3\zeta(\At+d^2)}\sqrt{d\At|\Nc|\zeta}\\
    & \leq 32\zeta Hd\sqrt{\beta_3\At(d^2+\At)|\Nc|}\\
    &\leq 32\zeta Hd^2\At\sqrt{\beta_3|\Nc|},
\end{align*}
where $(\romannumeral1)$ follows from  Cauchy-Schwarz inequality and importance sampling, and $(\romannumeral2)$ follows from \Cref{corollary: revised elliptical potential lemma}.
Hence, the statement of \Cref{lemma:sublinear_lowrank} is verified.
\end{proof}

\subsection{Proof of \Cref{thm:lowrank}}\label{app:prooftheorem3}

\begin{theorem}[Restatement of \Cref{thm:lowrank}]\label{appx:thm:lowrank}
Given $\epsilon,\delta\in(0,1)$, and safety constraint $(c,\tau)$, let $\mathfrak{U}  = \min\left\{\frac{\epsilon}{2}, \frac{\Delta_{\min}}{2}, \frac{\epsilon\Delta_{\min}}{5}, \frac{\tau}{6}, \frac{\kappa}{24}\right\}$, and $\mathtt{T} = \Delta(c,\tau)\mathfrak{U}/3$ be the termination condition of Low-rank-SWEET.
Then, with probability at least $1-\delta$, Low-rank-SWEET achieves the learning objective  of safe reward-free exploration (\Cref{eqn: Safety,eqn:objective}) and 
the number of trajectories collected in the exploration phase is at most  
\[O\left(\frac{H^3d^4A^2\iota}{\kappa^2\Delta(c,\tau)^2\mathfrak{U}^2}+  \frac{H^3d^4A^2\iota}{\kappa^4}\right),\]
where $\iota = \log^2\left[\left(  \frac{H^2 d^4 A^2}{\kappa^2\Delta(c,\tau)^2\mathfrak{U}^2} + \frac{ H^2d^4A^2}{\kappa^4}  \right) |\Phi||\Psi|H/\delta \right]$ .
\end{theorem}

\begin{proof}

The proof of \Cref{thm:lowrank} mainly instantiates \Cref{main:thm:meta_safe} by verifying that (a) $\mathtt{U}^{(n)}(\pi) = \hV_{\hP^{(n)},\hb^{(n)}}^{\pi}+\sqrt{\At \zeta/n}$ is a valid approximation error bound for $V_{\hP^{(n)},u}^{\pi}$, and (b) Low-rank-SWEET satisfies the termination condition within $N$ iterations. The proof consists of three steps with the first two steps verifying the above two conditions and the last step characterizes the sample complexity.

\textbf{Step 1:} This step establishes that $\mathtt{U}^{(n)}(\pi) = \hV_{\hP^{(n)},\hb^{(n)}}^{\pi}+\sqrt{\At \zeta/n}$ is a valid approximation error bound in Low-rank-SWEET.

\begin{lemma}
\label{lemma:lowrank_errorbound}
For all $n\in[N]$, policy $\pi$ and the normalized utility function $u$, given that the event $\mathcal{E}$ occurs, we have
\begin{equation*}
     \left| V_{P^*,u}^{\pi} - V_{\hP^{(n)},u}^{\pi}\right| \leq   \hV_{\hP^{(n)},\hb^{(n)}}^{\pi}+\sqrt{\At \zeta/n}.
\end{equation*}
\end{lemma}

\begin{proof}

We first show that $\left| V_{P^*,u}^{\pi} - V_{\hP^{(n)},u}^{\pi}\right| \leq   \hV_{\hP^{(n)},f^{(n)}}^{\pi}$.

Recall the definition of the truncated value functions $\hV_{h,\hP^{(n)} , u} (s_h)$ and $\hQ_{h,\hP^{(n)}, u}(s_h,a_h)$:
\begin{align*}
    &\hQ_{h,\hP^{(n)},u}^{\pi}(s_h,a_h) = u(s,a) + \hP_h^{(n)}\hV_{h+1,\hP^{(n)},u}^{\pi}(s_h,a_h),\\
    &\hV_{h,\hP^{(n)},u}^{\pi}(s_h) = \min\bigg\{1, \mathop{\Eb}_{\pi}\left[\hat{Q}^{\pi}_{h,\hP^{(n)},u}(s_h,a_h)\right]\bigg\}.
\end{align*}
We develop the proof by induction. For the base case $h=H+1$, we have $\left|V_{H+1,\hP^{(n)},u}^{\pi}(s_{H+1}) - V_{H+1,P^*,u}^{\pi}(s_{H+1}) \right| = 0  = \hV_{H+1,\hP^{(n)},\hb^{(n)}}^{\pi}(s_{H+1})$.

Assume that $\left|V_{h+1,\hP^{(n)},u}^{\pi}(s_{h+1}) - V^{\pi}_{h+1,P^*,u}(s_{h+1}) \right| \leq \hV_{h+1,\hP^{(n)},\hb^{(n)}}^{\pi}(s_{h+1})$ holds for any $s_{h+1}$.

Then, from Bellman equation, we have,
\begin{align}
    \bigg|Q^{\pi}_{h,\hP^{(n)},u}&(s_h,a_h) - Q^{\pi}_{P^*,u}(s_h,a_h) \bigg| \nonumber\\
    & = \bigg|\hP_h^{(n)}V^{\pi}_{h,\hP^{(n)},u}(s_h,a_h) - P_h^*V^{\pi}_{h+1,P^*,u}(s_h,a_h)\bigg| \nonumber\\
    & = \bigg| \hP_h^{(n)}\left(V_{h+1,\hP^{(n)},u}^{\pi} - V^{\pi}_{h+1,P^*,u}\right)(s_h,a_h) + \left(\hP_h^{(n)} - P_h^*\right)V^{\pi}_{h,P^*,u}(s_h,a_h)\bigg| \nonumber\\
    &\overset{(\romannumeral1)}\leq f_h^{(n)}(s_h,a_h) + \hP_h^{(n)} \bigg|V_{h+1,\hP^{(n)},u}^{\pi} - V^{\pi}_{h+1,P^*,u}\bigg|(s_h,a_h) \nonumber\\
    &\overset{(\romannumeral2)}\leq f_h^{(n)}(s_h,a_h) + \hP_h^{(n)} \hV_{h+1,\hP^{(n)},f^{(n)}}^{\pi}(s_h,a_h) \nonumber\\
    & = \hQ_{h,\hP^{(n)},f^{(n)}}^{\pi}(s_h,a_h), \label{eqn:lowrank:hatQ-Q<hatQ}
\end{align}
where $(\romannumeral1)$ follows from $\|\hP_h^{(n)}(\cdot|s_h,a_h) - P^*_h(\cdot|s_h,a_h)\|_1 =  f_h^{(n)}(s_h,a_h)$ and the assumption that $u$ is normalized, and $(\romannumeral2)$ follows from the induction hypothesis.

Then, by the definition of $\hV^{\pi}_{h,\hP^{(n)},u}(s_h)$, we have
\begin{align*}
    \bigg| V_{h, \hP^{(n)},u}^{\pi}&(s_h) - V_{h, P^*,u}^{\pi}(s_h)\bigg|\\
    & = \bigg| \min\bigg\{ 1 - V_{h, P^*,u}^{\pi}(s_h), \mathop{\Eb}_{\pi}\left[Q_{h,\hP^{(n)}, u}^{\pi} (s_h,a_h)\right] - \mathop{\Eb}_{\pi}\left[Q_{h,P^*,u}^{\pi}(s_h,a_h)\right]\bigg\}\bigg| \\
    & \overset{(\romannumeral1)}\leq  \min\bigg\{ 1 , \bigg| \mathop{\Eb}_{\pi}\left[Q_{h,\hP^{(n)}, u}^{\pi} (s_h,a_h) - Q_{h,P^*,u}^{\pi}(s_h,a_h)\right] \bigg|\bigg\} \\
    &\overset{(\romannumeral2)}\leq \min\bigg\{ 1 , \mathop{\Eb}_{\pi}\left[\hat{Q}_{h,\hP^{(n)}, f^{(n)}}^{\pi} (s_h,a_h) \right] \bigg\} \\
    &= \hV_{h,\hP^{(n)},f^{(n)}}^{\pi}(s_h),
\end{align*}
where $(\romannumeral1)$ follows because $\hat{Q}_{h,\hP^{(n)}, u}^{\pi} (s_h,a_h) - Q_{h,P^*,u}^{\pi}(s_h,a_h) > -1 $, and $(\romannumeral2)$ follows from \Cref{eqn:lowrank:hatQ-Q<hatQ}.

Therefore, by induction, we have 
\begin{align*}
    \left| V_{P^*,u}^{\pi}-\hV_{\hP^{(n)},u}^{\pi}\right|\leq
    \hV_{\hP^{(n)},f^{(n)}}^{\pi}.
\end{align*}

Then, we show that $\hV_{\hP^{(n)},f^{(n)}}^{\pi}\leq \hV_{\hP^{(n)},\hb^{(n)}}^{\pi}+\sqrt{\At \zeta_n}$.

By \Cref{ineq:hP_Bound_TV} and the fact that the total variation distance is upper bounded by 1,  with probability at least $1-\delta/2$, we have 
\begin{align}
    \mathop{\Eb}_{\hP^{(n)},\pi}\left[f_h^{(n)}(s_h,a_h)\bigg| s_{h-1}\right] \leq \mathop{\Eb}_{\pi }\left[\min\left(\alpha\left\|\hphi_{h-1}^{(n)}\right\|_{(U_{h-1,\hphi}^{(n)})^{-1}}, 1\right) \right], \forall h\geq2.\label{ineq: f_1 bound}
\end{align}
Similarly, when $h=1$, 
\begin{align}
\mathop{\Eb}_{a_1\sim\pi}\left[f_1^{(n)}(s_1,a_1)\right]\leq \sqrt{\At\mathop{\Eb}_{a\sim G_1^{\eo}\Pi_n}\left[\left(f_1^{(n)}(s_1,a_1)\right)^2\right]}\leq\sqrt{\At\zeta_n}.\label{ineq:f1<K_zeta}
\end{align}
Based on \Cref{coro:concentration on b}, \Cref{ineq: f_1 bound} and $\alpha = 5 \hat{\alpha}$, we have
\begin{align}
    \mathop{\Eb}_{\pi} \left[\hb^{(n)}_h(s_h,a_h)\bigg|s_h\right] \geq \mathop{\Eb}_{\pi}\left[\min\left(\alpha\left\|\hphi_{h}^{(n)}\right\|_{(U_{h,\hphi}^{(n)})^{-1}}, 1\right) \right] \geq \mathop{\Eb}_{\hP^{(n)},\pi }\left[f_{h+1}^{(n)}(s_{h+1},a_{h+1})\bigg|s_h\right] .\label{ineq:f<b}
\end{align}

For the base case $h= H$, we have 

\begin{align*}
\mathop{\Eb}_{\hP^{(n)},\pi}\left[\hV_{H,\hP^{(n)},f^{(n)}}^{\pi}(s_H)\bigg|s_{H-1}\right] 
&= \mathop{\Eb}_{\hP^{(n)}, \pi}\left[f_H^{(n)}(s_H,a_H)\bigg| s_{H-1}\right]\\
&\leq \mathop{\Eb}_{\pi}\left[b_{H-1}^{(n)}(s_{H-1},a_{H-1})|s_{H-1}\right]\\
&\leq \min\left\{1, \mathop{\Eb}_{\pi}\left[\hQ_{H-1,\hP^{(n)},\hb^{(n)}}^{\pi}(s_{H-1},a_{H-1})\bigg|s_{H-1}\right]\right\}\\
& = \hV^{\pi}_{H-1,\hP^{(n)},\hb^{(n)}}(s_{H-1}).
\end{align*}

Assume that $\mathop{\Eb}_{\hP^{(n)}, \pi}\left[\hV_{h+1,\hP^{(n)},f^{(n)}}^{\pi}(s_{h+1})\bigg|s_{h}\right]\leq \hV_{h,\hP^{(n)}, \hb^{(n)}}^{\pi}(s_h)$ holds for step $h+1$. Then, by Jensen's inequality, we obtain
\begin{align*}
    \mathop{\Eb}_{\hP^{(n)}, \pi}&\bigg[\hV_{h,\hP^{(n)},f^{(n)}}^{\pi}(s_h)\bigg|s_{h-1}\bigg] \\
    & \leq \min\left\{1, \mathop{\Eb}_{\hP^{(n)}, \pi}\left[ f_h^{(n)}(s_h,a_h) + \hP_h^{(n)}\hV_{h+1,\hP^{(n)},f^{(n)}}^{\pi}(s_h,a_h)\bigg|s_{h-1}\right]\right\}\\
    & \overset{(\romannumeral1)}\leq \min\left\{1, \mathop{\Eb}_{\pi}\left[ \hb_{h-1}^{(n)}(s_{h-1},a_{h-1})\right] + \mathop{\Eb}_{\hP^{(n)},\pi}\left[\mathop{\Eb}_{\hP^{(n)},\pi} \left[ \hV_{h+1,\hP^{(n)},f^{(n)}}^{\pi}(s_{h+1})\bigg|s_{h}\right]\bigg|s_{h-1}\right]\right\}\\
    & \overset{(\romannumeral2)}\leq \min\left\{1, \mathop{\Eb}_{\pi}\left[ b_{h-1}^{(n)}(s_{h-1},a_{h-1})\right] + \mathop{\Eb}_{\hP^{(n)},\pi} \left[ \hV_{h,\hP^{(n)},\hb^{(n)}}^{\pi}(s_{h})\bigg|s_{h-1}\right]\right\}\\
    & = \min\left\{1, \mathop{\Eb}_{\pi}\left[\hQ_{h-1,\hP^{(n)},\hb^{(n)}}^{\pi}(s_{h-1},a_{h-1})\right] \right\}\\
    & = \hV^{\pi}_{h-1,\hP^{(n)},\hb^{(n)}}(s_{h-1}),
\end{align*}
where $(\romannumeral1)$ follows from \Cref{ineq:f<b}, and $(\romannumeral2)$ is due to the induction hypothesis.

By induction, we conclude that
\begin{align*}
    \hV_{\hP^{(n)},f^{(n)}}^{\pi} & = \mathop{\Eb}_{\pi}\left[f_1^{(s)}(s_1,a_1)\right] + \mathop{\Eb}_{\hP^{(n)},\pi}\left[\hV_{2,\hP^{(n)},f^{(n)}}^{\pi}(s_2)\bigg|s_1\right]\\
    &\leq \sqrt{\tilde{A}\zeta/n} + \hV_{\hP^{(n)},\hb^{(n)}}^{\pi}.
\end{align*}

Combining Step 1 and Step 2, we conclude that
\[\left| V_{P^*,u}^{\pi} - V_{\hP^{(n)},u}^{\pi}\right|\leq \sqrt{\tilde{A}\zeta/n} + \hV_{\hP^{(n)},\hb^{(n)}}^{\pi}.\]

\end{proof}

\iffalse
\begin{lemma}
 $G_{h-1,h}^{\eo}\pi^{(n)}$ is safe
\end{lemma}

\begin{proof}
Consider the case where $\pi^{(n)}\neq \pi^0$. Based on \Cref{lemma:greedy_performance}, we have
    \begin{align*}
        V_{P^*,c}^{G_{h-1,h}^{\eo}\pi^{(n)}} & \leq V^{\pi^{(n)}}_{P^*,c} + 2\epsilon_0\\
        &\leq V_{\hP^{(n)}, c}^{\pi^{(n)}} + \hV_{\hP^{(n)},  \hb^{(n)}}^{\pi^{(n)}} + \sqrt{\At\zeta_n} + 2\eo\\
        &\leq \tau - \frac{\kappa}{3} + \frac{\kappa}{3} = \tau,
    \end{align*}
where the last inequality is due to the exploration phase of Low-rank-SWEET.
\end{proof}

\fi

\textbf{Step 2:} This step shows that Low-rank-SWEET terminates in finite episodes.
\begin{lemma}\label{lemma:stop_low-rank}
On the event $\mathcal{E}$, there exists $n_{\epsilon}\in[N]$ such that $ \hV_{\hP^{(n_{\epsilon})},\hb^{(n_{\epsilon})}}^{\pi^{n_{\epsilon}}} + \sqrt{\tilde{A}\zeta/n_{\epsilon}} \leq \mathtt{T}$, where  
 $N$ is defined in \Cref{eqn:N_lowrank}.
\end{lemma}

Let $\Nc_0 = \left\{n: \left|\mathcal{C}_L^{(n)}\right| = 1\right\}$.
We first show that $\Nc_0$ is a finite set.

Note that, $n\in\Nc_0$ implies that $\hV_{\hP^{(n)},c}^{\pi^0} + \hV_{\hP^{(n)},\hb^{(n)}}^{\pi^0} + \sqrt{\tilde{A}\zeta/n} > \tau-2\kappa/3$, and $\pi^{(n)} = \pi^0$. Then, we have,
\begin{align*}
    |\Nc_0|\kappa/3 &< \sum_{n\in\Nc_0}\left(\hV_{\hP^{(n)},c}^{\pi^0} + \hV_{\hP^{(n)},\hb^{(n)}}^{\pi^0} + \sqrt{\tilde{A}\zeta_{n}} - V_{P^*, c }^{\pi^0}\right)\\
    &\overset{(\romannumeral1)}\leq \sum_{n\in\Nc_0} 2\hV_{\hP^{(n)},\hb^{(n)}}^{\pi^0} + 2\sqrt{\tilde{A}\zeta_n}\\
    &\overset{(\romannumeral2)}\leq 64\zeta Hd^2\At\sqrt{\beta_3|\Nc|},
\end{align*}
where $(\romannumeral1)$ is due to \Cref{lemma:lowrank_errorbound} and  the $(\romannumeral2)$ follows from \Cref{lemma:sublinear_lowrank}. Therefore, we have 
\[|\Nc_0|\leq \frac{2^{12}\cdot3^2\beta_3 H^2d^4\At^2\zeta^2}{\kappa^2}.\]

Next, we prove the existence of $n_{\epsilon}$ via contradiction. Assume $\hV_{\hP^{(n)}, \hb^{(n)}}^{\pi^{(n)}} + \sqrt{\tilde{A}\zeta_n}>\mathtt{T}, \forall n\in[N]/\Nc_0$. By \Cref{lemma:sublinear_lowrank}, we have
\begin{align*}
    (N-|\Nc_0|)\mathtt{T} & <\sum_{n\in\Nc} V_{\hP^{(n)}, \hb^{(n)}}^{\pi^{(n)}} + \sqrt{\At\zeta_n}\\
    &\leq 32\zeta Hd^2\At\sqrt{\beta_3|\Nc|},
\end{align*}
which implies
\[N < |\Nc_0| + \frac{2^{10}\beta_3H^2d^4\tilde{A}^2\zeta^2}{\mathtt{T}^2}\leq \frac{2^{10}\beta_3 H^2d^4\tilde{A}^2\zeta^2}{\mathtt{T}^2} + \frac{2^{12}\cdot3^2\beta_3 H^2d^4\At^2\zeta^2}{\kappa^2}.\]
This contradicts with the fact that $N = \frac{2^{10}\beta_3 H^2d^4\tilde{A}^2\zeta^2}{\mathtt{T}^2} + \frac{2^{12}\cdot3^2\beta_3 H^2d^4\At^2\zeta^2}{\kappa^2}$. 
\end{proof}

%Combining \Cref{lemma:lowrank_errorbound,lemma:sublinear_lowrank} and \Cref{thm:meta_safe}, we are ready to prove \Cref{thm:lowrank}.

\textbf{Step 3:} This step analyzes the sample complexity of Low-rank-SWEET as follows.

Given that the event $\mathcal{E}$ occurs, since $\mathtt{T} = \Delta(c,\tau)\mathfrak{U}/3$ and $\tilde{A} = A/\eo = 6A/\kappa$, by \Cref{lemma:stop_low-rank}, the number of iterations is at most 
\[N = \frac{2^{12} 3^4 \beta_3 H^2d^4A^2\zeta^2}{\kappa^2\Delta(c,\tau)^2\mathfrak{U}^2} + \frac{2^{14}3^4 \beta_3 H^2d^4A^2\zeta^2}{\kappa^4}.\]

Note that $n = c_0\log^2(c_1 n)$ implies $n\leq 4c_0 \log^2(c_0c_1)$. Thus,
\[N = O\left(\frac{H^2 d^4 A^2\iota}{\kappa^2\Delta(c,\tau)^2\mathfrak{U}^2} + \frac{ H^2d^4A^2\iota}{\kappa^4}\right),\]
where 
\[\iota = \log^2\left[\left(  \frac{H^2 d^4 A^2}{\kappa^2\Delta(c,\tau)^2\mathfrak{U}^2} + \frac{ H^2d^4A^2}{\kappa^4}  \right) |\Phi||\Psi|H/\delta \right].\]

Since there are $H$ episodes in each iteration, the sample complexity is at most 
\[O\left(\frac{H^3 d^4 A^2\iota}{\kappa^2\Delta(c,\tau)^2\mathfrak{U}^2} + \frac{ H^3 d^4A^2\iota}{\kappa^4}\right).\]
Thus, Low-rank-SWEET terminates in finite episodes. 

Note that $\mathtt{U}(\pi) = \hV^{\pi}_{\hP^{\epsilon},\hb^{\epsilon}} + \sqrt{\tilde{A}\zeta/{n_{\epsilon}}}$. Given that the event $\mathcal{E}$ occurs, by \Cref{appx:lemma:Concave} and \Cref{lemma:lowrank_errorbound}, $\mathtt{U}(\pi)$ is an approximation error function under $\hP^{\epsilon}$, and is concave and continuous on $\mathcal{X}$.

We further note that $\frac{\kappa/3(\Delta(c,\tau) - 2\kappa/3)}{4(\Delta(c,\tau)-\kappa/3)} \geq \frac{\kappa}{24}$ , and $\Delta(c,\tau) - \kappa/3 \geq 2\Delta(c,\tau)/3$ due to the condition $\Delta(c,\tau)\geq \kappa$, which implies that $\mathtt{T} = \Delta(c,\tau)\mathfrak{U}/3 $ satisfies the requirement in \Cref{thm:meta_safe}.

Therefore, using \Cref{thm:meta_safe}, we conclude that with probability at least $1-\delta$, the exploration phase of Low-rank-SWEET is safe and $\bar{\pi}$ is an $\epsilon$-optimal policy subject to the constraint $(c^*,\tau^*).$

\vspace{0.3cm}

\section{Auxiliary lemmas}\label{appx:auxiliary}

We first provide the following property of a mixture policy and its equivalent Markov policy for completeness.
\begin{lemma}[Theorem 6.1 in \cite{Altman:CMDP:1999}]\label{lemma:thm6.1}
Given a model $P$, any Markov policies $\pi,\pi'\in\mathcal{X}$, and $\gamma\in[0,1]$, there exists $\pi^{\gamma}\in\mathcal{X}$ that is Markov and equivalent to the mixture policy $\gamma \pi\oplus (1-\gamma)\pi'$. Let $\rho_h^{\pi}(s_h)$ and $\rho_h^{\pi}(s_h,a_h)$ be the marginal distributions over the state and the state-action pairs induced by $\pi$ under $P$, respectively. Then, the following statements hold:
\begin{itemize}
    \item $\rho_h^{\pi^{\gamma}}(s_h) = \gamma \rho_h^{\pi}(s_h) + (1-\gamma) \rho_h^{\pi'}(s_h) $,
    \item $\rho_h^{\pi^{\gamma}}(s_h,a_h) = \gamma \rho_h^{\pi}(s_h,a_h) + (1-\gamma) \rho_h^{\pi'}(s_h,a_h) $,
    \item $\pi^{\gamma}(a_h|s_h) = \rho_h^{\pi^{\gamma}}(s_h,a_h)/ \rho_h^{\pi^{\gamma}}(s_h)$,
    \item $V^{\pi^{\gamma}}_{P,u} = \gamma V_{P,u}^{\pi} + (1-\gamma) V_{P,u}^{\pi'}$ holds for any utility function $u$.
\end{itemize}
\end{lemma}

%Recall $f_h^{(n)}(s,a)=\left\|\hP_h^{(n)}(\cdot|s,a) - P^*_h(\cdot|s,a)\right\|_1$ 

Next, we present the estimation error of MLE in the $n$-th iteration at step $h$, given state $s$ and action $a$, in terms of the total variation distance, i.e. $f_h^{(n)}(s,a)=\left\|\hP_h^{(n)}(\cdot|s,a) - P^*_h(\cdot|s,a)\right\|_1$. By Theorem 21 in \citet{NEURIPS2020_e894d787}, we are able to guarantee that under all exploration policies, the estimation error can be bounded with high probability.

\begin{lemma}\label{lemma:MLE}
 {\rm(MLE guarantee).} Given $\delta\in(0,1)$, we have the following inequality holds for any $n\in[N],h\in[H]$ with probability at least $1-\delta/2$:
\begin{align*}
    \sum_{m=0}^{n-1} \mathop{\Eb}_{s_{h}\sim \left(P^*,G_{h-1}^{\eo}\pi^{(m)}\right)\atop a_h\sim G_h^{\eo}\pi^{(m)}}\left[f_h^{n}(s_h,a_h)^2\right]\leq \zeta, \quad\mbox{ where } \zeta : = \log\left(2|\Phi||\Psi|NH/\delta\right) .
\end{align*}
\end{lemma}

Dividing both sides of the inequality in \Cref{lemma:MLE} by $n$, we have the following corollary hold, which is intensively used in the analysis.
\begin{corollary}\label{coro:MLE}
Given $\delta\in(0,1)$, the following inequality holds for any $n,h\geq 1$ with probability at least $1-\delta/2$:
\[ \mathop{\Eb}_{s_{h}\sim \left(P^*,G_{h-1}^{\eo}\Pi_n\right)\atop a_h\sim G_{h}^{\eo}\Pi_n } \left[f_h^{n}(s_h,a_h)^2\right]\leq \zeta/n,\]
where $\Pi_n$ and $G_h^{\eo}\Pi_n$ are defined in \Cref{eqn:uniform_policy} and \Cref{eqn:greedy}, respectively.
\end{corollary}

Then, we present two critical lemmas which ensure the summation of the approximation errors grows sublinearly in Tabular-SWEET and Low-rank-SWEET.

\begin{lemma}[Lemma 9 in \cite{menard2021fast}]\label{lemma:1/N < logN}
Suppose $\{a_n\}_{n=0}^{\infty}$ is a sequence with $a_n\in[0,1]$, $\forall n$. Let $S_n= \max\{1, \sum_{m=0}^n a_m \} $. Then, the following inequality holds:
\[\sum_{n=1}^N \frac{a_n}{{S_{n-1}}} \leq 4\log(S_N+1).\]

\end{lemma}

\begin{lemma}[Elliptical potential lemma: Lemma B.3 in \cite{he2021logarithmic}]\label{corollary: revised elliptical potential lemma} \label{lemma: Elliptical_potential}
Consider a sequence of $d \times d$ positive semidefinite matrices $X_1, \dots, X_N$ with ${\rm tr}(X_n) \leq 1$ for all $n \in [N]$. Define $M_0=\lambda_0 I$ and $M_n=M_{n-1}+X_n$. Then, 
$\forall \Nc \subset [N]$,
\begin{equation*}
    \sum_{n \in \Nc} {\rm tr}(X_nM_{n-1}^{-1})  \leq 2d\log\left(1+\frac{|\Nc|}{d\lambda_0}\right).
\end{equation*}
\end{lemma}

\iffalse
Next, we introduce some useful inequalities that help convert the finite sample error bound into the sample complexity.
\begin{lemma} \label{lemma: log transform}
 $\forall c \geq e^2, n \geq 1, \alpha  \in \Rb^+$. 
 \begin{align*}
     n \geq 2c \log (\alpha c) &\Rightarrow n \geq c\log (\alpha n),\\
     n \leq c \log (\alpha n) &\Rightarrow n \leq 2c\log (\alpha c),\\
     n \geq 4c\log^2 (\alpha c) &\Rightarrow n \geq c \log^2 (\alpha n) ,\\
     n \leq c \log^2 (\alpha n) &\Rightarrow n \leq 4c\log^2 (\alpha c).
 \end{align*}
\end{lemma}
\begin{proof}
We prove the first result and the second result is just the contrapositive of the first result. Consider the following two cases. First we assume $\alpha =1$ to prove $n \geq 2c \log  c \Rightarrow n \geq c\log n$.

If $n \leq c^2$, then
\begin{align*}
    c \log n \leq c \log c^2 \leq 2c \log c \leq n.
\end{align*}
If $n > c^2$, then
\begin{align*}
    c\log n \leq \sqrt{n}\log n \leq \sqrt{n}^2 = n.
\end{align*}
Replacing $n$ and $c$ above with $\alpha n$ and $\alpha c$, we have $\alpha n \geq 2\alpha c \log  (\alpha c) \Rightarrow \alpha n \geq \alpha c\log (\alpha n)$, i.e. $n \geq 2c \log (\alpha c) \Rightarrow n \geq c\log (\alpha n)$. 

The third and forth results can be proved similarly.
\end{proof}

\fi

Finally, the following lemma is used in \Cref{thm:meta_safe}.

\begin{lemma}\label{lemma:sequence converge}
Given $a,b>0$, define a positive sequence $\{x_n\}_{n\geq1}$ recursively by
\[x_{n+1} = \frac{b}{a - x_{n}}.\]

If $a^2>4b$ and $x_1\in [\frac{a-\sqrt{a^2-4b}}{2}, \frac{a+\sqrt{a^2-4b}}{2})$, then, $\{x_n\}$ converges to $\frac{a-\sqrt{a^2-4b}}{2}$.
\end{lemma}

\begin{proof}
\textbf{Step 1.} We first show that $x_n\in [\frac{a-\sqrt{a^2-4b}}{2}, \frac{a+\sqrt{a^2-4b}}{2})$.

This is true for $n=1$, as $x_1\in [\frac{a-\sqrt{a^2-4b}}{2}, \frac{a+\sqrt{a^2-4b}}{2})$.

Assume that $x_{n-1}\in [\frac{a-\sqrt{a^2-4b}}{2}, \frac{a+\sqrt{a^2-4b}}{2})$.
Then, with simple algebra, we can show that 
\[x_n = \frac{b}{a-x_{n-1}} \in \left[\frac{a-\sqrt{a^2-4b}}{2}, \frac{a+\sqrt{a^2-4b}}{2}\right).\]

\textbf{Step 2.} We show that $\{x_n\}$ is a non-increasing sequence.

Indeed, from Step 1, we have
\begin{align*}
    &\quad |a-2x_n|\leq \sqrt{a^2-4b}\\
    &\Rightarrow  a^2 - 4ax_n + 4x_n^2 \leq a^2-4b\\
    &\Rightarrow ax_n -x_n^2\geq b\\
    &\Rightarrow x_n\geq \frac{b}{a - x_n} = x_{n+1}.
\end{align*}

Therefore, $x_{n+1}\leq x_{n}$ holds for all $n\geq 1$. Combining Steps 1 and 2, we conclude that there exists a limit of the sequence $\{x_n\}$, denoted by $x^*$.

By the recursive formula, $x^*$ must be a solution to the following equation
\begin{align*}
    x^* = \frac{b}{a-x^*}.
\end{align*}

Since $x^*\leq x_1 < \frac{a+\sqrt{a^2-4b}}{2}$, by solving the above equation, we conclude that 
\[\lim_{n\rightarrow\infty}x_n = x^* = \frac{a-\sqrt{a^2-4b}}{2}.\]

\end{proof}

\end{document}